\documentclass[12pt, reqno,english]{amsart}

\usepackage[letterpaper]{geometry}
\usepackage{amsmath,amssymb,amsfonts,amsthm}

\usepackage{graphicx}
\usepackage{bm, mathtools}
\usepackage{mathrsfs,enumitem}
\usepackage{color}
\usepackage{fancyhdr}
\usepackage{epstopdf}
\usepackage{amsopn}
\usepackage{amstext}
\usepackage{amscd}
\usepackage{latexsym}
\usepackage{algorithm}
\usepackage{algorithmic}
\newtheorem{thm}{Theorem}[section]
\newtheorem{lemma}{Lemma}[section]

\newtheorem{coro}{Corollary}[section]

\newtheorem{prop}{Proposition}[section]

\theoremstyle{definition}

\newtheorem{rmk}{Remark}[section]

\newtheorem{asp}{Assumption}[subsection]

\title[Zeroth-order SGD Inference]{Statistical Inference for Polyak-Ruppert Averaged Zeroth-order Stochastic Gradient Algorithm}
\author{Yanhao Jin
\and Tesi Xiao 
\and Krishnakumar Balasubramanian
}
\address{\hspace{-0.16in}Department of Statistics, University of California, Davis. \newline Email address: \texttt{yahjin@ucdavis.edu}, \texttt{texiao@ucdavis.edu} and \texttt{kbala@ucdavis.edu}.}

\begin{document}
\maketitle
\begin{abstract}
Statistical machine learning models trained with stochastic gradient algorithms are increasingly being deployed in critical scientific applications. However, computing the stochastic gradient in several such applications is highly expensive or even impossible at times. In such cases, derivative-free or zeroth-order algorithms are used. An important question which has thus far not been addressed sufficiently in the statistical machine learning literature is that of equipping stochastic zeroth-order algorithms with practical yet rigorous inferential capabilities so that we not only have point estimates or predictions, but also quantify the associated uncertainty via confidence intervals or sets.  Towards this, in this work, we first establish a central limit theorem for Polyak-Ruppert averaged stochastic zeroth-order gradient algorithm. We then provide online estimators of the asymptotic covariance matrix appearing in the central limit theorem, thereby providing a practical procedure for constructing asymptotically valid confidence sets (or intervals) for parameter estimation (or prediction) in the zeroth-order setting.
\end{abstract}
\section{Introduction}
\noindent Consider the following stochastic optimization problem
\begin{align}\label{eq:mainproblem}
x^* =\arg\min_{x \in \mathbb{R}^d}\{f(x) = \mathbb{E}[F(x,\zeta)]\}.
\end{align}
The goal in stochastic zeroth-order setting is to solve the above problem when we can only query the function $f(\cdot)$ and have access to noisy evaluations of the form $F(\cdot,\zeta)$, where $\zeta$ is a noise vector which is not necessarily assumed to be additive. Several algorithms have been proposed for the above problem including the seminal works of~\cite{kiefer1952stochastic, fermi1952numerical, blum1954multidimensional, hooke1961direct, spendley1962sequential, powell1964efficient, nelder1965simplex,  nemirovskij1983problem, spall1987stochastic}. We refer the interested reader also to the following books,~\cite{kolda2003optimization, spall2005introduction, conn2009introduction, brent2013algorithms, zabinsky2013stochastic, audet2017derivative}, for details regarding more recent progress, and applications to simulation-based optimization, statistical machine learning and signal processing.

Several statistical machine learning problems could be formulated as stochastic zeroth-order optimization problems of the form in~\eqref{eq:mainproblem}. For example, machine learning models require a large number of hyperparameters to be tuned in order to maximize its efficiency with respect to some performance metric (for example, the test error). However, the analytical form of the functional relationship between the hyperparameters and the performance metric is not typically available. Hence, zeroth-order methods are employed for hyperparameter tuning~\cite{snoek2012practical,lian2016comprehensive,  liu2020admm}. Yet another application is that of designing black-box attacks to machine learning models. In this setting, the goal is to construct minor perturbations of the original training samples so as to fool the machine learning model into predicting incorrectly. Such attacks in-turn are used to build robust training algorithms for machine learning models. As the architecture of the machine learning model is unknown to the attacker, the problem of designing attacks is formulated as solving zeroth-order optimization problem. We refer the reader, for example, to~\cite{zhao2019design, chen2017zoo, cheng2018query, ilyas2018black, dong2019efficient, li2019nattack, cheng2019sign, chen2019zo, chen2020frank} for more details regarding this application. 

Motivated by the above applications (and several others, for example, to reinforcement learning~\cite{mania2018simple, choromanski2020provably}), there has been a renewed interest in the machine learning and optimization communities for developing and analyzing stochastic zeroth-order optimization in the last few years. Specifically, non-asymptotic rates of convergence for stochastic zeroth-oder gradient algorithm, using the popular Gaussian smoothing technique, has been studied, for example, in~\cite{nesterov2017random, ghadimi2013stochastic, shamir2017optimal, gorbunov2018accelerated, yu2018generic, li2020zeroth}; related high-dimensional extensions were studied in~\cite{wang2018stochastic, balasubramanian2018zeroth, cai2020zeroth, golovin2019gradientless}. Non-asymptotic lower bounds are studied in~\cite{jamieson2012query, duchi2015optimal, shamir2013complexity}. Stochastic zeroth-order versions of ADMM algorithms~\cite{liu2018zeroth}, Frank-Wolfe algorithms~\cite{balasubramanian2018zerothc, sahu2019towards, huang2020accelerated, gao2020can}, proximal method~\cite{huang2019faster} have also been recently explored in the literature. Variance reduction in the zeroth-order setting has also been explored in~\cite{liu2018zerothv, ji2019improved, fang2018spider}. The above discussion is only a partial list of recent developments. The intense activity in this sub-area makes it impossible to summarize every work in this direction; we therefore refer the reader to~\cite{larson2019derivative, liu2020primer} for a comprehensive survey of recent advancements.

\subsection{Our Contributions} Despite this recent surge of interest in stochastic zeroth-order optimization, a majority of the existing algorithms are focussed only on providing rates of convergence (either asymptotic or non-asymptotic) from an optimization or estimation error perspective. However, as such algorithms are deployed in critical applications, it becomes important to equip stochastic zeroth-order optimization algorithms with practical yet rigorous inferential capabilities. Towards this, in this work we make the following contributions.
\begin{itemize}[leftmargin=0.16in]
\item We prove a central limit theorem for the Polyak-Ruppert averaged stochastic zeroth-order gradient algorithm. The established CLT is biased; the bias goes to zero as the smoothing parameter used in the zeroth-order gradient estimation procedure goes to zero. 
\item We construct online estimators for estimating the asymptotic covariance matrix appearing in the central limit theorem, so as to enable practical construction of confidence sets. 
\end{itemize}

\subsection{Motivating Application}\label{sec:yahoo}
A concrete motivating example is to construct confidence intervals for predicting the user ratings for the Yahoo! Music dataset~\cite{dror2012yahoo}, which was a part of 2011 KDD-cup.  The approach of the winning team in this competition was to construct around 200 base predictive machine learning models for predicting the ratings, and combining the obtained ratings linearly using ensemble methods. Recently,~\cite{lian2016comprehensive, yu2018generic} solved a zeroth-order least-squares problem for performing the ensemble step with linear combinations of base predictions, and obtained performance similar to the winning method. The need for zeroth-order optimization arise naturally in the ensemble step as the true predictions are known only to the competition organizers and is unknown to the participants.  However, they not provide any confidence intervals for their predictions. This is due to the lack of practical and rigorous methods for uncertainty quantification of stochastic zeroth-order optimization algorithms.   

 
\subsection{Related Works on Inference for Stochastic Optimization Algorithms} In the stochastic first-order setting, studying the asymptotic distribution of stochastic gradient algorithm goes back to the early works of~\cite{chung1954stochastic, sacks1958asymptotic, fabian1968asymptotic}; see also~\cite{shapiro1989asymptotic}. These works primarily studied the asymptotic distribution of the last iteration of the stochastic gradient algorithm. However, they had the drawback that in order for the last iterate to achieve the celebrated Cramer-Rao lower bound for parameter estimation (see, for example,~\cite{rao1945information, cramir1946mathematical, van2000asymptotic}), the choice of parameters of the algorithm depended on the unknown minimizer $x^*$. It was shown later in~\cite{ruppert1988efficient} and~\cite{polyak1992acceleration} that averaging the iterates of the stochastic gradient algorithm achieves the Cramer-Rao lower bound, without knowledge of $x^*$; hence the name Polyak-Ruppert averaging. This result has been recently extended to implicit stochastic gradient algorithms~\cite{toulis2017asymptotic}, Nesterov's dual averaging algorithm~\cite{duchi2016local} and proximal-point methods~\cite{asi2019stochastic}. Non-asymptotic rates for this normal approximation was also examined in~\cite{anastasiou2019normal} based on rates of Martingale central limit theorems. Furthermore,~\cite{dieuleveut2020bridging} and~\cite{yu2020analysis} established asymptotic normality of constant step-size stochastic gradient algorithm in the convex and nonconvex setting respectively. Several works also considered the problem of estimating the asymptotic covariance matrix appearing in the central limit theorem. Towards that~\cite{su2018statistical, fang2018online} proposed an online bootstrap procedure, and~\cite{zhu2020fully} provided trajectory-averaging based online estimators motivated by multivariate time-series analysis. See also~\cite{li2018approximate, chen2016statistical} which provided a semi-online procedure. Such works provide a practical approach for constructing confidence intervals for stochastic first-order optimization algorithms.

In the zeroth-order setting, to the best of our knowledge,~\cite{spall1992multivariate} provided the first central limit theorem for simultaneous perturbation gradient estimation algorithm in the deterministic setting. The work of~\cite{l1998budget} provided a central limit result for the Kiefer-Wolfowitz type stochastic zeroth-order algorithm. We also refer the reader to~\cite{gerencser1997rate, yin1999rates, kleinman1999simulation, shapiro1996simulation, dippon1997weighted, dippon2003accelerated} for other related works in this direction. Similar to the first-order setting, the above works consider only central limit theorems for the last iteration of the algorithm, and hence the choice of the step-size parameter require knowledge of the minimizer $x^*$ to remain close to the optimal covariance. This in fact prompted~\cite{l1998budget} to state the question of establishing central limit theorems for Polyak-Ruppert averaged stochastic gradient algorithm in the zeroth-order setting, as an open problem. Finally, it is worth mentioning that some works, for example,~\cite{tang1999asymptotic, liang2010trajectory} study stochastic optimization with various Markov Chain Monte Carlo (MCMC) based gradient estimators and provide asymptotic analysis for those algorithms (including central limit theorems). To the best of our knowledge, there is no work on establishing central limit theorem for Polyak-Ruppert averaged stochastic gradient algorithm with the gradient being estimated by the Gaussian smoothing technique that we consider. Furthermore, there exists no work on constructing and quantifying the accuracy of estimating the asymptotic covariance matrix appearing in the central limit theorem, in an online fashion. Our main focus in this work is on addressing these open questions and thereby quantifying the uncertainty associated with stochastic zeroth-order algorithms.

 The rest of the paper is organized as follows: In Section~\ref{sec:zosgdclt}, we establish asymptotic normality results for the Polyak-Ruppert averaged stochastic zeroth-order gradient algorithm. We also provide bounds on the difference between the covariance matrix appearing the zeroth-order setting and the one appearing in the first-order setting. In Section~\ref{sec:onlineest}, we introduce the online covariance estimation procedure which are subsequently used for constructing practical confidence intervals. Experimental results are provided in Section~\ref{sec:expresults}. All proofs are relegated to the appendix. We end this section with a list of notations, we use. 
 
\vspace{0.1in}
\noindent \textbf{Notations:} In the following, $\|\cdot\|$ denotes the Euclidean norm in $\mathbb{R}^d$ or the operator norm in $\mathbb{R}^{d\times d}$ (depending on context). Furthermore, $\|\cdot \|_F$ denotes the Frobenius norm of a matrix. We let $\mathbf{1}$ and $\mathbf{I}_d $ denote the $d$-dimensional vector of ones and $d\times d$ identity matrix respectively. Furthermore, $ z \stackrel{D}{\rightarrow} N(\mu,\Sigma)$, means that a random vector $z\in\mathbb{R}^d$ converges in distribution to a Gaussian random vector with mean $\mu$ and covariance $\Sigma$.

\section{Asymptotic Normality of Stochastic Zeroth-order Gradient Algorithm}\label{sec:zosgdclt}
In this section, we recall the Gaussian smoothing technique based stochastic zeroth-order gradient algorithm, introduce the assumptions made in our work and establish the central limit theorem for Polyak-Ruppert averaged Stochastic Zeroth-order Gradient Algorithm based on estimating gradients with Gaussian smoothing technique.

\subsection{Preliminaries}

We first describe the precise assumption made on the \emph{stochastic zeroth-order oracle} in this work. 
\begin{asp}\label{as:zo}
	For any $x\in\mathbb{R}^d$, the stochastic zeroth-order oracle outputs an estimator $F(x,\zeta)$ of $f(x)$ such that, $\mathbb{E}[F(x,\zeta)]=f(x)$, $\mathbb{E}[\nabla F(x,\zeta)]=\nabla f(x)$, and $\mathbb{E}[\|\nabla F(x,\zeta)-\nabla f(x)\|^2]\leq \sigma^2$.
\end{asp}
This assumption is made throughout this work and is hence not explictly mentioned in the statements of subsequent results. The assumption above assumes that we have accesses to a stochastic zeroth-order oracle which provides unbiased function evaluations with bounded variance. It is worth noting that in the above, we do not necessarily assume the noise $\zeta$ is additive. Our gradient estimator is then constructed by leverage the Gaussian smoothing technique~\cite{nesterov2017random, ghadimi2013stochastic,balasubramanian2018zeroth}. Specifically, for a point $x \in \mathbb{R}^d$, we define an estimate $G_\nu(x,\zeta,u)$, of  the gradient $\nabla f(x) $ as follows:
\begin{align}\label{eq:gradest}
G_\nu(x,\zeta,u)\coloneqq\frac{F(x + \nu u, \zeta) - F(x, \zeta)}{\nu}~ u,
\end{align}
where $u\sim N(0,\mathbf{I}_d)$ and $\nu >0$ is a tuning parameter. An interpretation of the gradient estimator in~\eqref{eq:gradest} as a consequence of Gaussian Stein's identity, popular in the statistics literature~\cite{stein1972bound}, was provided in~\cite{balasubramanian2018zeroth}. It is also well-known that the gradient estimator in~\eqref{eq:gradest} is a \emph{biased} estimator of the true gradient $\nabla f(x)$.

The gradient estimator in~\eqref{eq:gradest} is referred to as the two-point estimator in the literature. The reason is that, for a given random vector $\zeta$, it is assumed that the stochastic function in~\eqref{eq:gradest} could be evaluated at two points, $F(x+\nu u,\zeta)$ and $F(x,\zeta)$. Such an assumption is satisfied in several statistics, machine learning and simulation based optimization and sampling; see, for example~\cite{spall2005introduction, mokkadem2007companion, dippon2003accelerated, agarwal2010optimal, duchi2015optimal, ghadimi2013stochastic, nesterov2017random}. Yet another estimator in the literature is the one-point estimator, which assumes that for each $\zeta$, we observe only one noisy function evaluation $F(x,\zeta)$. Admittedly, the one-point setting is more challenging than the two-point setting~\cite{shamir2013complexity}. From a theoretical point of view, the use of two-point evaluation based gradient estimator is primarily motivated by the sub-optimality (in terms of oracle complexity) of one-point feedback based stochastic zeroth-order optimization methods either in terms of the approximation accuracy or dimension dependency. In rest of this work, we focus on the two-point setting. We leave the question of obtaining inferential results in the one-point setting as future work. We now describe the assumptions on the objective function being optimized. 
\begin{asp}\label{simpasp4.1}
The function $f(x)$ is assumed to be twice continuously differentiable function, such that $\mu \mathbf{I}_d \leq \nabla^{2} f(x) \leq L \mathbf{I}_d$, for all $x \in \mathbb{R}^d$, for some $\mu>0$ and $L>0$. 
The above assumption is equivalent to assuming that the function is $\mu$-strongly convex and has $L$-Lipschitz continuous gradients. 
\end{asp}
Such an assumption is standard in the stochastic optimization literature and ensures that there is a unique minimizer for~\eqref{eq:mainproblem}. We now describe our assumptions on the random vector $\zeta$ and the function $F$. 
\begin{asp}\label{simpasp4.2}
We assume $\zeta_{t}$ is a sequence of i.i.d. random vectors from a distribution $\Pi$ and $F(x, \zeta)$ is continuously differentiable in $x$ for any $\zeta$, and $\|\nabla F(x, \zeta)\|$ is uniformly integrable for any $x$. Let $\xi_{i}(x_{i-1}) = \nabla F (x_{i-1}, \zeta_i) - \nabla f(x_{i-1})$ be the difference between the true gradient and the gradient of the function $F$ at a point $x_{i-1}$. Then, we assume there exists a constant $K>0$ such that for all $t\geq 0$, almost surely,
$$\mathbb{E}_{i-1}\left[\|\xi_{i}(x_{i-1})\|^{2}\right]+\|\nabla f(x_{i-1})\|^{2}\leq K(1+\|\Delta_{i-1}\|^{2}),$$
where $\Delta_{i-1} = x_{i-1} - x^*$ and $\mathbb{E}_{i-1}(\cdot)$ stands for conditional expectation $\mathbb{E}(\cdot|\mathcal{F}_{i-1})$ with $\mathcal{F}_{i-1}$ denoting $\sigma$ algebra generated by $\{\zeta_1,\ldots, \zeta_{i-1}\}$.
\end{asp}
The first part of the above assumption entails that $\mathbb{E}
_{i-1}[\xi_t(x_{i-1})] = 0$, which makes $\xi_i$ a martingale difference sequence.
We also emphasize that the above assumption is on the gradient of the stochastic function $F$, which is still not observed in the zeroth-order setting. However, we just assume that the function $F$ satisfies the above regularity assumption (similar to Assumption~\ref{as:zo}). Such assumptions are standard in the literature, and are made in works including that  of~\cite{polyak1992acceleration, toulis2017asymptotic, duchi2016local, chen2016statistical, asi2019stochastic}, and is satisfied in various statistical machine learning problem including (regularized) linear and logistic regression.  
\begin{asp}\label{simpasp4.3}
There is a function $H(\zeta)$ with bounded fourth moment, such that the Hessian of $F(x, \zeta)$ is bounded, for all $x \in \mathbb{R}^d$, by $\left\|\nabla^{2} F(x, \zeta)\right\| \leq H(\zeta).$
\end{asp}
The above assumption provides a control over the fourth-moment of the gradient of the function $F$, and once again holds in linear and logistic regression, as proved in~\cite{chen2016statistical}.

\subsection{Central Limit Theorem}\label{sec:clt}
With the above preliminaries, the Polyak-Ruppert averaged stochastic zeroth-order stochastic gradient algorithm, with the gradient estimated using the two-point method as in~\eqref{eq:gradest}, is given by the following iterates:
\begin{equation}\label{zeroth-ordersgd}
    \begin{aligned}
    & x_{i}=x_{i-1}-\eta_{i}  G_\nu(x_{i-1},\zeta_i,u_{i})\\
    & \bar{x}_{n}=\frac{1}{n} \sum_{i=0}^{n-1} x_{i}, \\
    \end{aligned}
\end{equation}
with $x_{0} \in \mathbb{R}^d$ being an initial point, $u_i\in\mathbb{R}^d$ being independent standard Gaussian vectors, and $\eta_{i}$ and $\nu$ being the step-size choice and smoothing parameter respectively. In the sequel, we denote the overall randomness by $\tilde{\zeta}\coloneqq(\zeta,u)$. Finally, we remark that we use only a single stochastic gradient (computed using only a pair of function evaluations) in each iteration of~\eqref{zeroth-ordersgd}. It is possible to use an average of a  mini-batch of stochastic gradients in each iteration as well. In our setting, however, using such a mini-batch only affects the constants in the theoretical results. 

Our approach to proving the asymptotic normality of the above zeroth-order algorithm, proceeds by considering the performance of the above algorithm on the following surrogate problem:
\begin{equation}\label{surgprob1}
 x_\nu^*= \underset{x\in \mathbb{R}^d}{\arg\min} \{ f_\nu(x) \coloneqq \mathbb{E}_{u,\zeta} [F(x+\nu u, \zeta) ] = \mathbb{E}_{\Tilde{\zeta}} [\Tilde{F}(x, \Tilde{\zeta}) ]\}
\end{equation}
by the algorithm given by~\eqref{zeroth-ordersgd}. The reason for doing is as follows: First, it is easy to note (for example, see~\cite{nesterov2017random}) that while $G_{\nu}(x_{i-1},\zeta_{i},u_{i})$ is a biased estimator for $\nabla f(x_{i-1})$, it is an unbiased estimator of $\nabla f_{\nu}(x_{i-1})$. Hence, for a fixed value of $\nu$, the iterates in~\eqref{zeroth-ordersgd} are necessarily distributed around $x_\nu^*$ and not $x^*$. Furthermore, only when $\nu\to 0$ (at some rate to be described later) the iterates are centered around $x^*$. From a practical perspective, as $\nu$ appears in the denominator of the gradient estimator in~\eqref{eq:gradest}, having an extremely small value of $\nu$ leads to numerical instability issues. Hence, we present our asymptotic central limit theorem for any fixed value of $\nu$. We then quantify the difference between the mean of the Gaussian distribution centered at $x_\nu^*$ and $x^*$. Further, we quantify the difference between the asymptotic covariance matrix appearing in the central limit theorem and the Cramer-Rao covariance matrix lower bound for parameter estimation. To proceed, we also require the following standard assumption on the step-size choice (see, for example~\cite{polyak1992acceleration}).
\begin{asp}\label{simpasp4.5}
The step-size choice satisfies for all $i$, $ \eta_{i}>0$, $\left(\eta_{i}-\eta_{i+1}\right) / \eta_{i}=o\left(\eta_{i}\right)$ and $\sum_{i=1}^{\infty} \eta_{i}^{(1+\lambda) / 2} i^{-1 / 2}<\infty$.
\end{asp}
Under Assumption~\ref{simpasp4.1}, the surrogate function, $f_{\nu}(x)$, is also $\mu$-strong convexity and $L$-smooth, i.e. we have $f_\nu(x)$ is also twice continuously differentiable function and
satisfies $\mu \mathbf{I}_d \leq \nabla^{2} f_\nu(x) \leq L \mathbf{I}_d$, for all $x\in\mathbb{R}^d$.
Furthermore, under Assumption~\ref{simpasp4.2} and~\ref{simpasp4.3}, we have the following similar results for the surrogate function $\tilde{F}$.
\begin{lemma}\label{surproblm1}
Under Assumption~\ref{simpasp4.1} and \ref{simpasp4.2}, for the surrogate problem (\ref{surgprob1}), we can find a constant $\tilde{K}>0$ such that for all $t\geq 1$, almost surely we have
$$\mathbb{E}_{i-1}\left[\|\tilde{\xi}_{i}(x_{i-1})\|^{2}\right]+\|\nabla f_{\nu}(x_{i-1})\|^{2}\leq \tilde{K}(1+\|\tilde{\Delta}_{i-1}\|^{2})$$
where $\tilde{\Delta}_{i}=x_{i}-x_{\nu}^{*}$ and $\tilde{\xi}_{i}(x_{i-1}) = G_\nu(x_{i-1},\zeta_i,u_{i}) - \nabla f_\nu(x_{i-1})$.
\end{lemma}

{\color{black} 
\begin{lemma}\label{surproblm2}
Under Assumption \ref{simpasp4.3}, there is a function $\tilde{H}(\zeta,u)$ with bounded fourth moment, such that the operator norm of Jacobian of $G_{\nu}(x, \zeta,u)$ is bounded by 
$$\left\|\nabla_{x} G_{\nu}(x, \zeta,u)\right\|_{*} \leq \tilde{H}(\zeta,u)$$
for all $x$ in the domain of $G_{\nu}(x,\zeta,u)$.
\end{lemma}}
{\color{black}
The proof of Lemma~\ref{surproblm1} and~\ref{surproblm2} are provided in Section~\ref{sec:cltproof}. Given the results above, we could establish the property (stated in lemma \ref{zerolm1}) for the surrogate problem (\ref{surgprob1}) that is similar to the second part of Assumption 3.3 in \cite{polyak1992acceleration}.
}
{\color{black}
\begin{lemma}\label{zerolm1}\color{black}
	In Algorithm \ref{zeroth-ordersgd} and the random perturbation
	$$\tilde{\xi}_{t}=G_{\nu}(x_{t-1},\zeta_{t},u_{t})-\nabla f_{\nu}(x_{t-1})=\tilde{m}_{t}(0)+\tilde{\varphi}_{t}\left(\tilde{\Delta}_{t-1}\right)$$
	where
	\begin{align*}
	\tilde{m}_{t}(0)&=G_\nu(x_{\nu}^{*},\zeta_t,u_{t}) - \nabla f_\nu(x_{\nu}^{*})=G_\nu(x_{\nu}^{*},\zeta_t,u_{t})\\
	\tilde{\varphi}_{t}(\tilde{\Delta}_{t-1})&=\left(G_\nu(x_{t-1},\zeta_t,u_{t})-G_\nu(x_{\nu}^{*},\zeta_t,u_{t})\right) - \left(\nabla f_\nu(x_{t-1})-\nabla f_\nu(x_{\nu}^{*})\right)
	\end{align*}
	we have
	\begin{itemize}
		\item[(a)] $\mathbb{E}\left[\tilde{m}_{t}(0)\right]=0$
		\item[(b)] $\mathbb{E}_{t-1}\left[\tilde{m}_{t}(0)\tilde{m}_{t}(0)^{T}\right]= S_{\nu}$ for all $t\geq 1$ where 
		$$S_{\nu}=\mathbb{E}_{t-1}G_{\nu}(x_{\nu}^{*},u_{t},\zeta_{t})G_{\nu}(x_{\nu}^{*},u_{t},\zeta_{t})^{T}$$
		\item[(c)] $\mathbb{E}_{t-1}\left[\|\tilde{m}_{t}(0)\|_{2}^{2}I(\|\tilde{m}_{t}(0)\|_{2}>C)\right]\rightarrow 0$
		\item[(d)] There is a stochastic process $\tilde{\delta}(\tilde{\Delta}_{t-1},u_{t})$ such that
		$$\mathbb{E}_{t-1}\left\|\tilde{\varphi}_{t}\left(\tilde{\Delta}_{t-1}\right)\right\|_{2}^{2} \leq \tilde{\delta}(\tilde{\Delta}_{t-1},u_{t})\rightarrow 0$$
		as $\tilde{\Delta}_{t-1}\rightarrow 0$.
	\end{itemize}
\end{lemma}
Lemma \ref{zerolm1} will be used when we apply the martingale CLT in the proof of asymptotic distribution for zeroth-order algorithm output. In particular, the Lindeberg condition for martingale CLT (for example Theorem 5.5.11 in \cite{sh1986liptzer}) follows by Lemma \ref{zerolm1}. Given the above results, we can immediately apply Theorem 3 from~\cite{polyak1992acceleration} to the surrogate problem (\ref{surgprob1}) and obtain the following limiting result.}
\begin{prop}\label{prop:clt}
Consider solving the surrogate problem in~\eqref{surgprob1}, using the algorithm in~\eqref{zeroth-ordersgd}. For any fixed $\nu >0$, and fixed dimensionality $d$, under Assumptions~\ref{simpasp4.1},~\ref{simpasp4.2}~\ref{simpasp4.3} and~\ref{simpasp4.5}, as $n \to \infty$, we have 
\begin{itemize}[noitemsep,leftmargin=0.2in]
\item Almost surely we have $\bar{x}_n \to {x_\nu}^*$.
\item $\sqrt{n}\left(\bar{x}_{n}-x_{\nu}^{*}\right)\stackrel{D}{\rightarrow} N(0, \tilde{V})$. 
\end{itemize}
where we have $\tilde{V}\coloneqq\left(\nabla^{2} f_{\nu}(x_{\nu}^{*}) \right)^{-1} S_{\nu} \left(\nabla^{2} f_{\nu}(x_{\nu}^{*})\right)^{-1},~\text{with}~~S_{\nu}\coloneqq\mathbb{E}\left[G_{\nu}(x_{\nu}^{*},u,\zeta)G_{\nu}(x_{\nu}^{*},u,\zeta)^{T}\right]$.
\end{prop}

Recall that for the Polyak-Ruppert averaged stochastic first-order stochastic gradient algorithm, we have that the averaged iterates are asymptotically normal centered around $x^*$ with the asymptotic covariance matrix given by $$V = \left(\nabla^{2} f(x^{*})\right)^{-1}S\left(\nabla^{2} f(x^{*})\right)^{-1},~\text{with}~ S=\mathbb{E}_{{\zeta}}\left[\nabla F(x^{*},\zeta)\nabla F(x^{*},\zeta)^{T}\right].$$ Furthermore, $V$ is also the optimal covariance matrix for any estimator of the parameter $x^*$, due to the Cramer-Rao lower bound. In contrast, the Polyak-Ruppert averaged stochastic zeroth-order stochastic gradient algorithm is asymptotically normal around the point $x_{\nu}^{*}$ with the covariance matrix given by $\tilde{V}$, for any fixed value $\nu$. 
Such a biased central limit theorem is common in prior central limit theorem results for stochastic zeroth-order optimization (which as we discussed previously, is established only for the last iterate and not for the averaged iterate); see, for example~\cite{spall1992multivariate, gerencser1997rate, yin1999rates, kleinman1999simulation, shapiro1996simulation, dippon1997weighted, dippon2003accelerated}.
In what follows, we calculate the difference between the the centers $x_{\nu}^{*}$ and $x^*$, and covariances $V$ and $\tilde{V}$, as a function the smoothing parameter $\nu$.

\subsection{Bounds on Covariance Matrix}\label{sec:boundsoncov}

First recall that $x_\nu^*$ and $x^*$ are the minimizer of the surrogate problem \eqref{surgprob1} and original problem \eqref{eq:mainproblem}. It was shown in~\cite{nesterov2017random} that,
\begin{equation}\label{zerothsolutiondeviation}
\| x^* - x_\nu^* \| \leq \frac{\nu L(d+3)^{3/2}}{2\mu} = \mathcal{O}(\nu).
\end{equation}
We now bound $\|\tilde{V}-V\|_{F}^{2}$ to obtain its dependency on the smoothing parameter $\nu$. We do so by bounding the Frobenius norm of $S_{\nu}-S$ (Lemma~\ref{varerrorlemma1}) and leveraging Lemma 4.2 from~\cite{balasubramanian2018zeroth}. Note that we have that for any matrix $A\in\mathbb{R}^{d\times d}$, $\|A\| \leq\|A\|_{F} \leq \sqrt{d}\|A\|$. Hence, we immediately have
\begin{align*}
    \|S\|_{F}^{2}=\|\mathbb{E}_{\zeta}\nabla F(x^{*},\zeta)\nabla F(x^{*},\zeta)^{T}\|_{F}^{2} \leq \mathbb{E}_{\zeta}\|\nabla F(x^{*},\zeta)\nabla F(x^{*},\zeta)^{T}\|_{F}^{2}\leq \mathbb{E}_{\zeta}\|\nabla F(x^{*},\zeta)\|_{F}^{4}.
\end{align*}
{\color{black}
\begin{lemma}\label{varerrorlemma1}
Under the conditions of Proposition~\ref{prop:clt}, and assuming $\mathbb{E}_{\zeta}\|\nabla F(x^{*},\zeta)\|_{2}^{4}$ is finite,   we have $\|S_{\nu}-S\|_{F}\leq \mathcal{O}(\nu^{2})$. Hence, $\|S_{\nu}-S\|_{F}\rightarrow 0$ as $\nu\rightarrow 0$ with other problem parameters fixed.
\end{lemma}}
\begin{rmk}
By the triangle inequality, we can immediately get
\begin{align}\label{Snubound}
    \|S_{\nu}\|_{F}^{2}&\leq \|S_{\nu}-S\|_{F}^{2}+\|S\|_{F}^{2}\leq \mathcal{O}(\nu^{4})+\mathbb{E}_{\zeta}\|\nabla F(x^{*},\zeta)\|_{F}^{4}]]\lesssim \mathcal{O}(1+\nu^{2}+\nu^{4})
\end{align}
\end{rmk}
\noindent To bound the $\|\tilde{V}-V\|_{F}^{2}$, we also make the following smoothness assumption on the Hessian of the objective function $f$
\begin{asp}\label{asphessian}
The function $f$ is twice differentiable and has Lipschitz continuous Hessian i.e., there exists $L_{H}>0$ such that $\left\|\nabla^{2} f(x)-\nabla^{2} f(y)\right\| \leq L_{H}\|x-y\| \quad \forall x, y \in \mathbb{R}^{d}.$
\end{asp}
\begin{lemma}\label{varerrorlemma2}
Under the conditions of Proposition~\ref{prop:clt} and Assumption \ref{asphessian}, we have $\|\tilde{V}-V\|_{F}^{2}\leq \mathcal{O}(\nu^{6}).$ Hence, $\|\tilde{V}-V\|_{F}^{2}\rightarrow 0$ as $\nu\rightarrow 0$ with other problem parameters held fixed.
\end{lemma}
\section{Online Estimation of Asymptotic Covariance Matrix}
\label{sec:onlineest}In order to leverage the results of Section~\ref{sec:clt} and~\ref{sec:boundsoncov} to obtain practically computable confidence interval, the matrix $\tilde{V}$ (which depends on unknown $x^*_\nu$) needs to be efficiently estimated. We emphasize that this is a non-trivial problem even in the first-order setting, as the iterates $x_t$ form an inhomogeneous Markov chain. In this section, we leverage the recent work  by~\cite{zhu2020fully}, who proposed an online estimator of the asymptotic covariance matrix in the stochastic first-order setting, and extend it to the stochastic zeroth-order setting and propose an estimator of the covariance matrix $\tilde{V}$ appearing in Proposition~\ref{prop:clt}. 

For iterates $\left\{x_{i}\right\}_{i \geq 1}$ in \eqref{zeroth-ordersgd}, consider the batches:
\begin{align}\label{FOCMEequ5}
    \left\{x_{a_{1}}, \ldots, x_{a_{2}-1}\right\},\left\{x_{a_{2}}, \ldots, x_{a_{3}-1}\right\}, \ldots,\left\{x_{a_{m}}, \ldots, x_{a_{m+1}-1}\right\}, \ldots
\end{align}
where $\left\{a_{m}\right\}_{m \in \mathbb{N}}$ is a strictly increasing integer-valued sequence with $a_{1}=1$ and $a_{k}=\left\lfloor C k^{2 /(1-\alpha)}\right\rfloor$ for $k\geq 2$ and $\frac{1}{2}<\alpha<1$. For $i$-th iterate $x_{i},$ we construct a new batch $B_{i}$ including previous data points from iterations $t_{i}$ to $i$, where $t_{i}=a_{m}$ when $i \in\left[a_{m}, a_{m+1}\right)$, as
$$\begin{array}{c}
\ldots, x_{t_{i}-1},\left\{x_{t_{i}}, \ldots, x_{i}\right\}, x_{i+1}, \ldots \\
B_{i}
\end{array}$$
Based on the batch $B_{i}=\left\{x_{t_{i}}, \ldots, x_{i}\right\}$, the recursive estimator $\widehat{\Sigma}_{n}$ at $n$-th step is then defined as
\begin{equation}\label{FOCMEequ6}
    \widehat{\Sigma}_{n}=\frac{\sum_{i=1}^{n}\left(\sum_{k=t_{i}}^{i} x_{k}-l_{i} \bar{x}_{n}\right)\left(\sum_{k=t_{i}}^{i} x_{k}-l_{i} \bar{x}_{n}\right)^{T}}{\sum_{i=1}^{n} l_{i}},
\end{equation}
where $l_{i}=i-t_{i}+1=\left|B_{i}\right|$. As suggested in \cite{zhu2020fully}, this estimator can be calculated recursively via Algorithm \ref{zeroth-ordersgdfullonlinealg}. The main difference from~\cite{zhu2020fully} is the use of the stochastic zeroth-order gradient, due to which, several assumptions made in~\cite{zhu2020fully} are not satisfied in the stochastic zeroth-order setting we consider. Theorem \ref{zerothsgdthm1} shows that the recursive estimator $\widehat{\Sigma}_{n}$ converges in operator norm to the asymptotic covariance matrix $\tilde{V}$ appearing in Proposition~\ref{prop:clt}.

\begin{algorithm}[t]\caption{Online Asymptotic Covariance Estimation in the Zeroth-Order Setting}\label{zeroth-ordersgdfullonlinealg}
\begin{algorithmic}
\STATE \textbf{Input}: Access to stochastic zeroth-order oracle with Assumption~\ref{as:zo}, parameters $(\alpha, \eta)$ and step size $\eta_{k}=\eta k^{-\alpha}$ for $k \geq 1$, sequence $\left\{a_{k}\right\}$.

\STATE \textbf{Initialize:} $k_{0}=0, v_{0}=P_{0}=V_{0}=W_{0}=\bar{x}_{0}=0, x_{0}$
\FOR{$n=0,1,2,3, \ldots$}
  \STATE 1. $x_{n+1}=x_{n}-\eta_{n+1} G_\nu(x_{n},\zeta_n,u_{n})$
  \STATE 2. $\bar{x}_{n+1}=\left(n \bar{x}_{n}+x_{n+1}\right) /(n+1)$
  \STATE 3. $k_{n+1}=k_{n}+1_{\left\{n+1=a_{k_{n}+1}\right\}}$
  \STATE 4. $l_{n+1}=n+2-t_{n+1}$
  \STATE 5. $q_{n+1}=q_{n}+l_{n+1}^{2}$
  \STATE 6. $v_{n+1}=v_{n}+l_{n+1}$
  \STATE 7. $W_{n+1}=x_{n+1}+W_{n} 1_{\left\{k_{n}=k_{n+1}\right\}}$
  \STATE 8. $V_{n+1}=V_{n}+W_{n+1} W_{n+1}^{T}$
  \STATE 9. $P_{n+1}=P_{n}+l_{n+1} W_{n+1}$
  \STATE 10. $V_{n+1}^{\prime}=V_{n}+q_{n} \bar{x}_{n} \bar{x}_{n}^{T}-2 P_{n} \bar{x}_{n}^{T}$
\STATE \textbf{Output}: Estimator $\bar{x}_{n+1}$ of $x^*$,  estimated covariance $\widehat{\Sigma}_{n+1}=V_{n+1}^{\prime} / v_{n+1}$
\ENDFOR
\end{algorithmic}
\end{algorithm}

{\color{black}
Lemma \ref{surproblm2} yields that under Assumption \ref{simpasp4.3}, the operator norm of Jacobian of $G_{\nu}(x, \zeta,u)$ is bounded by 
$$\left\|\nabla_{x} G_{\nu}(x, \zeta,u)\right\|_{*} \leq \tilde{H}(\zeta,u)$$
for all $x$ in the domain of $G_{\nu}(x,\zeta,u)$. Lemma \ref{zeroonlinelm1} verifies the part 2) and 3) of Assumption 2 in \cite{zhu2020fully}.
\begin{lemma}\label{zeroonlinelm1}
	Under Assumption \ref{simpasp4.3}, the conditional covariance of $\tilde{\xi}_{n}$ has an expansion around $x=x_{\nu}^{*}$:
	\begin{equation}\label{expansionequ1}
	\mathbb{E}_{n-1} \tilde{\xi}_{n} \tilde{\xi}_{n}^{T}=S_{\nu}+\tilde{\Sigma}_{\nu}\left(\tilde{\Delta}_{n-1}\right)
	\end{equation}
	and there exists constants $\tilde{\Sigma}_{1}$ and $\tilde{\Sigma}_{2}>0$ such that for any $\Delta \in \mathbb{R}^{d}$
	\begin{equation}\label{expansionequ2}
	\|\tilde{\Sigma}_{\nu}(\tilde{\Delta})\|_{2} \leq \tilde{\Sigma}_{1}\|\tilde{\Delta}\|_{2}+\tilde{\Sigma}_{2}\|\tilde{\Delta}\|_{2}^{2}
	\end{equation}
	and
	\begin{equation}\label{expansionequ4}
	|\operatorname{tr}(\tilde{\Sigma}_{\nu}(\tilde{\Delta}))| \leq \tilde{\Sigma}_{1}\|\tilde{\Delta}\|_{2}+\tilde{\Sigma}_{2}\|\tilde{\Delta}\|_{2}^{2}
	\end{equation}
	There exists constants $\tilde{\Sigma}_{3}, \tilde{\Sigma}_{4}$ such that the fourth conditional moment of $\tilde{\xi}_{n}$ is bounded by 
	\begin{equation}\label{expansionequ3}
	\mathbb{E}_{n-1}\left\|\tilde{\xi}_{n}\right\|_{2}^{4} \leq \tilde{\Sigma}_{3}+\tilde{\Sigma}_{4}\left\|\tilde{\Delta}_{n-1}\right\|_{2}^{4}
	\end{equation}
	where $\tilde{\Sigma}_{1}=4\sqrt{\mathbb{E}_{n-1}\left\|G_{\nu}\left(x_{\nu}^{*}, u_{n},\zeta_{n}\right)\right\|_{2}^{2}\mathbb{E}\tilde{H}(\zeta)^{2}}$, $\tilde{\Sigma}_{2}=4 \mathbb{E} \tilde{H}(\zeta)^{2}$, $\tilde{\Sigma}_{3}=8\mathbb{E}_{n-1}\left\|G_{\nu}\left(x_{\nu}^{*}, u_{n},\zeta_{n}\right)\right\|_{2}^{4}$ and $\tilde{\Sigma}_{4}=128 \mathbb{E} \tilde{H}(\zeta)^{4}$
\end{lemma}
Therefore, under Assumption \ref{simpasp4.1},  \ref{simpasp4.3} and \ref{simpasp4.5}, similar results of Lemma 3.2 in \cite{chen2016statistical} also hold for the surrogate problem.
}

\begin{thm}\label{zerothsgdthm1}
Under Assumption \ref{simpasp4.1} and \ref{simpasp4.2}, let $a_{k}=\left\lfloor C k^{\beta}\right\rfloor$, where $C$ is a constant and $\beta>\frac{1}{1-\alpha}$. Set step size at the i-th iteration as $\eta_{i}=\eta i^{-\alpha}$ with $\frac{1}{2}<\alpha<1$ the satisfies the Assumption \ref{simpasp4.5}. Then for $\hat{\Sigma}_{n}$ defined in (\ref{FOCMEequ6}),  and for sufficiently large $M$, we have
\begin{equation}\label{FOCMEequ22}\color{black}
\begin{aligned}
\mathbb{E}\left\|\widehat{\Sigma}_{n}-\tilde{V}\right\| \lesssim &\sqrt{1+\tilde{\Sigma}_{1}+C_{d}^{3}+\nu+(\nu+\nu^{2}+\nu^{3})\tilde{\Sigma}_{1}} M^{-\frac{\alpha \beta}{4}}\\
&~~~~~~~+\sqrt{(1+\nu^{2}+\nu^{4})}M^{-\frac{1}{2}}+\sqrt{1+\nu+\nu^{2}}M^{\frac{(\alpha-1) \beta+1}{2}}
\end{aligned}
\end{equation}
where $M$ is the number of batches such that $a_{M} \leq n<a_{M+1}$.  
\end{thm}

\begin{rmk}
The estimation error rate of $\widehat{\Sigma}_{n}$ goes to zero as $M$ goes to infinity and thus $\widehat{\Sigma}_{n}$ is a consistent estimator of $\tilde{V}$. Setting $\beta=\frac{2}{1-\alpha}$, we get Corollary \ref{zerothsgdcoro1} from Theorem \ref{zerothsgdthm1}, which upper bounds the estimation error in terms of number of iterations $n$.
\end{rmk}
\begin{coro}\label{zerothsgdcoro1}
Under the same conditions in Theorem \ref{zerothsgdthm1} and setting $a_{k}=\left\lfloor C k^{2 /(1-\alpha)}\right\rfloor,$ we have
\begin{equation}\label{FOCMEequ23}
    \mathbb{E}\left\|\widehat{\Sigma}_{n}-\tilde{V}\right\| \lesssim \frac{\left(\sqrt{(1+\nu^{2}+\nu^{4})}+\sqrt{1+\nu+\nu^{2}}\right)}{n^{(1-\alpha) / 4}} .
\end{equation}
\end{coro}
Recall that in~\eqref{zeroth-ordersgd}, we use only one stochastic gradient (based on a pair of calls to the stochastic zeroth-order oracle). Hence, the above bound also provides the order on the number of call to the stochastic zeroth-order oracle to obtain a $\epsilon$-accurate estimator of the asymptotic covariance matrix appearing in the central limit result in Proposition~\ref{prop:clt}.

We now provide a sketch of the proof  highlighting the main differences from the proof of~\cite{zhu2020fully}. We defer the full proof to Section~\ref{sec:onlinest}, as it is involved.
\begin{proof}[\textbf{Proof Sketch of Theorem~\ref{zerothsgdthm1}}] 
The main idea of the proof for Theorem \ref{zerothsgdthm1} relies on the linear sequence defined by 
\begin{equation}\label{linearseq}
U_{n}=\left(I-\eta_{n} A\right) U_{n-1}+\eta_{n} \tilde{\xi}_{n}
\end{equation}
 where $A=\nabla^{2}f_{\nu}(x_{\nu}^{*})$ and $U_{0}=\tilde{\Delta}_{0}$. One can show that under the condition of Theorem \ref{zerothsgdthm1}, the estimator for linear sequence defined by
$$\tilde{\Sigma}_{n}=\frac{1}{\sum_{i=1}^{n} l_{i}} \sum_{i=1}^{n}\left(\sum_{k=t_{i}}^{i} U_{k}-l_{i} \bar{U}_{n}\right)\left(\sum_{k=t_{i}}^{i} U_{k}-l_{i} \bar{U}_{n}\right)^{T}$$
is a consistent estimator of $\tilde{V}$, where $\bar{U}_n=\frac{1}{n}\sum_{i=0}^{n-1}U_i$. 
To show the error rate of $\tilde{\Sigma}_{n}$, the triangle inequality gives us 
\begin{equation}\label{FOCME2equ39}
\begin{aligned}
&\mathbb{E}\left\|\tilde{\Sigma}_{n}-\tilde{V}\right\|\\
  \leq &\mathbb{E}\left\|\left(\sum_{i=1}^{n} l_{i}\right)^{-1} \sum_{i=1}^{n}\left(\sum_{k=t_{i}}^{i} U_{k}\right)\left(\sum_{k=t_{i}}^{i} U_{k}\right)^{T}-\tilde{V}\right\|\\
    +&\mathbb{E}\left\|\left(\sum_{i=1}^{n} l_{i}\right)^{-1} \sum_{i=1}^{n} l_{i}^{2} \bar{U}_{n} \bar{U}_{n}^{T}\right\|\\
    +&2 \mathbb{E}\left\|\left(\sum_{i=1}^{n} l_{i}\right)^{-1} \sum_{i=1}^{n}\left(\sum_{k=t_{i}}^{i} U_{k}\right)\left(l_{i} \bar{U}_{n}\right)^{T}\right\|.
\end{aligned}    
\end{equation}
The three terms in the right hand side of~\eqref{FOCME2equ39} in turn could be bounded respectively by {\color{black}
\begin{align*}
&\sqrt{1+\tilde{\Sigma}_{1}+C_{d}^{3}+\nu+(\nu+\nu^{2}+\nu^{3})\tilde{\Sigma}_{1}} M^{-\frac{\alpha \beta}{4}}+\sqrt{(1+\nu^{2}+\nu^{4})}M^{-\frac{1}{2}}.+\sqrt{1+\nu+\nu^{2}}M^{\frac{(\alpha-1) \beta+1}{2}} \\
&~~~~~~~\qquad\qquad~~~~(1+\tilde{\Sigma}_{1}+(1+\tilde{\Sigma}_{1})\nu+\nu^{2}) M^{-1},~\text{and}~(\tilde{\Sigma}_{1}(1+\nu)+(1+\nu+\nu^{2}))M^{-\frac{1}{2}}
\end{align*}}
where $C_d$ is a constant that only depends on dimension $d$ and parameter $\nu$ (whose order will be determined later in the proof).
Hence, we have, the error rate of $\tilde{\Sigma}_{n}$ is given by
\begin{align}\label{FOCME2equ38}
 \mathbb{E}\left\|\tilde{\Sigma}_{n}-\tilde{V}\right\| &\lesssim~\text{Right hand side of}~\eqref{FOCMEequ22}.
\end{align}

We could now adapt the proof of Theorem 1 in \cite{zhu2020fully} by replacing martingale difference $\tilde{\xi}_{n}$ in first-order SGD iterates by the new difference sequence $\tilde{\xi}_{n}=\nabla f\left(x_{n-1}\right)-\nabla \tilde{F}\left(x_{n-1}, \tilde{\xi}_{n}\right)$. To complete the proof of Theorem \ref{zerothsgdthm1}, it suffices to show that the order of $\mathbb{E}\|\hat{\Sigma}_{n}-\tilde{\Sigma}_{n}\|_{2}$ can be bounded by the same order as $\mathbb{E}\|\tilde{\Sigma}_{n}-\tilde{V}\|_{2}$. Note that $\tilde{\Delta}_{n}=x_{n}-x_{\nu}^{*}$ and $\bar{\Delta}_{n}=\frac{1}{n}\sum_{i=1}^{n}\tilde{\Delta}_{i}$, then $\widehat{\Sigma}_{n}$ can be rewritten as
$$\widehat{\Sigma}_{n}=\left(\sum_{i=1}^{n} l_{i}\right)^{-1}\left(\sum_{k=t_{i}}^{i} \tilde{\Delta}_{k}-l_{i} \bar{\Delta}_{n}\right)\left(\sum_{k=t_{i}}^{i} \tilde{\Delta}_{k}-l_{i} \bar{\Delta}_{n}\right)^{T}.$$
Given this representation, with some further calculations, it could be shown that the difference $\mathbb{E}\|\tilde{\Sigma}_{n}-\widehat{\Sigma}_{n}\|_{2} \lesssim M^{-1 / 2}$ and hence we obtained the stated result.

\end{proof}

\section{Experimental Results}\label{sec:expresults}
\begin{figure}[t!]
\begin{center}
\centerline{\includegraphics[scale=0.6]{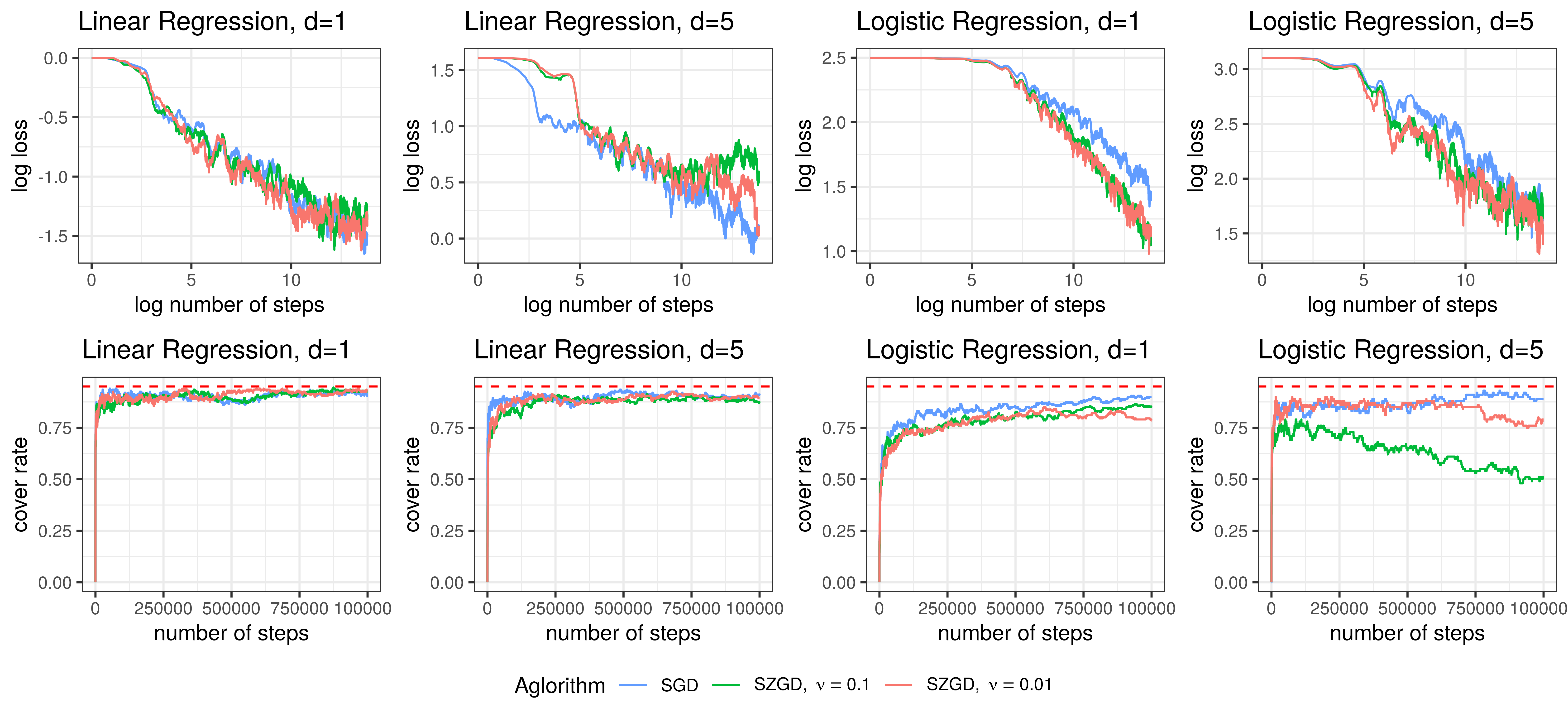}}
\caption{\textit{Top}: The average log estimation error of $\widehat{\Sigma}_{n}$ vs. the log number of iterations. \textit{Bottom}: The cover rate of $95\%$ confidence interval of $\mathbf{1}^Tx^*$ vs. the number of iterations. Red dashed line denotes the nominal coverage rate of 0.95. SGD and SZGD refers to first-order and zeroth-order stochastic gradient algorithms respectively.}
\label{fig:simulation}
\end{center}
\end{figure}

We first evaluate the empirical performance of the zeroth-order statistical inference in the context of online linear and logistic regression. Recall that the objective function $f(x)$ for linear and logistic regression respectively are defined by $\mathbb{E}_\xi [\big(\langle a, x\rangle-b\big)^{2}],$ and $f(x)=\mathbb{E}_\xi[ \log \left(1+\exp \left(-b\left\langle a, x\right\rangle\right)\right)]$ respectively, where $\xi=(a,b)$ with $a\in\mathbb{R}^d$ being a random covariate vector. For linear regression, the response variable $b\in\mathbb{R}$ is given by $b=\langle a ,x^*\rangle + \epsilon$,  $\epsilon\sim N(0,1)$, with $x^{*} \in \mathbb{R}^{d}$ being the true parameter vector. For logistic regression, the response $b\in\{-1,1\}$ is generated by the following probabilistic model $\mathbb{P}\left(b\mid a\right)=1/(1+\exp \left(-b\left\langle a, x^{*}\right\rangle\right))$. 
For the experiment, we run the stochastic gradient algorithm in both the first-order and zeroth-order setting, and compute the Polyak-Ruppert averaged estimator $\bar{x}_{n}$. Additionally, we compute the estimated covariance matrix $\widehat{\Sigma}_{n}$ by Algorithm \ref{zeroth-ordersgdfullonlinealg}. We measure the estimation error of $\widehat{\Sigma}_{n}$ with $\left|\mathbf{1}^{T}\left(\widehat{\Sigma}_{n}-V\right) \mathbf{1}\right|$, where $V$ is the asymptotic covariance of $\bar{x}_n$ in the first-order setting. Finally, we construct the following $(1-q)\times 100 \%$ confidence interval for $\mathbf{1}^{T} x^{*}$, given by \begin{equation}
\mathbf{1}^{T} \bar{x}_{n}\pm z_{1-q/2} \sqrt{\mathbf{1}^{T} \widehat{\Sigma}_{n} \mathbf{1} / n}.
\end{equation}
In the experiments, we monitor the estimation error of $\widehat{\Sigma}_{n}$ and the cover rate of the proposed confidence interval over $200$ ($d=1$) or $150$ ($d=5$) independent runs. For each simulation, the true coefficient $x^{*}$ is a $d$-dimensional vector with each element taking a uniform random value between $0$ and $1$. The covariate vector $a_i$ is generated from $N\left(0, \mathbf{I}_{d}\right)$ for linear regression and from $U[-1,1]^d$ for logistic regression. In Algorithm \ref{zeroth-ordersgdfullonlinealg}, we set $\alpha=0.505$, and $\eta_{j}$ is set to be $0.5 j^{-\alpha}$ when $d=1$ and $0.1 j^{-\alpha}$ when $d=5$, $a_{k}=\left\lfloor k^{2 /(1-\alpha)}\right\rfloor$.  We examine the performance of stochastic gradient algorithm in the first-order setting and the zeroth-order setting with the smoothing parameter set as $\nu=0.1,0.01$ and present the results in Figure \ref{fig:simulation}. We observe that the performance of the zeroth-order algorithm and the first-order algorithm are comparable as long as $\nu$ is sufficiently close to 0. 

\begin{rmk}
In the right panel of Figure \ref{fig:simulation}, it is interesting to observe that the performance in the zeroth-order setting is better than that of the first-order setting. This could be ascribed to the nature of risk function for logistic regression~\cite{ji2018risk}. For the strongly-convex risk function with nonseparable data, the zeroth-order surrogate problem seems to enjoy certain algorithmic benefits, which is worth exploring more in future. 
\end{rmk}
\section{Conclusion}
In this work, we proved a central limit theorem for the Polyak-Rupper averaged stochastic zeroth-order gradient algorithm and provided a practical procedure for constructing confidence interval based on estimating the asymptotic covariance matrix in an online manner. We also demonstrated the practical applicability of our method for constructing predictive confidence intervals. For future work, it is interesting to obtain non-asymptotic results where dimension $d$ is allowed to grow along with iterations. 
\section*{Acknowledgement}
YJ and TX contributed equally to the project. The research of KB and TX was supported in part by UC Davis CeDAR (Center for Data Science and Artificial Intelligence Research) Innovative Data Science Seed Funding Program.  
\appendix

\section{Proofs for Section~\ref{sec:zosgdclt}}\label{sec:cltproof}

\begin{proof}[Proof of Lemma~\ref{surproblm1}]
Consider the decomposition for the random perturbation for the surrogate problem (\ref{surgprob1}) 
$$\tilde{\xi}_{i}(x_{i-1})=G_\nu(x_{i-1},\zeta_i,u_{i})-\nabla f_\nu(x_{i-1})=\tilde{m}_{i}(0)+\tilde{\varphi}_{i}(\tilde{\Delta}_{i-1}),$$
where
\begin{align*}
    \tilde{m}_{i}(0)&=G_\nu(x_{\nu}^{*},\zeta_i,u_{i}) - \nabla f_\nu(x_{\nu}^{*})=G_\nu(x_{\nu}^{*},\zeta_i,u_{i}),\\
    \tilde{\varphi}_{i}(\tilde{\Delta}_{t-1})&=\left(G_\nu(x_{i-1},\zeta_i,u_{i})-G_\nu(x_{\nu}^{*},\zeta_i,u_{i})\right) - \left(\nabla f_\nu(x_{i-1})-\nabla f_\nu(x_{\nu}^{*})\right).
\end{align*}
Then, since $\zeta_{i}$ and $u_{i}$ are both i.i.d., and are mutually independent, we have for all $i\geq 1$ 
$$\mathbb{E}_{i-1}\left[\tilde{m}_{i}(0)\tilde{m}_{i}(0)^{T}\right]=\tilde{S},$$
where $$\tilde{S}=\mathbb{E}_{i-1}\left[G_\nu(x_{\nu}^{*},\zeta_i,u_{i})G_\nu(x_{\nu}^{*},\zeta_i,u_{i})^{T}\right],$$ 
    is the conditional coveraince matrix of $G_{\nu}(x_{\nu}^{*},\zeta_{i},u_{i})$ with $\mathrm{tr}\left(S\right)<\infty$. By trace trick, we obtain
    \begin{align*}
\mathbb{E}_{i-1}\left[\|\tilde{m}_{i}(0)\|^{2}\right]&=\mathbb{E}_{i-1}\left[\tilde{m}_{i}(0)^{T}\tilde{m}_{i}(0)\right]=\mathbb{E}_{i-1}\left[\mathrm{tr}(\tilde{m}_{i}(0)^{T}\tilde{m}_{i}(0))\right]\\
&=\mathbb{E}_{i-1}\left[\mathrm{tr}(\tilde{m}_{i}(0)\tilde{m}_{i}(0)^{T})\right]=\mathrm{tr}\mathbb{E}_{i-1}\left[\tilde{m}_{i}(0)\tilde{m}_{i}(0)^{T}\right]\\
&=\mathrm{tr}(\tilde{S})<\infty.
\end{align*}
Now we need to bound the expectation $\mathbb{E}_{i-1}\|\tilde{\varphi}_{i}(\tilde{\Delta}_{i-1})\|^{2}$. Note that
\begin{align*}
    \mathbb{E}_{i-1}\|\tilde{\varphi}_{i}(\tilde{\Delta}_{i-1})\|^{2} &\leq 2 \mathbb{E}_{i-1}\|G_\nu(x_{i-1},\zeta_i,u_{i})-G_\nu(x_{\nu}^{*},\zeta_i,u_{i})\|^{2}+2\|\nabla f_\nu(x_{i-1})-\nabla f_\nu(x_{\nu}^{*})\|^{2}.
\end{align*}
Now, by Equation 12 and Theorem 4 from~\cite{nesterov2017random}, and by reverse triangle inequality, we have

\begin{align*}
    \mathbb{E}_{i-1}\left[\|\tilde{\xi}_{t}(x_{i-1})\|^{2}\right]\leq
   K (\|\tilde{\Delta}_{i-1}\|^{2}+\nu^2)+2\mathrm{tr}(\tilde{S})
\end{align*}
for some constant $K$ that depends on $d$ and $L$ (all of which are assumed to be fixed in our setting). Hence, this result immediately implies that 
\begin{align*}
    \mathbb{E}_{i-1}\left[\|\tilde{\xi}_{i}(x_{i-1})\|^{2}\right]+\|\nabla f_{\nu}(x_{i-1})\|^{2}&=\mathbb{E}_{i-1}\left[\|\tilde{\xi}_{i}(x_{i-1})\|^{2}\right]+\|\nabla f_{\nu}(x_{i-1})-\nabla f_{\nu}(x_{\nu}^{*})\|^{2}\\
    &\leq K(\|\tilde{\Delta}_{i-1}\|^{2}+\nu^2)+2\mathrm{tr}(\tilde{S})+L^{2}\|\tilde{\Delta}_{i-1}\|^{2}\\
    &\leq \tilde{K}(1+\|\tilde{\Delta}_{i-1}\|^{2})
\end{align*}
where $\tilde{K}=\max\left\{K\nu^2+2\mathrm{tr}(\tilde{S}), K+L^2\right\}$.
\end{proof}

\begin{proof}[Proof of Lemma~\ref{surproblm2}]
\textcolor{black}{
Note that 
\begin{align*}
\nabla_{x}G_{\nu}(x, \zeta,u) &= \nabla_{x} \left(\frac{F\left(x+\nu u, \zeta\right)-F\left(x, \zeta\right)}{\nu} u\right)=\nabla_{x}\begin{bmatrix}
\frac{F\left(x+\nu u, \zeta\right)-F\left(x, \zeta\right)}{\nu}u_{1}\\
\vdots\\
\frac{F\left(x+\nu u, \zeta\right)-F\left(x, \zeta\right)}{\nu}u_{d}
\end{bmatrix}\\
&=\begin{bmatrix}
\frac{\partial}{\partial x_{1}}\frac{F\left(x+\nu u, \zeta\right)-F\left(x, \zeta\right)}{\nu}u_{1} & \dots & \frac{\partial}{\partial x_{d}}\frac{F\left(x+\nu u, \zeta\right)-F\left(x, \zeta\right)}{\nu}u_{1}\\
\vdots & \ddots & \vdots\\
\frac{\partial}{\partial x_{1}}\frac{F\left(x+\nu u, \zeta\right)-F\left(x, \zeta\right)}{\nu}u_{d} & \dots & \frac{\partial}{\partial x_{d}}\frac{F\left(x+\nu u, \zeta\right)-F\left(x, \zeta\right)}{\nu}u_{d}
\end{bmatrix}\\
&=\begin{bmatrix}
u_{1}\\
\vdots\\
u_{d}
\end{bmatrix}\begin{bmatrix}
\frac{\partial}{\partial x_{1}}\frac{F\left(x+\nu u, \zeta\right)-F\left(x, \zeta\right)}{\nu} & \dots & \frac{\partial}{\partial x_{d}}\frac{F\left(x+\nu u, \zeta\right)-F\left(x, \zeta\right)}{\nu}
\end{bmatrix}\\
&=u\left(\frac{\nabla F\left(x+\nu u, \zeta\right)-\nabla F\left(x, \zeta\right)}{\nu}\right)^{T}\\
&=u\left(\nabla^{2}F(x+\theta\nu u,\zeta)u\right)^{T}=uu^{T}\nabla^{2}F(x+\theta\nu u,\zeta)
\end{align*}
where $u\sim N(0,I_{d})$. Hence
\begin{align*}
\left\|\nabla_{x}G_{\nu}(x, \zeta,u)\right\|_{*}&\leq \|uu^{T}\|_{*}\cdot\|\nabla^{2}F(x+\theta\nu u,\zeta)\|_{*}\\
&\leq \|uu^{T}\|_{*}\cdot H(\zeta)\stackrel{\Delta}{=}\tilde{H}(\zeta,u)
\end{align*}
Under Assumption \ref{simpasp4.3}, $H(\zeta)$ has bounded fourth moment and
$$\mathbb{E}\tilde{H}(\zeta,u)^{4}=\mathbb{E}\|uu^T\|_{*}^{4}H(\zeta)^{4}\leq \mathbb{E}\|uu^T\|_{*}^{4}\mathbb{E}H(\zeta)^{4}<\infty$$
Hence, $\tilde{H}(\zeta,u)$ has bounded fourth moment.
}
\end{proof}
{\color{black}
\begin{proof}[Proof of Lemma~\ref{zerolm1}]\color{black}
	\begin{itemize}
		\item[(a)] $\mathbb{E}\left[\tilde{m}_{t}(0)\right]=\mathbb{E}\left[G_\nu(x_{\nu}^{*},\zeta_t,u_{t}) - \nabla f_\nu(x_{\nu}^{*})\right]=0$
		\item[(b)] $\mathbb{E}_{t-1}\left[\tilde{m}_{t}(0)\tilde{m}_{t}(0)^{T}\right]= S_{\nu}$ for all $t\geq 1$ where 
		\begin{align*}
		S_{\nu}&=\mathbb{E}_{t-1}\left[(G_\nu(x_{\nu}^{*},\zeta_t,u_{t}) - \nabla f_\nu(x_{\nu}^{*}))(G_\nu(x_{\nu}^{*},\zeta_t,u_{t}) - \nabla f_\nu(x_{\nu}^{*}))^{T}\right]\\
		&=\mathbb{E}_{t-1}G_{\nu}(x_{\nu}^{*},u_{t},\zeta_{t})G_{\nu}(x_{\nu}^{*},u_{t},\zeta_{t})^{T}    
		\end{align*} 
		\item[(c)] Note that
		\begin{align*}
		\mathbb{E}_{t-1}\left[\|\tilde{m}_{t}(0)\|_{2}^{2}\right]&=\mathbb{E}_{t-1}\left[\tilde{m}_{t}(0)^{T}\tilde{m}_{t}(0)\right]=\mathbb{E}_{t-1}\left[\mathrm{tr}(\tilde{m}_{t}(0)^{T}\tilde{m}_{t}(0))\right]\\
		&=\mathbb{E}_{t-1}\left[\mathrm{tr}(\tilde{m}_{t}(0)\tilde{m}_{t}(0)^{T})\right]=\mathrm{tr}\mathbb{E}_{t-1}\left[\tilde{m}_{t}(0)\tilde{m}_{t}(0)^{T}\right]\\
		&=\mathrm{tr}(S_{\nu})<\infty
		\end{align*}
		For any fixed constanc $C$, we have following decomposition
		$$\mathbb{E}_{t-1}\left[\|\tilde{m}_{t}(0)\|_{2}^{2}\right]=\mathbb{E}_{t-1}\left[\|\tilde{m}_{t}(0)\|_{2}^{2}I(\|\tilde{m}_{t}(0)\|_{2}>C)\right]+\mathbb{E}_{t-1}\left[\|\tilde{m}_{t}(0)\|_{2}^{2}I(\|\tilde{m}_{t}(0)\|_{2}\leq C)\right]$$
		Since $\|\tilde{m}_{t}(0)\|_{2}^{2}I(\|\tilde{m}_{t}(0)\|_{2}\leq C)$ is monotone increasing in $C$, using monotone convergence theorem yields $$\mathbb{E}_{t-1}\left[\|\tilde{m}_{t}(0)\|_{2}^{2}I(\|\tilde{m}_{t}(0)\|_{2}\leq C)\right]\rightarrow \mathbb{E}_{t-1}\left[\|\tilde{m}_{t}(0)\|_{2}^{2}\right]$$
		and by the fact $\mathbb{E}_{t-1}\left[\|\tilde{m}_{t}(0)\|_{2}^{2}\right]<\infty$ we can conclude that 
		$$\mathbb{E}_{t-1}\left[\|\tilde{m}_{t}(0)\|_{2}^{2}I(\|\tilde{m}_{t}(0)\|_{2}>C)\right]\rightarrow 0$$
		\item[(d)] By mean-value theorem,
		\begin{align*}
		\mathbb{E}_{t-1}\left\|G_\nu(x_{t-1}^{*},\zeta_{t},u_{t})-G_\nu(x_{\nu}^{*},\zeta_{t},u_{t})\right\|_{2}^{2} &\leq \mathbb{E}_{t-1}\left[\sup_{x}\left\|\nabla_{x} G_{\nu}(x, \zeta,u)\right\|_{*}^{2}\right]\left\|\tilde{\Delta}_{t-1}\right\|_{2}^{2}\\
		&\leq \mathbb{E} \tilde{H}(\zeta,u)^{2}\left\|\tilde{\Delta}_{t-1}\right\|_{2}^{2}
		\end{align*}
		By the convexity of $\|x\|_{2}^{2}$ and Jensen's inequality,
		\begin{align*}
		\left\|\nabla f_{\nu}\left(x_{t-1}\right)-\nabla f_{\nu}\left(x^{*}\right)\right\|_{2}^{2}&=\left\|\mathbb{E}_{t-1}\left(G_\nu(x_{t-1}^{*},\zeta_{t},u_{t})-G_\nu(x_{\nu}^{*},\zeta_{t},u_{t})\right)\right\|_{2}^{2}\\
		&\leq \mathbb{E}_{t-1}\left\|G_\nu(x_{t-1}^{*},\zeta_{t},u_{t})-G_\nu(x_{\nu}^{*},\zeta_{t},u_{t})\right\|_{2}^{2}\\
		&\leq \mathbb{E} \tilde{H}(\zeta,u)^{2}\left\|\tilde{\Delta}_{t-1}\right\|_{2}^{2}
		\end{align*}
		then as $\tilde{\Delta}_{t-1} \rightarrow 0$
		\begin{align*}
		\mathbb{E}_{t-1}\left\|\varphi_{t}\left(\Delta_{t-1}\right)\right\|_{2}^{2} &\leq \mathbb{E} \tilde{H}(\zeta,u)^{2}\left\|\tilde{\Delta}_{t-1}\right\|_{2}^{2}\stackrel{\Delta}{=}\tilde{\delta}(\tilde{\Delta}_{t-1},u_{t})\rightarrow 0
		\end{align*}
	\end{itemize}
\end{proof}
}
{\color{black}
\begin{proof}[Proof of Lemma~\ref{varerrorlemma1}]
Using Taylor's theorem, we have for every realization of $\zeta$, 
\begin{align*}
G_\nu(x_{\nu}^{*},\zeta,u)&= \frac{F(x_{\nu}^{*} + \nu u, \zeta) - F(x_{\nu}^{*},\zeta,u)}{\nu} u\\
&=u\left(\frac{1}{\nu}\nu u^{T}\nabla F(x_{\nu}^{*}  + \theta\nu u, \zeta)\right)\\
&=uu^{T}\nabla F(x_{\nu}^{*}  + \theta\nu u, \zeta)
\end{align*}
where $\theta\in (0,1)$. Therefore,
\begin{align*}
\|S_{\nu}-S\|_{F} &=\left\|\mathbb{E}uu^{T}\nabla F\left(x_{\nu}^{*}+\theta \nu u, \zeta\right) \nabla F\left(x_{\nu}^{*}+\theta\nu u, \zeta\right)^{T}uu^T-\mathbb{E}\nabla F(x^{*},\zeta)\nabla F(x^{*},\zeta)^{T}\right\|_{F}\\
&=\left\|\mathbb{E}uu^{T}\left(\nabla F\left(x_{\nu}^{*}+\theta \nu u, \zeta\right) \nabla F\left(x_{\nu}^{*}+\theta\nu u, \zeta\right)^{T}-\nabla F(x^{*},\zeta)\nabla F(x^{*},\zeta)^{T}\right)uu^T\right.\\
&\quad\quad\left.+\mathbb{E}uu^T\nabla F(x^{*},\zeta)\nabla F(x^{*},\zeta)^{T}uu^T-\mathbb{E}\nabla F(x^{*},\zeta)\nabla F(x^{*},\zeta)^{T}\right\|_{F}\\
&\leq T_{1}+T_{2}
\end{align*}
where
\begin{align*}
T_{1} &= \left\|\mathbb{E}uu^{T}\left(\nabla F\left(x_{\nu}^{*}+\theta \nu u, \zeta\right) \nabla F\left(x_{\nu}^{*}+\theta\nu u, \zeta\right)^{T}-\nabla F(x^{*},\zeta)\nabla F(x^{*},\zeta)^{T}\right)uu^T\right\|_{F}\\
T_{2} &= \left\|\mathbb{E}uu^T\nabla F(x^{*},\zeta)\nabla F(x^{*},\zeta)^{T}uu^T-\mathbb{E}\nabla F(x^{*},\zeta)\nabla F(x^{*},\zeta)^{T}\right\|_{F}
\end{align*}
By Jensen's inequality, we have
\begin{align*}
T_{1} &= \left\|\mathbb{E}uu^{T}\left(\nabla F\left(x_{\nu}^{*}+\theta \nu u, \zeta\right) \nabla F\left(x_{\nu}^{*}+\theta\nu u, \zeta\right)^{T}-\nabla F(x^{*},\zeta)\nabla F(x^{*},\zeta)^{T}\right)uu^T\right\|_{F}\\
&\leq \mathbb{E}\left\|uu^{T}\left(\nabla F\left(x_{\nu}^{*}+\theta \nu u, \zeta\right) \nabla F\left(x_{\nu}^{*}+\theta\nu u, \zeta\right)^{T}-\nabla F(x^{*},\zeta)\nabla F(x^{*},\zeta)^{T}\right)uu^T\right\|_{F}\\
&\leq \mathbb{E}\left\|\nabla F\left(x_{\nu}^{*}+\theta \nu u, \zeta\right) \nabla F\left(x_{\nu}^{*}+\theta\nu u, \zeta\right)^{T}-\nabla F(x^{*},\zeta)\nabla F(x^{*},\zeta)^{T}\right\|_{F}\|u\|_{2}^{4}\\
&\leq \sqrt{\mathbb{E}\left\|\nabla F\left(x_{\nu}^{*}+\theta \nu u, \zeta\right) \nabla F\left(x_{\nu}^{*}+\theta\nu u, \zeta\right)^{T}-\nabla F(x^{*},\zeta)\nabla F(x^{*},\zeta)^{T}\right\|_{F}^{2}\mathbb{E}\|u\|_{2}^{8}}
\end{align*}
Note that
\begin{align*}
T_{2}&=\|\mathbb{E}\left[uu^T\nabla F(x^{*},\zeta)\nabla F(x^{*},\zeta)^{T} uu^{T}\right]-\mathbb{E}\nabla F(x^{*},\zeta)\nabla F(x^{*},\zeta)^{T} \|_{F}\\
&=\|\mathbb{E}_{u}\left[uu^T\mathbb{E}_{\zeta}\nabla F(x^{*},\zeta)\nabla F(x^{*},\zeta)^{T} uu^{T}\right]-\mathbb{E}_{\zeta}\nabla F(x^{*},\zeta)\nabla F(x^{*},\zeta)^{T} \|_{F}\\
&=\|\mathbb{E}_{u}\left[uu^TS uu^{T}\right]-S \|_{F}
\end{align*}
Let $A=uu^T=(a_{ij})_{i,j=1}^{d}$ and $P=uu^TS uu^{T}$, then $\mathbb{E}a_{ij}=\mathbb{E}u_{i}u_{j}=0$ for $i\not=j$, $P=ASA$ and
$$p_{ij}=\sum_{l=1}^{d}\sum_{k=1}^{d}a_{ik}s_{kl}a_{lj}$$
Taking expectation w.r.t. $u$ gives
$$\mathbb{E}_{u}p_{ij}=\mathbb{E}_{u}a_{ii}s_{ij}a_{jj}=s_{ij}$$
so $T_{2}=0$. On the other hand, note that
\begin{align*}
&\quad\quad\mathbb{E}\|\nabla F(x_{\nu}^{*}+\theta\nu u,\zeta)\nabla F(x_{\nu}^{*}+\theta\nu u,\zeta)^{T}-\nabla F(x^{*},\zeta)\nabla F(x^{*},\zeta)^{T}\|_{F}^{2}\\
&=\mathbb{E}\left\|\nabla F(x_{\nu}^{*}+\theta\nu u,\zeta)\nabla F(x_{\nu}^{*}+\theta\nu u,\zeta)^{T}-\nabla F(x^{*},\zeta)\nabla F(x_{\nu}^{*}+\theta\nu u,\zeta)^{T}\right.\\
&\quad\quad\quad\left.+\nabla F(x^{*},\zeta)\nabla F(x_{\nu}^{*}+\theta\nu u,\zeta)^{T}-\nabla F(x^{*},\zeta)\nabla F(x^{*},\zeta)^{T}\right\|_{F}^{2}\\
&\leq \mathbb{E}\left\|\nabla F(x_{\nu}^{*}+\theta\nu u,\zeta)\nabla F(x_{\nu}^{*}+\theta\nu u,\zeta)^{T}-\nabla F(x^{*},\zeta)\nabla F(x_{\nu}^{*}+\theta\nu u,\zeta)^{T}\right\|_{F}^{2}\\
&\quad\quad\quad+\mathbb{E}\left\|\nabla F(x^{*},\zeta)\nabla F(x_{\nu}^{*}+\theta\nu u,\zeta)^{T}-\nabla F(x^{*},\zeta)\nabla F(x^{*},\zeta)^{T}\right\|_{F}^{2}\\
&= E_{1}+E_{2}
\end{align*}
where 
\begin{align*}
E_{1}&:=\mathbb{E}\left\|\nabla F(x_{\nu}^{*}+\theta\nu u,\zeta)\nabla F(x_{\nu}^{*}+\theta\nu u,\zeta)^{T}-\nabla F(x^{*},\zeta)\nabla F(x_{\nu}^{*}+\theta\nu u,\zeta)^{T}\right\|_{F}^{2}\\
E_{2}&:=\mathbb{E}\left\|\nabla F(x^{*},\zeta)\nabla F(x_{\nu}^{*}+\theta\nu u,\zeta)^{T}-\nabla F(x^{*},\zeta)\nabla F(x^{*},\zeta)^{T}\right\|_{F}^{2}
\end{align*}
\begin{align*}
E_{1}&=\mathbb{E}_{\zeta,u}\|(\nabla F(x_{\nu}^{*}+\theta\nu u,\zeta)-\nabla F(x^{*},\zeta))(\nabla F(x_{\nu}^{*}+\theta\nu u,\zeta))^{T}\|_{F}^{2}\\
&\leq \mathbb{E}_{\zeta,u}\|\nabla F(x_{\nu}^{*}+\theta\nu u,\zeta)-\nabla F(x^{*},\zeta)\|_{F}^{2}\|(\nabla F(x_{\nu}^{*}+\theta\nu u,\zeta))^{T}\|_{F}^{2}\\
&\leq \mathbb{E}_{\zeta,u}H(\zeta)^{2}\|x_{\nu}^{*}+\theta\nu u-x^{*}\|_{2}^{2}\|\nabla F(x_{\nu}^{*}+\theta\nu u,\zeta)\|_{2}^{2}
\end{align*}
Note that for any $a,b\in\mathbb{R}^{d}$, we have
$$\frac{1}{2}\|a\|^{2}-\|b\|^{2}\leq \|a-b\|^{2}$$
Therefore, applying this inequality with $a=\nabla F(x_{\nu}^{*}+\theta\nu u,\zeta)$ and $b=\nabla F(x^{*},\zeta)$ yields
\begin{align*}
\|\nabla F(x_{\nu}^{*}+\theta\nu u,\zeta)\|_{2}^{2} &\leq 2\|\nabla F(x_{\nu}^{*}+\theta\nu u,\zeta)-\nabla F(x^{*},\zeta)\|_{2}^{2}+\|\nabla F(x^{*},\zeta)\|_{2}^{2}\\
&\leq 2H(\zeta)^{2}\|x_{\nu}^{*}+\theta\nu u-x^{*}\|_{2}^{2}+\|\nabla F(x^{*},\zeta)\|_{2}^{2}
\end{align*}
Therefore,
\begin{align*}
E_{1}&\leq \mathbb{E}_{\zeta,u}\left[H(\zeta)^{2}\|x_{\nu}^{*}+\theta\nu u-x^{*}\|_{2}^{2}\left(2H(\zeta)^{2}\|x_{\nu}^{*}+\theta\nu u-x^{*}\|_{2}^{2}+\|\nabla F(x^{*},\zeta)\|_{2}^{2}\right)\right]\\
&\leq 2\mathbb{E}_{\zeta,u}\left[H(\zeta)^{4}\|x_{\nu}^{*}+\theta\nu u-x^{*}\|_{2}^{4}\right]+\mathbb{E}_{\zeta,u}\left[H(\zeta)^{2}\|x_{\nu}^{*}+\theta\nu u-x^{*}\|_{2}^{2}\|\nabla F(x^{*},\zeta)\|_{2}^{2}\right]\\
&=2\mathbb{E}_{\zeta}H(\zeta)^{4}\mathbb{E}_{u}\|x_{\nu}^{*}+\theta\nu u-x^{*}\|_{2}^{4}+\mathbb{E}_{\zeta}H(\zeta)^{2}\|\nabla F(x^{*},\zeta)\|_{2}^{2}\mathbb{E}_{u}\|x_{\nu}^{*}+\theta\nu u-x^{*}\|_{2}^{2}\\
&\leq 2\mathbb{E}_{\zeta}H(\zeta)^{4}\mathbb{E}_{u}\|x_{\nu}^{*}+\theta\nu u-x^{*}\|_{2}^{4}+\mathbb{E}_{u}\|x_{\nu}^{*}+\theta\nu u-x^{*}\|_{2}^{2}\sqrt{\mathbb{E}_{\zeta}H(\zeta)^{4}\mathbb{E}_{\zeta}\|\nabla F(x^{*},\zeta)\|_{2}^{4}}
\end{align*}
Note that for a Gaussian random vector $u \sim N\left(0, I_{d}\right)$, by Theorem 1.1 in \cite{balasubramanian2018zeroth} we have that 
$$\mathbb{E}\left[\|u\|_{2}^{k}\right] \leq(d+k)^{k / 2}$$
for any $k\geq 2$. Therefore,
\begin{align*}
\mathbb{E}_{u}\|x_{\nu}^{*}+\theta\nu u-x^{*}\|_{2}^{2}&\leq \|x_{\nu}^{*}-x^{*}\|_{2}^{2}+\theta^{2}\nu^{2}\mathbb{E}_{u}\|u\|_{2}^{2}\\
&\leq \frac{\nu^{2} L^{2}(d+3)^{3}}{4\mu^{2}}+\nu^{2}(d+2)\theta^{2}\\
\mathbb{E}_{u}\|x_{\nu}^{*}+\theta\nu u-x^{*}\|_{2}^{4}
&\leq \|x_{\nu}^{*}-x^{*}\|_{2}^{4}+\theta^{4}\nu^{4}\mathbb{E}_{u}\|u\|_{2}^{4}\\
&\leq \frac{\nu^{4} L^{4}(d+3)^{6}}{16\mu^{4}}+\nu^{4}(d+4)^{2}\theta^{4}
\end{align*}
Therefore, if 
\begin{align*}
&\mathbb{E}_{\zeta}H(\zeta)^{4}<\infty\\
&\mathbb{E}_{\zeta}\|\nabla F(x^{*},\zeta)\|_{2}^{4}<\infty
\end{align*}
we have that
\begin{equation*}
E_{1} \leq \mathcal{O}(\nu^{4})\quad\quad\text{and}\quad\quad E_{1}\rightarrow 0\quad\text{as}\quad\nu \rightarrow 0
\end{equation*}
Similarly,
\begin{align*}
E_{2} &=\mathbb{E}_{\zeta,u}\|\nabla F(x^{*},\zeta)(\nabla F(x_{\nu}^{*}+\theta\nu u,\zeta)-\nabla F(x^{*},\zeta))^{T}\|_{F}^{2}\\
&\leq \mathbb{E}_{\zeta,u}\|\nabla F(x^{*},\zeta)\|_{F}^{2}\|(\nabla F(x_{\nu}^{*}+\theta\nu u,\zeta)-\nabla F(x^{*},\zeta))^{T}\|_{F}^{2}\\
&= \mathbb{E}_{\zeta,u}\|\nabla F(x^{*},\zeta)\|_{F}^{2}\|\nabla F(x_{\nu}+\theta\nu u,\zeta)-\nabla F(x^{*},\zeta)\|_{F}^{2}\\
&\leq \mathbb{E}_{\zeta,u}\|\nabla F(x^{*},\zeta)\|_{F}^{2}H(\zeta)^{2}\|x_{\nu}^{*}+\theta\nu u-x^{*}\|_{2}^{2}\\
&=\mathbb{E}_{\zeta}\|\nabla F(x^{*},\zeta)\|_{2}^{2}H(\zeta)^{2}\mathbb{E}_{u}\|x_{\nu}^{*}+\theta\nu u-x^{*}\|_{2}^{2}\\
&\leq\sqrt{\mathbb{E}_{\zeta}\|\nabla F(x^{*},\zeta)\|_{2}^{4}\mathbb{E}_{\zeta}H(\zeta)^{4}}\mathbb{E}_{u}\|x_{\nu}^{*}+\theta\nu u-x^{*}\|_{2}^{2}\\
&\leq \left(\frac{\nu L(d+3)^{3/2}}{2\mu}+\nu^{2}(d+2)\theta^{2}\right)\sqrt{\mathbb{E}_{\zeta}\|\nabla F(x^{*},\zeta)\|_{2}^{4}\mathbb{E}_{\zeta}H(\zeta)^{4}}\\
&=\mathcal{O}(\nu^{2})
\end{align*}
Therefore, $E_{2}\rightarrow 0$ as $\nu \rightarrow 0$ and
$$\mathbb{E}\left\|\nabla F\left(x_{\nu}^{*}+\theta \nu u, \zeta\right) \nabla F\left(x_{\nu}^{*}+\theta\nu u, \zeta\right)^{T}-\nabla F(x^{*},\zeta)\nabla F(x^{*},\zeta)^{T}\right\|_{F}^{2}\leq \mathcal{O}(\nu^{4})$$
and by Theorem 1.1 in \cite{balasubramanian2018zeroth} we have that 
$$\mathbb{E}\|u\|_{2}^{8}\leq (d+8)^{4}$$
Finally, we obtain
\begin{align*}
\|S_{\nu}-S\|_{F}&\leq T_{1}\leq \sqrt{\mathbb{E}\left\|\nabla F\left(x_{\nu}^{*}+\theta \nu u, \zeta\right) \nabla F\left(x_{\nu}^{*}+\theta\nu u, \zeta\right)^{T}-\nabla F(x^{*},\zeta)\nabla F(x^{*},\zeta)^{T}\right\|_{F}^{2}\mathbb{E}\|u\|_{2}^{8}}\\
&\leq \mathcal{O}(\nu^{2})    
\end{align*}
\end{proof}}

\begin{proof}[Proof of Lemma~\ref{varerrorlemma2}]
Applying the matrix identity $A=\nabla^{2} f_{\nu}(x_{\nu}^{*})$ and $B=\nabla^{2} f(x^{*})$
$$A^{-1}-B^{-1}=A^{-1 }\left(B-A\right) B^{-1},$$
we have
\begin{align*}
    \|\nabla^{2} f_{\nu}(x_{\nu}^{*})^{-1}-\nabla^{2} f(x^{*})^{-1}\|_{F}^{2}&=\|\nabla^{2} f_{\nu}(x_{\nu}^{*})^{-1}(\nabla^{2} f(x^{*})-\nabla^{2} f_{\nu}(x_{\nu}^{*}))\nabla^{2} f(x^{*})^{-1}\|_{F}^{2}\\
    &\leq \|\nabla^{2} f(x^{*})\|_{F}^{2}\|\nabla^{2} f_{\nu}(x_{\nu}^{*})^{-1}\|_{F}^{2}\|\nabla^{2} f_{\nu}(x_{\nu}^{*})-\nabla^{2} f(x^{*})\|_{F}^{2}.
\end{align*}
Note that 
\begin{align*}
    \|\nabla^{2} f_{\nu}(x_{\nu}^{*})-\nabla^{2} f(x^{*})\|_{F}^{2} &=\|\nabla^{2} f_{\nu}(x_{\nu}^{*})-\nabla^{2} f(x_{\nu}^{*})+\nabla^{2} f(x_{\nu}^{*})-\nabla^{2} f(x^{*})\|_{F}^{2}\\
    &\leq \|\nabla^{2} f_{\nu}(x_{\nu}^{*})-\nabla^{2} f(x_{\nu}^{*})\|_{F}^{2}+\|\nabla^{2} f(x_{\nu}^{*})-\nabla^{2} f(x^{*})\|_{F}^{2}.
\end{align*}
By Lemma 4.2 in \cite{balasubramanian2018zeroth}, we have 
$$\|\nabla^{2} f_{\nu}(x_{\nu}^{*})-\nabla^{2} f(x_{\nu}^{*})\|_{*}\leq \frac{L_{H} \nu(d+6)^{\frac{5}{2}}}{4}.$$
Note that $\operatorname{rank}\left(\nabla^{2} f_{\nu}(x_{\nu}^{*})-\nabla^{2} f(x_{\nu}^{*})\right)\leq d$. Hence by the relationship between $\|\cdot\|_{F}$ and $\|\cdot\|$, we have
$$\|\nabla^{2} f_{\nu}(x_{\nu}^{*})-\nabla^{2} f(x_{\nu}^{*})\|_{F}^{2}\leq d \|\nabla^{2} f_{\nu}(x_{\nu}^{*})-\nabla^{2} f(x_{\nu}^{*})\|^{2}\leq \frac{L_{H}^{2}\nu^{2}d(d+6)^{5}}{16}.$$
Hence
\begin{align*}
    \|\nabla^{2} f_{\nu}(x_{\nu}^{*})-\nabla^{2} f(x^{*})\|_{F}^{2}&\leq \frac{L_{H}^{2}\nu^{2}d(d+6)^{5}}{16}+L_{H}^{2}\|x_{\nu}-x^{*}\|^{2}\\
    &\leq \frac{L_{H}^{2}\nu^{2}d(d+6)^{5}}{16}+ L_{H}^{2}\frac{\nu^{2} L^{2}(d+3)^{3}}{4\mu^{2}}\\
    &=\mathcal{O}(\nu^{2}),
\end{align*}
and
\begin{equation*}
   \|\nabla^{2} f_{\nu}(x_{\nu}^{*})^{-1}-\nabla^{2} f(x^{*})^{-1}\|_{F}^{2}\leq d^{2}\frac{L^{2}}{\mu^{2}}\|\nabla^{2} f_{\nu}(x_{\nu}^{*})-\nabla^{2} f(x^{*})\|_{F}^{2}\leq \mathcal{O}(\nu^{2}).
\end{equation*}
Finally, note that
\begin{align*}
    \|\tilde{V}-V\|_{F}^{2}&=\|(\nabla^{2} f_{\nu}(x_{\nu}^{*}))^{-1}S_{\nu}\left(\nabla^{2} f_{\nu}(x_{\nu}^{*})\right)^{-1}-\left(\nabla^{2} f(x^{*})\right)S\left(\nabla^{2} f(x^{*})\right)^{-1}\|_{F}^{2}\\
    &\leq \|S_{\nu}\|_{F}^{2}\|\nabla^{2}f_{\nu}(x_{\nu}^{*})^{-1}\|_{F}^{2}\|\nabla^{2}f_{\nu}(x_{\nu}^{*})-\nabla^{2} f(x^{*})\|_{F}^{2}\\
    &\quad\quad+\|\nabla^{2}f(x^{*})\|_{F}^{2}\|\nabla^{2}f_{\nu}(x_{\nu}^{*})^{-1}\|_{F}^{2}\|S_{\nu}-S\|_{F}^{2}\\
    &\quad\quad+\|\nabla^{2}f(x^{*})\|_{F}^{2}\|S\|_{F}^{2} \|\nabla^{2} f_{\nu}(x_{\nu}^{*})^{-1}-\nabla^{2} f(x^{*})^{-1}\|_{F}^{2}\\
    &=N_{1}+N_{2}+N_{3}.
\end{align*}
Next, note that we have
\begin{align*}
    \|\nabla^{2}f(x^{*})\|_{F}^{2}&=\mathrm{tr}\nabla^{2} f(x^{*})^{T}\nabla^{2} f(x^{*})\leq dL^{2},\\
    \|\nabla^{2}f_{\nu}(x_{\nu}^{*})^{-1}\|_{F}^{2}&=\mathrm{tr}\left(\nabla^{2}f_{\nu}(x_{\nu}^{*})^{T}\nabla^{2}f_{\nu}(x_{\nu}^{*})\right)^{-1}\leq \frac{d}{\mu^{2}}.
\end{align*}
Combining the above displays, we obtain
$$N_{1} \leq \mathcal{O}(\nu^{6})
\quad\quad
N_{2}\leq \mathcal{O}(\nu^{4})
\quad\quad
N_{3}\leq \mathcal{O}(\nu^{2})
$$
and $N_{1}\rightarrow 0$, $N_{2}\rightarrow 0$ and $N_{3}\rightarrow 0$ as $\nu\rightarrow 0$. This yields that $\|\tilde{V}-V\|_{F}^{2}\leq \mathcal{O}(\nu^{6})$ and so $\|\tilde{V}-V\|_{F}^{2}\rightarrow 0$ as $\nu\rightarrow 0$.
\end{proof}
\section{Proofs for Section~\ref{sec:onlineest}}\label{sec:onlinest}
{\color{black}
\begin{proof}[Proof of Lemma~\ref{zeroonlinelm1}]\color{black}
	Define $$L\left(\tilde{\Delta}_{n-1}, u_{n},\zeta_{n}\right)=\left(\nabla f_{\nu}\left(x_{n-1}\right)-\nabla f_{\nu}\left(x_{\nu}^{*}\right)\right)-\left(G_{\nu}\left(x_{n-1}, u_{n},\zeta_{n}\right)-G_{\nu}\left(x_{\nu}^{*}, u_{n},\zeta_{n}\right)\right).$$ Note that $\nabla f_{\nu}\left(x_{\nu}^{*}\right)=0$, we have
	$$\tilde{\xi}_{n}=L\left(\tilde{\Delta}_{n-1},u_{n}, \zeta_{n}\right)-G_{\nu}\left(x_{\nu}^{*}, u_{n},\zeta_{n}\right)$$
	and this decomposition yields
	\begin{align*}
	\mathbb{E}_{n-1}\tilde{\xi}_{n}\tilde{\xi}_{n}^{T}=&S_{\nu}-\mathbb{E}_{n-1}L\left(\tilde{\Delta}_{n-1},u_{n}, \zeta_{n}\right)G_{\nu}\left(x_{\nu}^{*}, u_{n},\zeta_{n}\right)^{T}-\mathbb{E}_{n-1}G_{\nu}\left(x_{\nu}^{*}, u_{n},\zeta_{n}\right)L\left(\tilde{\Delta}_{n-1},u_{n}, \zeta_{n}\right)^{T}\\
	&\quad\quad+\mathbb{E}_{n-1}L\left(\tilde{\Delta}_{n-1},u_{n}, \zeta_{n}\right)L\left(\tilde{\Delta}_{n-1}, u_{n},\zeta_{n}\right)^{T}
	\end{align*}
	Now for $m=2,4$, by mean-value theorem we have
	\begin{align*}
	\mathbb{E}_{n-1}\left\|G_{\nu}(x_{n-1},u_{n},\zeta_{n})-G_{\nu}(x_{\nu}^{*},u_{n},\zeta_{n})\right\|_{2}^{m}&\leq \mathbb{E}_{n-1}\left[\sup _{x}\left\|\nabla_{x} G_{\nu}\left(x,u_{n}, \zeta_{n}\right)\right\|^{m}\right]\left\|\tilde{\Delta}_{n-1}\right\|_{2}^{m}\\
	&\leq\left\|\tilde{\Delta}_{n-1}\right\|_{2}^{m} \mathbb{E} \tilde{H}(\zeta)^{m} .
	\end{align*}
	and by Jensen's inequality,
	\begin{align*}
	\left\|\nabla f_{\nu}(x_{n-1})-\nabla f_{\nu}(x_{\nu}^{*})\right\|_{2}^{m} &= \left\|\mathbb{E}_{n-1}G_{\nu}(x_{n-1},u_{n},\zeta_{n})-G_{\nu}(x_{\nu}^{*},u_{n},\zeta_{n})\right\|_{2}^{m}\\
	&\leq \mathbb{E}_{n-1}\left\|G_{\nu}(x_{n-1},u_{n},\zeta_{n})-G_{\nu}(x_{\nu}^{*},u_{n},\zeta_{n})\right\|_{2}^{m}
	\end{align*}
	Using the inequality (given in Lemma B.1 in \cite{chen2016statistical})
	$$n^{3}\left(\left\|x_{1}\right\|_{2}^{4}+\ldots+\left\|x_{n}\right\|_{2}^{4}\right) \geq\left\|x_{1}+\ldots+x_{n}\right\|_{2}^{4}$$
	gives us for $m=2,4$
	$$\mathbb{E}_{n-1}\left\|L\left(\tilde{\Delta}_{n-1}, u_{n},\zeta_{n}\right)\right\|_{2}^{m} \leq 2^{m}\left\|\tilde{\Delta}_{n-1}\right\|_{2}^{m} \mathbb{E} \tilde{H}(\zeta)^{m}$$
	Now $\left\|\mathbb{E}_{n-1}\tilde{\xi}_{n}\tilde{\xi}_{n}^{T}-S_{\nu}\right\|_{2}$ could be bounded by
	\begin{align*}
	\left\|\mathbb{E}_{n-1}\tilde{\xi}_{n}\tilde{\xi}_{n}^{T}-S_{\nu}\right\|\leq &\left\|\mathbb{E}_{n-1}L\left(\tilde{\Delta}_{n-1},u_{n}, \zeta_{n}\right)G_{\nu}\left(x_{\nu}^{*}, u_{n},\zeta_{n}\right)^{T}\right\|+\left\|\mathbb{E}_{n-1}G_{\nu}\left(x_{\nu}^{*}, u_{n},\zeta_{n}\right)L\left(\tilde{\Delta}_{n-1},u_{n}, \zeta_{n}\right)^{T}\right\|\\
	&\quad\quad+\left\|\mathbb{E}_{n-1}L\left(\tilde{\Delta}_{n-1},u_{n}, \zeta_{n}\right)L\left(\tilde{\Delta}_{n-1},u_{n}, \zeta_{n}\right)^{T}\right\|
	\end{align*} 
	with
	\begin{align*}
	\left\|\mathbb{E}_{n-1}L\left(\tilde{\Delta}_{n-1},u_{n}, \zeta_{n}\right)G_{\nu}\left(x_{\nu}^{*}, u_{n},\zeta_{n}\right)^{T}\right\| &\leq \mathbb{E}_{n-1}\left\|L\left(\tilde{\Delta}_{n-1},u_{n}, \zeta_{n}\right)G_{\nu}\left(x_{\nu}^{*}, u_{n},\zeta_{n}\right)^{T}\right\|\\
	&\leq \mathbb{E}_{n-1}\left\|L\left(\tilde{\Delta}_{n-1},u_{n}, \zeta_{n}\right)\right\|_{2}\left\|G_{\nu}\left(x_{\nu}^{*}, u_{n},\zeta_{n}\right)^{T}\right\|_{2}\\
	&\leq \sqrt{\mathbb{E}_{n-1}\left\|L\left(\tilde{\Delta}_{n-1},u_{n}, \zeta_{n}\right)\right\|_{2}^{2}\mathbb{E}_{n-1}\left\|G_{\nu}\left(x_{\nu}^{*}, u_{n},\zeta_{n}\right)\right\|_{2}^{2}}\\
	&\leq 2\sqrt{\mathbb{E}_{n-1}\left\|G_{\nu}\left(x_{\nu}^{*}, u_{n},\zeta_{n}\right)\right\|_{2}^{2}\mathbb{E}\tilde{H}(\zeta)^{2}}\|\tilde{\Delta}_{n-1}\|_{2}
	\end{align*}
	and 
	\begin{align*}
	\left\|\mathbb{E}_{n-1}L\left(\tilde{\Delta}_{n-1},u_{n}, \zeta_{n}\right)L\left(\tilde{\Delta}_{n-1},u_{n}, \zeta_{n}\right)^{T}\right\| &\leq \mathbb{E}_{n-1}\left\|L\left(\tilde{\Delta}_{n-1},u_{n}, \zeta_{n}\right)L\left(\tilde{\Delta}_{n-1},u_{n}, \zeta_{n}\right)^{T}\right\|\\
	&\leq \mathbb{E}_{n-1}\left\|L\left(\tilde{\Delta}_{n-1},u_{n}, \zeta_{n}\right)\right\|_{2}^{2}\\
	&\leq 4\|\tilde{\Delta}_{n-1}\|_{2}^{2}\mathbb{E}\tilde{H}(\zeta)^{2}
	\end{align*}
	Combining these yields
	\begin{equation*}
	\left\|\mathbb{E}_{n-1}\tilde{\xi}_{n}\tilde{\xi}_{n}^{T}-S_{\nu}\right\|\leq 4\sqrt{\mathbb{E}_{n-1}\left\|G_{\nu}\left(x_{\nu}^{*}, u_{n},\zeta_{n}\right)\right\|_{2}^{2}\mathbb{E}\tilde{H}(\zeta)^{2}}\|\tilde{\Delta}_{n-1}\|_{2}+4\mathbb{E}\tilde{H}(\zeta)^{2}\|\tilde{\Delta}_{n-1}\|_{2}^{2}
	\end{equation*}
	Similar results hold for the trace since 
	\begin{align*}
	\left|\mathrm{tr}\left(\mathbb{E}_{n-1}L\left(\tilde{\Delta}_{n-1},u_{n}, \zeta_{n}\right)G_{\nu}\left(x_{\nu}^{*}, u_{n},\zeta_{n}\right)^{T}\right)\right| &\leq\mathbb{E}_{n-1} \left|\mathrm{tr}\left(L\left(\tilde{\Delta}_{n-1},u_{n}, \zeta_{n}\right)G_{\nu}\left(x_{\nu}^{*}, u_{n},\zeta_{n}\right)^{T}\right)\right|\\
	&\leq \mathbb{E}_{n-1}\left\|L\left(\tilde{\Delta}_{n-1},u_{n}, \zeta_{n}\right)\right\|_{2}\left\|G_{\nu}\left(x_{\nu}^{*}, u_{n},\zeta_{n}\right)^{T}\right\|_{2}
	\end{align*}
	and
	\begin{align*}
	\left|\mathrm{tr}\left(\mathbb{E}_{n-1} L\left(\Delta_{n-1},u_{n}, \zeta_{n}\right) L\left(\Delta_{n-1},u_{n}, \zeta_{n}\right)^{T}\right)\right| &\leq \mathbb{E}_{n-1}\left|\mathrm{tr}\left( L\left(\Delta_{n-1},u_{n}, \zeta_{n}\right) L\left(\Delta_{n-1},u_{n}, \zeta_{n}\right)^{T}\right)\right|\\
	&\leq \mathbb{E}_{n-1}\left\|L\left(\tilde{\Delta}_{n-1}, u_{n}, \zeta_{n}\right)\right\|_{2}^{2}
	\end{align*}
	Finally,
	\begin{align*}
	\mathbb{E}_{n-1}\left\|\tilde{\xi}_{n}\right\|_{2}^{4} & = \mathbb{E}_{n-1}\left\|L\left(\tilde{\Delta}_{n-1},u_{n}, \zeta_{n}\right)-G_{\nu}\left(x_{\nu}^{*}, u_{n},\zeta_{n}\right)\right\|_{2}^{4}\\
	&\leq 8\mathbb{E}_{n-1}\left\|L\left(\tilde{\Delta}_{n-1},u_{n}, \zeta_{n}\right)\right\|_{2}^{4}+8\mathbb{E}_{n-1}\left\|G_{\nu}\left(x_{\nu}^{*}, u_{n},\zeta_{n}\right)\right\|_{2}^{4}\\
	&\leq 8\mathbb{E}_{n-1}\left\|G_{\nu}\left(x_{\nu}^{*}, u_{n},\zeta_{n}\right)\right\|_{2}^{4}+128 \mathbb{E} \tilde{H}(\zeta)^{4} \left\|\tilde{\Delta}_{n-1}\right\|_{2}^{4}
	\end{align*}
\end{proof}
}

In \cite{zhu2020fully}, where the authors consider the stochastic first-order setting, all parameters that are related to the function $f(x)$ are hidden in a general constant and are not carefully tracked. As we are interested in dependency of the bound on the smoothing parameter $\nu$ in stochastic zeroth-order setting, we do a more careful analysis by tracking the explicit dependency of the bounds on $\nu$. The parameter $\nu$ mainly comes from four sources: $\|S_{\nu}\|$, $\mathbb{E}\|\tilde{\xi}_{p}\|$, $\mathbb{E}\tilde{\xi}_{p}\tilde{\xi}_{p}^{T}$ and $\|\tilde{\Delta}_{0}\|$. Besides, the scalar $\tilde{\Sigma}_{1}$, $\tilde{\Sigma}_{3}$ and $C_{d}=\max \left\{L, \tilde{\Sigma}_{1}^{\frac{2}{3}}, \tilde{\Sigma}_{2}^{\frac{1}{2}}, \tilde{\Sigma}_{3}^{\frac{1}{2}}, \tilde{\Sigma}_{4}^{\frac{1}{4}}, \operatorname{tr}\left(S_{\nu}\right)\right\}$ also depend on the parameter $\nu$. 

{\color{black}
For $\|\tilde{\Delta}_{0}\|_{2}$, we have
$$\|\tilde{\Delta}_{0}\|_{2}\leq \|x_{0}-x^{*}\|_{2}+\|x^{*}-x_{\nu}^{*}\|_{2}\lesssim \mathcal{O}(1+\nu)$$
and similarly
$$\|\tilde{\Delta}_{0}\|_{2}^{2}\lesssim \mathcal{O}(1+\nu+\nu^{2})$$
We are also interested how these four constants $\tilde{\Sigma}_{i}$'s and $\mathrm{tr}(S_{\nu})$ are related to the parameter $\nu$ where $\tilde{\Sigma}_{1}=4\sqrt{\mathbb{E}_{n-1}\left\|G_{\nu}\left(x_{\nu}^{*}, u_{n},\zeta_{n}\right)\right\|_{2}^{2}\mathbb{E}\tilde{H}(\zeta)^{2}}$ and $\tilde{\Sigma}_{3}=8\mathbb{E}_{n-1}\left\|G_{\nu}\left(x_{\nu}^{*}, u_{n},\zeta_{n}\right)\right\|_{2}^{4}$. Note that $$\mathbb{E}\left[\|u\|_{2}^{k}\right] \leq(d+k)^{k/2}$$
hence
\begin{align*}
\mathbb{E}\left\|G_{\nu}\left(x_{\nu}^{*},u, \zeta\right)\right\|_{2}^{2} &=\mathbb{E}\left\|uu^{T}\nabla F(x_{\nu}^{*}  + \theta\nu u, \zeta)\right\|_{2}^{2} \leq \sqrt{\mathbb{E}\left\|u\right\|_{2}^{8}\mathbb{E}\|\nabla F(x_{\nu}^{*}  + \theta\nu u, \zeta)\|_{2}^{4}}\\
&\leq \sqrt{8(d+8)^{4}\left(\mathbb{E}\left\|\nabla F(x_{\nu}^{*}+\theta\nu u,\zeta)-\nabla F(x^{*},\zeta)\right\|_{2}^{4}+\mathbb{E}\left\|\nabla F(x^{*},\zeta)\right\|_{2}^{4}\right)}\\
&\leq\sqrt{8(d+8)^{4}\left(\mathbb{E}H(\zeta)^{4}\mathbb{E}\|x_{\nu}^{*}+\theta\nu u-x^{*}\|_{2}^{4}+\mathbb{E}\left\|\nabla F(x^{*},\zeta)\right\|_{2}^{4} \right)}\\
&\lesssim \sqrt{1+\nu^{4}}\\
\mathbb{E}\left\|G_{\nu}\left(x_{\nu}^{*},u, \zeta\right)\right\|_{2}^{4} & \leq\sqrt{8(d+16)^{8}\left(\mathbb{E}H(\zeta)^{8}\mathbb{E}\|x_{\nu}^{*}+\theta\nu u-x^{*}\|_{2}^{8}+\mathbb{E}\left\|\nabla F(x^{*},\zeta)\right\|_{2}^{8} \right)}\\
&\lesssim \sqrt{1+\nu^{8}}
\end{align*}
Therefore, $\tilde{\Sigma}_{1}\lesssim (1+\nu^{4})^{\frac{1}{4}}$ and $\tilde{\Sigma}_{3}\lesssim \sqrt{1+\nu^{8}}$. Besides, by (\ref{expansionequ1}), we have
$$\mathbb{E}_{n-1}\|\tilde{\xi}_{n}\|_{2}^{2}=\mathrm{tr}\mathbb{E}_{n-1} \tilde{\xi}_{n} \tilde{\xi}_{n}^{T}=\mathrm{tr}S_{\nu}+\mathrm{tr}\tilde{\Sigma}_{\nu}\left(\tilde{\Delta}_{n-1}\right)$$
Rearranging the equation above gives us
\begin{align*}
|\mathrm{tr}S_{\nu}|&\leq \left|\mathrm{tr}\tilde{\Sigma}_{\nu}\left(\tilde{\Delta}_{n-1}\right)\right|+\mathbb{E}_{n-1}\|\tilde{\xi}_{n}\|_{2}^{2}\leq \tilde{\Sigma}_{1}\|\tilde{\Delta}_{n-1}\|_{2}+\tilde{\Sigma}_{2}\|\tilde{\Delta}_{n-1}\|_{2}^{2}+\sqrt{\mathbb{E}_{n-1}\|\tilde{\xi}_{n}\|_{2}^{4}}\\
&\leq \tilde{\Sigma}_{1}\|\tilde{\Delta}_{n-1}\|_{2}+\tilde{\Sigma}_{2}\|\tilde{\Delta}_{n-1}\|_{2}^{2}+\sqrt{\tilde{\Sigma}_{3}+\tilde{\Sigma}_{4}\left\|\tilde{\Delta}_{n-1}\right\|_{2}^{4}}\\
&\leq \tilde{\Sigma}_{1}(n-1)^{-\frac{\alpha}{2}}(1+\|\tilde{\Delta}_{0}\|_{2})+\tilde{\Sigma}_{2}(n-1)^{-\alpha}(1+\|\tilde{\Delta}_{0}\|_{2}^{2})+\sqrt{\tilde{\Sigma}_{3}+\tilde{\Sigma}_{4}(n-1)^{-2\alpha}(1+\|\tilde{\Delta}_{2}\|_{2}^{4})}\\
&\lesssim (1+\nu^{4})^{\frac{1}{4}}(1+\nu)+(1+\nu^{2})+\sqrt{(1+\nu^{8})^{\frac{1}{2}}+(1+\nu^{4})}
\end{align*}
Note that for any $a,b\in\mathbb{R}$, 
$$\max\{a,b\}=\frac{|a+b|}{2}+\frac{|a-b|}{2}\leq |a|+|b|$$
then by induction
\begin{align*}
C_{d} &\leq  |L|+ |\tilde{\Sigma}_{1}^{\frac{2}{3}}|+ |\tilde{\Sigma}_{2}^{\frac{1}{2}}|+ |\tilde{\Sigma}_{3}^{\frac{1}{2}}|+ |\tilde{\Sigma}_{4}^{\frac{1}{4}}|+ |\operatorname{tr}\left(S_{\nu}\right)| \\
&\lesssim 1+(1+\nu^{4})^{\frac{1}{6}}+(1+\nu^{8})^{\frac{1}{4}}+(1+\nu^{4})^{\frac{1}{4}}(1+\nu)+(1+\nu^{2})+\sqrt{(1+\nu^{8})^{\frac{1}{2}}+(1+\nu^{4})}
\end{align*}
}
\begin{lemma}\label{FOCMElma5}
Let $a_{k}=\lfloor C k^{\beta}\rfloor,$ where $C>0, \beta>\frac{1}{1-\alpha} .$ Set step size at the $i$-th iteration as $\eta_{i}=\eta i^{-\alpha}$ with $\frac{1}{2}<\alpha<1 .$ Then under Assumption \ref{simpasp4.1} and \ref{simpasp4.2}, we have
\begin{equation}\label{FOCMEequ38}\color{black}
\begin{aligned}
\mathbb{E}\left\|\tilde{\Sigma}_{n}-\tilde{V}\right\| \lesssim & \sqrt{1+\tilde{\Sigma}_{1}+C_{d}^{3}+\nu+(\nu+\nu^{2}+\nu^{3})\tilde{\Sigma}_{1}} M^{-\frac{\alpha \beta}{4}}\\
&~~~~~~~~+\sqrt{(1+\nu^{2}+\nu^{4})}M^{-\frac{1}{2}}+\sqrt{1+\nu+\nu^{2}}M^{\frac{(\alpha-1) \beta+1}{2}}
\end{aligned}
\end{equation}
where $M$ is the number of batches such that $a_{M} \leq n<a_{M+1}$.
\end{lemma}
\begin{proof}
Recall that
$$\tilde{\Sigma}_{n}=\left(\sum_{i=1}^{n} l_{i}\right)^{-1} \sum_{i=1}^{n}\left(\sum_{k=t_{i}}^{i} U_{k}-l_{i} \bar{U}_{n}\right)\left(\sum_{k=t_{i}}^{i} U_{k}-l_{i} \bar{U}_{n}\right)^{T}$$
Using triangle inequality we have
\begin{equation}\label{FOCMEequ39}
\begin{aligned}
    \mathbb{E}\left\|\tilde{\Sigma}_{n}-\tilde{V}\right\| &\leq \mathbb{E}\left\|\left(\sum_{i=1}^{n} l_{i}\right)^{-1} \sum_{i=1}^{n}\left(\sum_{k=t_{i}}^{i} U_{k}\right)\left(\sum_{k=t_{i}}^{i} U_{k}\right)^{T}-\tilde{V}\right\|\\
    &+\mathbb{E}\left\|\left(\sum_{i=1}^{n} l_{i}\right)^{-1} \sum_{i=1}^{n} l_{i}^{2} \bar{U}_{n} \bar{U}_{n}^{T}\right\|+2 \mathbb{E}\left\|\left(\sum_{i=1}^{n} l_{i}\right)^{-1} \sum_{i=1}^{n}\left(\sum_{k=t_{i}}^{i} U_{k}\right)\left(l_{i} \bar{U}_{n}\right)^{T}\right\|
\end{aligned}    
\end{equation}
We will show that all these three terms in
(\ref{FOCMEequ39}) are bounded in next few lemmas, which implies Lemma \ref{FOCMElma5}.
\end{proof}
In Lemma \ref{FOCMElma6} we show that $\widehat{S}$ defined by 
$$\widehat{S}=\left(\sum_{i=1}^{n} l_{i}\right)^{-1} \sum_{i=1}^{n}\left(\sum_{k=t_{i}}^{i} \tilde{\xi}_{k}\right)\left(\sum_{k=t_{i}}^{i} \tilde{\xi}_{k}\right)^{T}$$ 
converges to the covariance of $G_{\nu}\left(x_{\nu}^{*},u,\zeta\right),$ i.e. $S_{\nu}$. We also require the intermediate result. 

\begin{lemma}\label{surproblm3}
For the surrogate problem (\ref{surgprob1}), the term $\left\|\mathbb{E} \tilde{\xi}_{p_{1}} \tilde{\xi}_{p_{2}}^{T} \tilde{\xi}_{p_{3}} \tilde{\xi}_{p_{4}}^{T}\right\|$ can be bounded by constant $C$ for any $p_{r}, r \in\{1,2,3,4\}$.
\end{lemma}
\begin{proof}
Note that $\operatorname{rank} \left(\tilde{\xi}_{p_{1}} \tilde{\xi}_{p_{2}}^{T} \tilde{\xi}_{p_{3}} \tilde{\xi}_{p_{4}}^{T}\right)=1$, for any general choice of $p_{1}, \dots,p_{4}$. Hence, we have
$$\left|\operatorname{tr}\left(\tilde{\xi}_{p_{1}} \tilde{\xi}_{p_{2}}^{T} \tilde{\xi}_{p_{3}} \tilde{\xi}_{p_{4}}^{T}\right)\right|=\left\|\tilde{\xi}_{p_{1}} \tilde{\xi}_{p_{2}}^{T} \tilde{\xi}_{p_{3}} \tilde{\xi}_{p_{4}}^{T}\right\|.$$
By Jensen's inequality, we then have 
\begin{align*}
    \|\mathbb{E}\tilde{\xi}_{p_{1}} \tilde{\xi}_{p_{2}}^{T} \tilde{\xi}_{p_{3}} \tilde{\xi}_{p_{4}}^{T}\| &\leq \mathbb{E}\|\tilde{\xi}_{p_{1}} \tilde{\xi}_{p_{2}}^{T} \tilde{\xi}_{p_{3}} \tilde{\xi}_{p_{4}}^{T}\|=\mathbb{E}\left|\operatorname{tr}\left(\tilde{\xi}_{p_{1}} \tilde{\xi}_{p_{2}}^{T} \tilde{\xi}_{p_{3}} \tilde{\xi}_{p_{4}}^{T}\right)\right|\\
    &\leq \frac{1}{4} \mathbb{E}\left[\left\|\tilde{\xi}_{p_{1}}\right\|^{4}+\left\|\tilde{\xi}_{p_{2}}\right\|^{4}+\left\|\tilde{\xi}_{p_{3}}\right\|^{4}+\left\|\tilde{\xi}_{p_{4}}\right\|^{4}\right]\\ 
    &\leq \tilde{\Sigma}_{3}+\tilde{\Sigma}_{4} \mathbb{E}\left\|\tilde{\Delta}_{p_{1}}\right\|^{4} \lesssim C_{d}^{2}+C_{d}^{6} p_{1}^{-2 \alpha}\lesssim C
\end{align*}
where $C_{d}:=\max \left\{L, \tilde{\Sigma}_{1}^{\frac{2}{3}}, \tilde{\Sigma}_{2}^{\frac{1}{2}}, \tilde{\Sigma}_{3}^{\frac{1}{2}}, \tilde{\Sigma}_{4}^{\frac{1}{4}}, \operatorname{tr}(S_{\nu})\right\}$ is given by Equation (9) in \cite{chen2016statistical}.
\end{proof}

\begin{lemma}\label{FOCMElma6}
Under the same conditions in Lemma \ref{FOCMElma5}, we have 
\begin{equation}\label{FOCMEequ40}\color{black}
    \mathbb{E}\left\|\widehat{S}-S_{\nu}\right\| \lesssim \sqrt{1+\tilde{\Sigma}_{1}+C_{d}^{3}+\nu+(\nu+\nu^{2}+\nu^{3})\tilde{\Sigma}_{1}} M^{-\frac{\alpha \beta}{4}}+\sqrt{(1+\nu^{2}+\nu^{4})}M^{-\frac{1}{2}}.
\end{equation}
when $a_{M} \leq n<a_{M+1}$.
\end{lemma}
\begin{proof}
Since $\widehat{S}-S_{\nu}$ is symmetric, following inequality holds
\begin{equation}\label{FOCMEequ41}
    \mathbb{E}\|\widehat{S}-S_{\nu}\|=\mathbb{E}\left|\lambda_{\max }(\widehat{S}-S_{\nu})\right|=\mathbb{E} \sqrt{\lambda_{\max }(\widehat{S}-S_{\nu})^{2}}
\end{equation}
Note that $(\widehat{S}-S_{\nu})^{2}$ is positive definite. We have
\begin{equation}\label{FOCMEequ42}
    \mathbb{E}\|\widehat{S}-S_{\nu}\| \leq \mathbb{E} \sqrt{\operatorname{tr}(\widehat{S}-S_{\nu})^{2}} \leq \sqrt{\operatorname{tr} \mathbb{E}(\widehat{S}-S_{\nu})^{2}} \leq \sqrt{d\left\|\mathbb{E}(\widehat{S}-S_{\nu})^{2}\right\|}
\end{equation}
For any $p>q, \mathbb{E}\tilde{\xi}_{p} \tilde{\xi}_{q}^{T}=\mathbb{E} \tilde{\xi}_{q} \mathbb{E}_{p-1} \tilde{\xi}_{p}^{T}=0$. Using
$\mathbb{E}_{n-1} \tilde{\xi}_{n} \tilde{\xi}_{n}^{T}=S_{\nu}+\tilde{\Sigma}_{\nu}\left(\tilde{\Delta}_{n-1}\right)$ in Equation (\ref{expansionequ1}), we have
\begin{equation}\label{FOCMEequ43}
\begin{aligned}
\mathbb{E} \widehat{S}&=\left(\sum_{i=1}^{n} l_{i}\right)^{-1} \sum_{i=1}^{n}\left[\sum_{k=t_{i}}^{i} \mathbb{E} \tilde{\xi}_{k}\tilde{\xi}_{k}^{T}+\sum_{p \neq q \in\left[t_{i}, i\right]} \mathbb{E} \tilde{\xi}_{p} \tilde{\xi}_{q}\right]\\
&=S_{\nu}+\left(\sum_{i=1}^{n} l_{i}\right)^{-1} \sum_{i=1}^{n} \sum_{k=t_{i}}^{i} \mathbb{E} \tilde{\Sigma}_{\nu}\left(\tilde{\Delta}_{k-1}\right)
\end{aligned}
\end{equation}
Define 
$$(*)\coloneqq\left(\sum_{i=1}^{n} l_{i}\right)^{-1} \sum_{i=1}^{n} \sum_{k=t_{i}}^{i} \mathbb{E} \tilde{\Sigma}_{\nu}\left(\tilde{\Delta}_{k-1}\right).$$ 
Then
\begin{equation}\label{FOCMEequ44}
\begin{aligned}
\left\|\mathbb{E}(\widehat{S}-S_{\nu})^{2}\right\|&=\left\|\mathbb{E} \widehat{S}^{2}+S_{\nu}^{2}-S_{\nu} \mathbb{E} \widehat{S}-\mathbb{E} \hat{S} S_{\nu}\right\|\\
&=\left\|\mathbb{E}^{2}-S_{\nu}^{2}-S_{\nu}(*)-(*) S_{\nu}\right\|\\
&\leq\left\|\mathbb{E} \widehat{S}^{2}-S_{\nu}^{2}\right\|+2\|S_{\nu}\|\|(*)\|.
\end{aligned}
\end{equation}
Based on Equation (\ref{expansionequ2}) and $\lim_{M \rightarrow \infty} \frac{\sum_{i=1}^{n} l_{i}}{\sum_{i=1}^{a_{M+1}-1} l_{i}}=1$ we have,
\begin{equation}\label{FOCMEequ45}
\begin{aligned}
\|(*)\| &\leq\left(\sum_{i=1}^{n} l_{i}\right)^{-1} \sum_{i=1}^{n} \sum_{k=t_{i}}^{i} \mathbb{E}\left\|\tilde{\Sigma}_{\nu}\left(\tilde{\Delta}_{k-1}\right)\right\|\\
&\lesssim\left(\sum_{i=1}^{a_{M+1}-1} l_{i}\right)^{-1} \sum_{m=1}^{M} \sum_{i=a_{m}}^{a_{m+1}-1} \sum_{k=a_{m}}^{i}  \mathbb{E}\left(\tilde{\Sigma}_{1}\left\|\tilde{\Delta}_{k-1}\right\|+\tilde{\Sigma}_{2}\left\|\tilde{\Delta}_{k-1}\right\|^{2}\right).
\end{aligned}
\end{equation}
According to bounds on $\mathbb{E}\left\|\tilde{\Delta}_{k-1}\right\|$ and $\mathbb{E}\left\|\tilde{\Delta}_{k-1}\right\|^{2}$, we have
\begin{align*}
    \|(*)\| &\lesssim\left(\sum_{i=1}^{a_{M+1}-1} l_{i}\right)^{-1} \sum_{m=1}^{M} \sum_{i=a_{m}}^{a_{m+1}-1} \sum_{k=a_{m}}^{i}\left[\tilde{\Sigma}_{1}(k-1)^{-\alpha / 2}\left(1+\left\|\tilde{\Delta}_{0}\right\|\right)+(k-1)^{-\alpha}\left(1+\left\|\tilde{\Delta}_{0}\right\|^{2}\right)\right]\\
    &\leq\left(\sum_{i=1}^{a_{M+1}-1} l_{i}\right)^{-1} \sum_{m=1}^{M} \sum_{i=a_{m}}^{a_{m+1}-1} l_{i}\left[\tilde{\Sigma}_{1}\left(a_{m}-1\right)^{-\alpha / 2}\left(1+\left\|\tilde{\Delta}_{0}\right\|\right)+\left(a_{m}-1\right)^{-\alpha}\left(1+\left\|\tilde{\Delta}_{0}\right\|^{2}\right)\right]\\
    &\lesssim\left(\sum_{m=1}^{M} \sum_{i=a_{m}}^{a_{m+1}-1} l_{i}\right)^{-1} \sum_{m=1}^{M} \sum_{i=a_{m}}^{a_{m+1}-1} l_{i}\tilde{\Sigma}_{1}\left(a_{m}-1\right)^{-\alpha / 2}\left(1+\left\|\tilde{\Delta}_{0}\right\|\right).
\end{align*}
Since $\sum_{i=a_{m}}^{a_{m+1}-1} l_{i}=\left(a_{m+1}-a_{m}\right)\left(a_{m+1}-a_{m}+1\right) / 2 \asymp n_{m}^{2}, n_{m} \asymp m^{\beta-1}$, $a_{m} \asymp m^{\beta}$ and
$$\|\tilde{\Delta}_{0}\|=\|x_{0}-x_{\nu}^{*}\|\leq \|x_{0}-x^{*}\|+\|x^{*}-x_{\nu}^{*}\|\leq \|\Delta_{0}\|+\frac{\nu L(d+3)^{3 / 2}}{2 \mu},$$
we have
\begin{align*}
\|(*)\| &\lesssim \left(\sum_{m=1}^{M} \sum_{i=a_{m}}^{a_{m+1}-1} l_{i}\right)^{-1} \sum_{m=1}^{M} \sum_{i=a_{m}}^{a_{m+1}-1} l_{i}\left(a_{m}-1\right)^{-\alpha / 2}\left(1+\left\|\tilde{\Delta}_{0}\right\|\right)\tilde{\Sigma}_{1}\\
&\leq \left(\sum_{m=1}^{M} \sum_{i=a_{m}}^{a_{m+1}-1} l_{i}\right)^{-1} \sum_{m=1}^{M} \sum_{i=a_{m}}^{a_{m+1}-1} l_{i}\left(a_{m}-1\right)^{-\alpha / 2}\left(1+\left\|\Delta_{0}\right\|+\frac{\nu L(d+3)^{3 / 2}}{2 \mu}\right)\tilde{\Sigma}_{1}\\
&\lesssim \left(\sum_{m=1}^{M} \sum_{i=a_{m}}^{a_{m+1}-1} l_{i}\right)^{-1} \sum_{m=1}^{M} \sum_{i=a_{m}}^{a_{m+1}-1} l_{i}\left(a_{m}-1\right)^{-\alpha / 2}\left(1+\nu\right)\tilde{\Sigma}_{1}.
\end{align*}
Hence, we obtain
\begin{equation}\label{FOCMEequ47}
    \|(*)\| \lesssim M^{-\frac{\alpha \beta}{2}}\left(1+\nu\right)\tilde{\Sigma}_{1}.
\end{equation}
Now, note that $\mathbb{E} \tilde{\xi}_{p_{1}} \tilde{\xi}_{p_{2}}^{T} \tilde{\xi}_{p 3} \tilde{\xi}_{p_{4}}^{T}$ is nonzero if and only if for any $r$ there exist $r^{\prime} \neq r$ such that $p_{r}=p_{r^{\prime}}, r, r^{\prime} \in\{1,2,3,4\} .$ There are two cases we can consider. First is $p_{1}=p_{3} \neq p_{2}=p_{4}$
or $p_{1}=p_{4} \neq p_{2}=p_{3} .$ This requires $i$ and $j$ in the same block. Second case is that $p_{1}=p_{2}$ and $p_{3}=p_{4}$. Therefore,
\begin{equation}\label{FOCMEequ48}
\begin{aligned}
\mathbb{E} \widehat{S}^{2}&=\mathbb{E}\left(\sum_{i=1}^{n} l_{i}\right)^{-2} \sum_{1 \leq i, j \leq n}\left(\sum_{k=t_{i}}^{i} \tilde{\xi}_{k}\right)\left(\sum_{k=t_{i}}^{i} \tilde{\xi}_{k}\right)^{T}\left(\sum_{k=t_{j}}^{j} \tilde{\xi}_{k}\right)\left(\sum_{k=t_{j}}^{j} \tilde{\xi}_{k}\right)^{T}\\
&=\left(\sum_{i=1}^{n} l_{i}\right)^{-2} I+\left(\sum_{i=1}^{n} l_{i}\right)^{-2} I I,
\end{aligned}
\end{equation}
where
\begin{align*}
    I&=\mathbb{E} \sum_{m=1}^{M-1} \sum_{i=a_{m}}^{a_{m+1}-1}\left[2\sum_{j=a_{m}} \sum_{ a_{m} \leq p_{1} \neq p_{2} \leq j}\left(\tilde{\xi}_{p_{1}}\tilde{\xi}_{p_{2}}^{T} \tilde{\xi}_{p_{1}} \tilde{\xi}_{p_{2}}^{T}+\tilde{\xi}_{p_{1}} \tilde{\xi}_{p_{2}}^{T} \tilde{\xi}_{p_{2}} \tilde{\xi}_{p_{1}}^{T}\right)\right.\\
    &\quad\quad\quad\quad\left.+\sum_{a_{m} \leq p_{1} \neq p_{2} \leq i}\left(\tilde{\xi}_{p_{1}} \tilde{\xi}_{p_{2}}^{T} \tilde{\xi}_{p_{1}} \tilde{\xi}_{p_{2}}^{T}+\tilde{\xi}_{p_{1}} \tilde{\xi}_{p_{2}}^{T} \tilde{\xi}_{p_{2}} \tilde{\xi}_{p_{1}}^{T}\right)\right]\\
    &\quad\quad\quad\quad+\mathbb{E} \sum_{i=a_{M}}^{n}\left[2 \sum_{j=a_{M}}^{i-1} \sum_{a_{M} \leq p_{1} \neq p_{2} \leq j}\left(\tilde{\xi}_{p_{1}} \tilde{\xi}_{p_{2}}^{T} \tilde{\xi}_{p_{1}} \tilde{\xi}_{p_{2}}^{T}+\tilde{\xi}_{p_{1}} \tilde{\xi}_{p_{2}}^{T} \tilde{\xi}_{p_{2}} \tilde{\xi}_{p_{1}}^{T}\right)\right.\\
    &\quad\quad\quad\quad\left.+\sum_{a_{M} \leq p_{1} \neq p_{2} \leq i}\left(\tilde{\xi}_{p_{1}} \tilde{\xi}_{p_{2}}^{T} \tilde{\xi}_{p_{1}} \tilde{\xi}_{p_{2}}^{T}+\tilde{\xi}_{p_{1}} \tilde{\xi}_{p_{2}}^{T} \tilde{\xi}_{p_{2}} \tilde{\xi}_{p_{1}}^{T}\right)\right],
\end{align*}
and
$$I I=\sum_{i=1}^{n} \sum_{j=1}^{n} \sum_{p=t_{i}}^{i} \sum_{q=t_{j}}^{j} \mathbb{E} \tilde{\xi}_{p} \tilde{\xi}_{p}^{T} \tilde{\xi}_{q} \tilde{\xi}_{q}^{T}.$$
Note that by Lemma~\ref{surproblm3}, the term $\left\|\mathbb{E} \tilde{\xi}_{p_{1}} \tilde{\xi}_{p_{2}}^{T} \tilde{\xi}_{p_{3}} \tilde{\xi}_{p_{4}}^{T}\right\|$ can be bounded by constant $\mathrm{C}$ for any $p_{r}, r \in\{1,2,3,4\} .$ Then we can bound $I$ as follows,
\begin{equation}\label{FOCMEequ49}
\begin{aligned}
\|I\| &\leq \sum_{m=1}^{M} \sum_{i=a_{m}}^{a_{m+1}-1}\left[2 \sum_{j=a_{m}}^{i-1} \sum_{a_{m} \leq p_{1} \neq p_{2} \leq j}(C+C)+\sum_{a_{m} \leq p_{1} \neq p_{2} \leq i}(C+C)\right]\\
&\lesssim \sum_{m=1}^{M} \sum_{i=a_{m}}^{a_{m+1}-1}\left(1 * 2+2 * 3+\ldots+\left(l_{i}-1\right) * l_{i}\right)\\
&\lesssim \sum_{m=1}^{M} \sum_{i=a_{m}}^{a_{m+1}-1} l_{i}^{3} \lesssim \sum_{m=1}^{M} n_{m}^{4}.
\end{aligned}
\end{equation}
Since $\sum_{i=1}^{n} l_{i} \asymp \sum_{m=1}^{M} \sum_{i=a_{m}}^{a_{m+1}-1} l_{i} \asymp \sum_{m=1}^{M} n_{m}^{2}$ and $\frac{n_{M}^{2}}{\sum_{m=1}^{M} n_{m}^{2}} \lesssim M^{-1},$ we have,
\begin{equation}\label{FOCMEequ50}
    \left(\sum_{i=1}^{n} l_{i}\right)^{-2}\|I\| \lesssim \frac{\sum_{m=1}^{M} n_{m}^{4}}{\left(\sum_{m=1}^{M} n_{m}^{2}\right)^{2}} \lesssim \frac{\max_{1\leq m \leq M}  n_{m}^{2}}{\sum_{m=1}^{M} n_{m}^{2}} \lesssim M^{-1}.
\end{equation}
Next, notice that $\sum_{i=1}^{n} \sum_{j=1}^{n} \sum_{p=t_{i}}^{i} \sum_{q=t_{j}}^{j} 1=\left(\sum_{i=1}^{n} l_{i}\right)^{2} .$ Hence, we have
\begin{equation}\label{FOCMEequ51}
\begin{aligned}
\left\|\left(\sum_{i=1}^{n} l_{i}\right)^{-2} I I-S_{\nu}^{2}\right\| &\leq\left(\sum_{i=1}^{n} l_{i}\right)^{-2} \sum_{i=1}^{n} \sum_{j=1}^{n} \sum_{p=t_{i}}^{i} \sum_{q=t_{j}}^{j}\left\|\mathbb{E} \tilde{\xi}_{p} \tilde{\xi}_{p}^{T} \tilde{\xi}_{q} \tilde{\xi}_{q}^{T}-S_{\nu}^{2}\right\|\\
&\lesssim\left(\sum_{i=1}^{a_{M+1}-1} l_{i}\right)^{-2} \sum_{m=1}^{M} \sum_{k=1}^{M} \sum_{i=a_{m}}^{a_{m+1}-1} \sum_{j=a_{k}}^{a_{k+1}-1} \sum_{p=a_{m}}^{i} \sum_{q=a_{k}}^{j}\left\|\mathbb{E} \tilde{\xi}_{p} \tilde{\xi}_{p}^{T} \tilde{\xi}_{q} \tilde{\xi}_{q}^{T}-S_{\nu}^{2}\right\|\\
&=\left(\sum_{i=1}^{a_{M+1}-1} l_{i}\right)^{-2} I I I+\left(\sum_{i=1}^{a_{M+1}-1} l_{i}\right)^{-2} I V.
\end{aligned}
\end{equation}
We consider two cases here. One is when $p$ and $q$ are in the same block. In this case, we have
$$I I I=\sum_{m=1}^{M} \sum_{i=a_{m}}^{a_{m+1}-1} \sum_{j=a_{m}}^{1 a_{m+1}-1} \sum_{p=a_{m}}^{i} \sum_{q=a_{m}}^{j}\left\|\mathbb{E} \tilde{\xi}_{p} \tilde{\xi}_{p}^{T} \tilde{\xi}_{q} \tilde{\xi}_{q}^{T}-S_{\nu}^{2}\right\|.$$
Here $\left\|\mathbb{E} \tilde{\xi}_{p_{1}} \tilde{\xi}_{p_{2}}^{T} \tilde{\xi}_{p_{3}} \tilde{\xi}_{p_{4}}^{T}\right\|$ is still bounded by constant $C$. Then we have,
\begin{equation}\label{FOCMEequ52}
\begin{aligned}
\left(\sum_{i=1}^{a_{M+1}-1} l_{i}\right)^{-2} I I I &\leq\left(\sum_{i=1}^{a_{M+1}-1} l_{i}\right)^{-2} \sum_{m=1}^{M} \sum_{i=a_{m}}^{a_{m+1}-1} \sum_{j=a_{m}}^{a_{m+1}-1} \sum_{p=a_{m}}^{i} \sum_{q=a_{m}}^{j}\left(C+\left\|S_{\nu}^{2}\right\|\right)\\
&\lesssim\left(\sum_{i=1}^{a_{M+1}-1} l_{i}\right)^{-2} \sum_{m=1}^{M}\left(\sum_{i=a_{m}}^{a_{m+1}-1} l_{i}\right)^{2}(1+\nu^{2}+\nu^{4})\\
&\lesssim \frac{\sum_{m=1}^{M} n_{m}^{4}}{\left(\sum_{m=1}^{M} n_{m}^{2}\right)^{2}}(1+\nu^{2}+\nu^{4}) \lesssim \frac{\max _{1 \leq m \leq M} n_{m}^{2}}{\sum_{m=1}^{M} n_{m}^{2}}(1+\nu^{2}+\nu^{4})\\ 
&\lesssim (1+\nu^{2}+\nu^{4})M^{-1}.
\end{aligned}
\end{equation}
The next case is when $p$ and $q$ are in different blocks. In this case we have
$$I V=\sum_{m \neq k} \sum_{j=a_{k}}^{a_{k+1}-1} \sum_{i=a_{m}}^{a_{m+1}-1} \sum_{q=a_{k}}^{j} \sum_{p=a_{m}}^{i}\left\|\mathbb{E}\tilde{\xi}_{p} \tilde{\xi}_{p}^{T} \tilde{\xi}_{q} \tilde{\xi}_{q}^{T}-S_{\nu}^{2}\right\|.$$
For $p>q,$ we have
\begin{equation}\label{FOCMEequ53}
    \left\|\mathbb{E} \tilde{\xi}_{p} \tilde{\xi}_{p}^{T} \tilde{\xi}_{q} \tilde{\xi}_{q}^{T}-S_{\nu}^{2}\right\| \lesssim C_{d}^{3}q^{-\alpha / 2} \lesssim C_{d}^{3} p^{-\alpha / 2}+C_{d}^{3} q^{-\alpha / 2}.
\end{equation}
Hence we obtain
\begin{equation}\label{FOCMEequ54}
\begin{aligned}
\left(\sum_{i=1}^{a_{M+1}-1} l_{i}\right)^{-2} I V&=\left(\sum_{i=1}^{a_{M+1}-1} l_{i}\right)^{-2} \sum_{m \neq k} \sum_{j=a_{k}}^{a_{k+1}-1} \sum_{i=a_{m}}^{a_{m+1}-1} \sum_{q=a_{k}}^{j} \sum_{p=a_{m}}^{i}\left\|\mathbb{E} \tilde{\xi}_{p} \tilde{\xi}_{p}^{T} \tilde{\xi}_{q} \tilde{\xi}_{q}^{T}-S_{\nu}^{2}\right\|\\
&\lesssim\left(\sum_{i=1}^{a_{M+1}-1} l_{i}\right)^{-2} \sum_{k=1}^{M} \sum_{j=a_{k}}^{a_{k+1}-1} \sum_{q=a_{k}}^{j} \sum_{m=1}^{M} \sum_{i=a_{m}}^{a_{m+1}-1} \sum_{p=a_{m}}^{i}\left(p^{-\alpha / 2}+q^{-\alpha / 2}\right)C_{d}^{3}\\
&\leq 2C_{d}^{3}\left(\sum_{i=1}^{a_{M+1}-1} l_{i}\right)^{-1} \sum_{m=1}^{M} \sum_{i=a_{m}}^{a_{m+1}-1} l_{i} a_{m}^{-\alpha / 2}.
\end{aligned}
\end{equation}
Further using that $\sum_{i=1}^{a_{M+1}-1} l_{i} \asymp \sum_{m=1}^{M} n_{m}^{2},$ we have
\begin{equation}\label{FOCMEequ55}
\left(\sum_{i=1}^{a_{M+1}-1} l_{i}\right)^{-2} I V \lesssim C_{d}^{3} \left(\sum_{m=1}^{M} n_{m}^{2}\right)^{-1} \sum_{m=1}^{M} a_{m}^{-\alpha / 2} n_{m}^{2} \asymp C_{d}^{3} M^{-\frac{\alpha \beta}{2}}.
\end{equation}
Combining (\ref{FOCMEequ51}), (\ref{FOCMEequ52}) and (\ref{FOCMEequ55}), we have
\begin{equation}\label{FOCMEequ56}
\begin{aligned}
\left\|\left(\sum_{i=1}^{n} l_{i}\right)^{-2} I I-S_{\nu}^{2}\right\| &\lesssim\left(\sum_{i=1}^{a_{M+1}-1} l_{i}\right)^{-2} I I I+\left(\sum_{i=1}^{a_{M+1}-1} l_{i}\right)^{-2} I V\\ 
&\lesssim (1+\nu^{2}+\nu^{4})M^{-1}+ C_{d}^{3}M^{-\frac{\alpha \beta}{2}}.
\end{aligned}
\end{equation}
Then combining (\ref{FOCMEequ48}), (\ref{FOCMEequ50}) and (\ref{FOCMEequ56}) we have,
\begin{equation}\label{FOCMEequ57}
\begin{aligned}
\left\|\mathbb{E} \widehat{S}^{2}-S_{\nu}^{2}\right\| &\leq\left(\sum_{i=1}^{n} l_{i}\right)^{-2}\left\|I I-S_{\nu}^{2}\right\|+\left(\sum_{i=1}^{n} l_{i}\right)^{-2}\|I\|\\
&\lesssim (1+\nu^{2}+\nu^{4})M^{-1}+\left(1+C_{d}^{3}+\nu\right)M^{-\frac{\alpha \beta}{2}}.
\end{aligned}
\end{equation}
Finally, putting things together, we obtain the result
\begin{equation}\label{FOCMEequ58}\color{black}
\begin{aligned}
\mathbb{E}\|\widehat{S}-S_{\nu}\| &\lesssim \sqrt{\left\|\mathbb{E} \widehat{S}^{2}-S_{\nu}^{2}\right\|+2\|S_{\nu}\|\|(*)\|}\\ 
&\lesssim \sqrt{1+\tilde{\Sigma}_{1}+C_{d}^{3}+\nu+(\nu+\nu^{2}+\nu^{3})\tilde{\Sigma}_{1}} M^{-\frac{\alpha \beta}{4}}+\sqrt{(1+\nu^{2}+\nu^{4})}M^{-\frac{1}{2}}.
\end{aligned}
\end{equation}
\end{proof}
\begin{lemma}\label{FOCMElma7}
Under the same conditions in Lemma \ref{FOCMElma5}, we have
\begin{equation}\label{FOCMEequ59}\color{black}
\begin{aligned}
\mathbb{E}\left\|\left(\sum_{i=1}^{n} l_{i}\right)^{-1} \sum_{i=1}^{n}\left(\sum_{k=t_{i}}^{i} U_{k}\right)\left(\sum_{k=t_{i}}^{i} U_{k}\right)^{T}-\tilde{V}\right\|
\lesssim & \sqrt{1+\tilde{\Sigma}_{1}+C_{d}^{3}+\nu+(\nu+\nu^{2}+\nu^{3})\tilde{\Sigma}_{1}} M^{-\frac{\alpha \beta}{4}}\\ &+\sqrt{(1+\nu^{2}+\nu^{4})}M^{-\frac{1}{2}}+\sqrt{1+\nu+\nu^{2}}M^{\frac{(\alpha-1) \beta+1}{2}},
\end{aligned}
\end{equation}
where $a_{M} \leq n<a_{M+1}$.
\end{lemma}
\begin{proof}
Recall that 
$$U_{k}=\left(I-\eta_{k} A\right) U_{k-1}+\eta_{k} \tilde{\xi}_{k}=Y_{s}^{k} U_{s}+\sum_{p=s+1}^{k} Y_{p}^{k} \eta_{p} \epsilon_{p}.$$
For $k \in\left[t_{i}, i\right]$ we have
$$U_{k}=Y_{t_{i}-1}^{k} U_{t_{i}-1}+\sum_{p=t_{i}}^{k} Y_{p}^{k} \eta_{p} \tilde{\xi}_{p}.$$
Recalling the definition of $S_{j}^{k}$ we have
$$\sum_{k=t_{i}}^{i} U_{k}=\sum_{k=t_{i}}^{i}\left(Y_{t_{i}-1}^{k} U_{t_{i}-1}+\sum_{p=t_{i}}^{k} Y_{p}^{k} \eta_{p} \tilde{\xi}_{p}\right)=S_{t_{i}-1}^{i} U_{t_{i}-1}+\sum_{p=t_{i}}^{i}\left(I+S_{p}^{i}\right) \eta_{p} \tilde{\xi}_{p}.$$
Then we have the following expansion:
\begin{equation}\label{FOCMEequ60}
\begin{aligned}
&\left(\sum_{i=1}^{n} l_{i}\right)^{-1} \sum_{i=1}^{n}\left(\sum_{k=t_{i}}^{i} U_{k}\right)\left(\sum_{k=t_{i}}^{i} U_{k}\right)^{T}\\
&=\left(\sum_{i=1}^{n} l_{i}\right)^{-1} \sum_{i=1}^{n}\left(S_{t_{i}-1}^{i} U_{t_{i}-1}+\sum_{p=t_{i}}^{i}\left(I+S_{p}^{i}\right) \eta_{p} \tilde{\xi}_{p}\right)\left(S_{t_{i}-1}^{i} U_{t_{i}-1}+\sum_{p=t_{i}}^{i}\left(I+S_{p}^{i}\right) \eta_{p} \tilde{\xi}_{p}\right)^{T}\\
&=\left(\sum_{i=1}^{n} l_{i}\right)^{-1} \sum_{i=1}^{n}\left(A^{-1}\left(\sum_{n=t}^{i} \tilde{\xi}_{p}\right)\left(\sum_{n=t_{i}}^{i} \tilde{\xi}_{p}\right)^{T} A^{-1}+B_{i} A_{i}^{T}+A_{i} B_{i}^{T}+B_{i} B_{i}^{T}\right)
\end{aligned}
\end{equation}
where $A_{i}=\sum_{p=t_{i}}^{i} A^{-1} \tilde{\xi}_{p}, B_{i}=S_{t_{i}-1}^{i} U_{t_{i}-1}+\sum_{p=t_{i}}^{i}\left(\eta_{p} S_{p}^{i}+\eta_{p} I-A^{-1}\right) \tilde{\xi}_{p}$. We can then have
\begin{equation}\label{FOCMEequ61}
\begin{aligned}
&\mathbb{E}\left\|\left(\sum_{i=1}^{n} l_{i}\right)^{-1} \sum_{i=1}^{n}\left(\sum_{k=t_{i}}^{i} U_{k}\right)\left(\sum_{k=t_{i}}^{i} U_{k}\right)^{T}-\tilde{V}\right\|\\
\lesssim & E\left\|\left(\sum_{i=1}^{n} l_{i}\right)^{-1} \sum_{i=1}^{n}\left(A^{-1}\left(\sum_{p=t_{i}}^{i} \tilde{\xi}_{p}\right)\left(\sum_{p=t_{i}}^{i} \tilde{\xi}_{p}\right)^{T} A^{-1}-\tilde{V}\right)\right\|\\
&+\mathbb{E}\left\|\left(\sum_{i=1}^{n} l_{i}\right)^{-1} \sum_{i=1}^{n} B_{i} A_{i}^{T}\right\|+\mathbb{E}\left\|\left(\sum_{i=1}^{n} l_{i}\right)^{-1} \sum_{i=1}^{n} B_{i} B_{i}^{T}\right\|\\
=&I+I I+I I I
\end{aligned}
\end{equation}
It suffices to show that all three parts above are bounded. Recall that 
$$\widehat{S}=\left(\sum_{i=1}^{n} l_{i}\right)^{-1} \sum_{i=1}^{n}\left(\sum_{k=t_{i}}^{i} \tilde{\xi}_{k}\right)\left(\sum_{k=t_{i}}^{i} \tilde{\xi}_{k}\right)^{T},$$ 
and $\tilde{V}=A^{-1} S_{\nu} A^{-1}$. We can bound $I$ using Lemma \ref{FOCMElma6}
\begin{equation}\label{FOCMEequ62}
I \leq\left\|A^{-1}\right\|^{2} \mathbb{E}\|\widehat{S}-S_{\nu}\| \lesssim \sqrt{1+\tilde{\Sigma}_{1}+C_{d}^{3}+\nu+(\nu+\nu^{2}+\nu^{3})\tilde{\Sigma}_{1}} M^{-\frac{\alpha \beta}{4}}+\sqrt{(1+\nu^{2}+\nu^{4})}M^{-\frac{1}{2}}.
\end{equation}
For term III, since $B_{i} B_{i}^{T}$ is positive-semidefinite, we have
$$\mathbb{E}\left\|B_{i} B_{i}^{T}\right\| \leq \mathbb{E} \operatorname{tr} B_{i} B_{i}^{T}=\operatorname{tr}\left(\mathbb{E} B_{i} B_{i}^{T}\right) \leq d\left\|\mathbb{E} B_{i} B_{i}^{T}\right\|.$$
Since $\tilde{\xi}_{p}$ is martingale difference, we have $\mathbb{E} U_{a_{m}-1} \tilde{\xi}_{p}^{T}=0$ for any $p \geq a_{m}$ and $\mathbb{E} \tilde{\xi}_{p_{1}} \tilde{\xi}_{p_{2}}=0$ for any $p_{1} \neq p_{2} .$ So,
\begin{equation}\label{FOCMEequ63}
\begin{aligned}
\left\|\mathbb{E} B_{i} B_{i}^{T}\right\|&=\left\|S_{a_{m}-1}^{i} \mathbb{E} U_{a_{m}-1} U_{a_{m}-1}{ }^{T} S_{a_{m}-1}^{i}+\sum_{p=a_{m}}^{i}\left(\eta_{p} S_{p}^{i}+\eta_{p} I-A^{-1}\right) \mathbb{E} \tilde{\xi}_{p} \tilde{\xi}_{p}^{T}\left(\eta_{p} S_{p}^{i}+\eta_{p} I-A^{-1}\right)^{T}\right\|\\
&\leq\left\|S_{a_{m}-1}^{i}\right\|^{2}\left\|\mathbb{E} U_{a_{m}-1} U_{a_{m}-1}^{T}\right\|+\sum_{p=a_{m}}^{i}\left\|\eta_{p} S_{p}^{i}+\eta_{p} I-A^{-1}\right\|^{2}\left\|\mathbb{E} \tilde{\xi}_{p} \tilde{\xi}_{p}^{T}\right\|.
\end{aligned}
\end{equation}
From Lemma \ref{FOCMElma2} and \ref{FOCMElma3}, we can see that $\left\|S_{a_{m}-1}^{i}\right\|^{2} \lesssim a_{m}^{2 \alpha}$, and
$$\left\|\mathbb{E} U_{a_{m}-1} U_{a_{m}-1}^{T}\right\| \lesssim \operatorname{tr} \mathbb{E} U_{a_{m}-1} U_{a_{m}-1}^{T} \lesssim \mathbb{E} \operatorname{tr} U_{a_{m}-1} U_{a_{m}-1}^{T} \lesssim E\left\|U_{a_{m}-1}\right\|^{2} \lesssim\left(a_{m}-1\right)^{-\alpha}.$$
Hence, we have
$$\left\|S_{a_{m}-1}^{i}\right\|^{2}\left\|\mathbb{E} U_{a_{m}-1} U_{a_{m}-1}^{T}\right\| \lesssim a_{m}^{\alpha}.$$
For the remaining part in (\ref{FOCMEequ63}), we first bound $\left\|\mathbb{E} \tilde{\xi}_{p} \tilde{\xi}_{p}^{T}\right\|$ by Equation (\ref{expansionequ2}), as
\begin{align*}
    \left\|\mathbb{E}\tilde{\xi}_{p}\tilde{\xi}_{p}^{T}\right\|&=\left\|\mathbb{E}\mathbb{E}_{p-1}\tilde{\xi}_{p}\tilde{\xi}_{p}^{T}\right\|\leq \left\|S_{\nu}\right\|+ \tilde{\Sigma}_{1}\mathbb{E}\|\tilde{\Delta}_{p-1}\|+\tilde{\Sigma}_{2}\mathbb{E}\|\tilde{\Delta}_{p-1}\|^{2}\\
    &\lesssim (1+\nu+\nu^{2})+\tilde{\Sigma}_{1}(p-1)^{-\frac{\alpha}{2}}\left(1+\|\tilde{\Delta}_{0}\|\right)+\tilde{\Sigma}_{2}(p-1)^{-\alpha}\left(1+\|\tilde{\Delta}_{0}\|^{2}\right)\\
    &\lesssim \tilde{\Sigma}_{1}(1+\nu)+(1+\nu+\nu^{2}).
\end{align*}
On the other side,
\begin{equation}\label{FOCMEequ64}
\sum_{p=a_{m}}^{i}\left\|\eta_{p} S_{p}^{i}+\eta_{p} I-A^{-1}\right\|^{2} \lesssim \sum_{p=a_{m}}^{i}\left(\left\|\eta_{p} S_{p}^{i}-A^{-1}\right\|^{2}+\left\|\eta_{p} I\right\|^{2}\right).
\end{equation}
Next, we need to bound $\left\|\eta_{p} S_{p}^{i}-A^{-1}\right\|^{2} .$ When $\eta_{j}=\eta j^{-\alpha}$ and $a_{m} \leq p \leq i<a_{m+1}$
based on Lemma D.3 (3) in \cite{chen2016statistical}, we have
$$\left\|\eta_{p} S_{p}^{i}-A^{-1}\right\| \leq p^{2 \alpha-2}+\exp \left(-2 \lambda_{A} \sum_{j=p}^{i} \eta_{j}\right).$$
We also have
$$\sum_{p=a_{m}}^{i} \exp \left(-2 \lambda_{A} \sum_{j=p}^{i} \eta_{j}\right) \leq \sum_{p=a_{m}}^{i} \exp \left(-2 \lambda_{A} \eta(i-p) i^{-\alpha}\right) \leq \sum_{k=0}^{\infty} \exp \left(-2 \lambda_{A} \eta i^{-\alpha} k\right).$$
Note that $\int_{x=0}^{\infty} e^{-a x} d x=a^{-1} .$ Hence, we can use integration to bound the summation above as
$$\sum_{k=0}^{\infty} \exp \left(-2 \lambda_{A} \eta i^{-\alpha} k\right) \leq \int_{k=0}^{\infty} \exp \left(-2 \lambda_{A} \eta i^{-\alpha} k\right) \lesssim i^{\alpha}.$$
Furthermore, $p^{2 \alpha-2} \geq p^{-2 \alpha}$ since $\alpha \geq 1 / 2$. Hence, we have
$$\sum_{p=a_{m}}^{i}\left\|\eta_{p} S_{p}^{i}+\eta_{p} I-A^{-1}\right\|^{2} \lesssim l_{i} a_{m}^{2 \alpha-2}+i^{\alpha}.$$
Now, recalling the definition of $B_{i},$ note that when $t_{i}=a_{m}$, we have
\begin{equation}\label{FOCMEequ65}
\begin{aligned}
\left\|\mathbb{E} B_{i} B_{i}^{T}\right\| &\lesssim \left(i^{\alpha}+l_{i} a_{m}^{2 \alpha-2}\right)\left(\tilde{\Sigma}_{1}(1+\nu)+(1+\nu+\nu^{2})\right)\\ 
&\lesssim \left(a_{m}^{\alpha}+l_{i} a_{m}^{2 \alpha-2}\right)\left(\tilde{\Sigma}_{1}(1+\nu)+(1+\nu+\nu^{2})\right).
\end{aligned}
\end{equation}
Now since $\sum_{i=1}^{n} l_{i} \asymp \sum_{m=1}^{M} n_{m}^{2},$ we can bound term III as follows,
\begin{equation}\label{FOCMEequ66}
\begin{aligned}
III &\lesssim\left(\sum_{m=1}^{M} n_{m}^{2}\right)^{-1} \sum_{m=1}^{M} \sum_{i=a_{m}}^{a_{m+1}-1} \mathbb{E}\left\|B_{i} B_{i}^{T}\right\|\\
&\leq\left(\sum_{m=1}^{M} n_{m}^{2}\right)^{-1} \sum_{m=1}^{M} \sum_{i=a_{m}}^{a_{m+1}-1}\left(a_{m}^{\alpha}+l_{i} a_{m}^{2 \alpha-2}\right)\left(\tilde{\Sigma}_{1}(1+\nu)+(1+\nu+\nu^{2})\right)\\
&\lesssim\left(\sum_{m=1}^{M} n_{m}^{2}\right)^{-1} \sum_{m=1}^{M}\left(n_{m}^{2} a_{m}^{2 \alpha-2}+n_{m} a_{m}^{\alpha}\right)\left(\tilde{\Sigma}_{1}(1+\nu)+(1+\nu+\nu^{2})\right).
\end{aligned}
\end{equation}
Recalling that $a_{m} \asymp m^{\beta}, n_{m} \asymp m^{\beta-1},$ we have
\begin{equation}\label{FOCMEequ67}
\begin{aligned}
I I I &\lesssim \left(a_{M}^{2 \alpha-2}+\frac{a_{M}^{\alpha}}{n_{M}}\right)\left(\tilde{\Sigma}_{1}(1+\nu)+(1+\nu+\nu^{2})\right)\\ 
&\lesssim \left(\tilde{\Sigma}_{1}(1+\nu)+(1+\nu+\nu^{2})\right)M^{(\alpha-1) \beta+1}.
\end{aligned}
\end{equation}
For the second part, using Cauchy-Schwarz inequality we have
\begin{equation}\label{FOCMEequ68}
I I \leq \sqrt{\frac{\sum_{i=1}^{n} \mathbb{E}\left\|A_{i} A_{i}^{T}\right\|}{\sum_{i=1}^{n} l_{i}} \frac{\sum_{i=1}^{n} \mathbb{E}\left\|B_{i} B_{i}^{T}\right\|}{\sum_{i=1}^{n} l_{i}}}.
\end{equation}
We already have $\left(\sum_{i=1}^{n} l_{i}\right)^{-1} \sum_{i=1}^{n} \mathbb{E}\left\|B_{i} B_{i}^{T}\right\| \lesssim M^{-1} .$ To finish the proof, the only term remained to bound is $\left(\sum_{i=1}^{n} l_{i}\right)^{-1} \sum_{i=1}^{n} \mathbb{E}\left\|A_{i} A_{i}^{T}\right\|$. Recall the definition of $A_{i}=\sum_{p=a_{m}}^{i} A^{-1} \tilde{\xi}_{p}$ when $a_{m} \leq i<a_{m+1} .$ Since $A_{i} A_{i}^{T}$ is positive semi-definite, we have
\begin{equation}\label{FOCMEequ69}
\mathbb{E}\left\|A_{i} A_{i}^{T}\right\| \leq \mathbb{E} \operatorname{tr} A_{i} A_{i}^{T}=\operatorname{tr}\left(\mathbb{E} A_{i} A_{i}^{T}\right)=\operatorname{tr}\left(A^{-1} \mathbb{E}\left(\sum_{p=a_{m}}^{i} \tilde{\xi}_{p}\right)\left(\sum_{p=a_{m}}^{i} \tilde{\xi}_{p}^{T}\right) A^{-T}\right).
\end{equation}
When $p \neq q,$ we have $\mathbb{E} \tilde{\xi}_{p} \tilde{\xi}_{q}=0 .$ Furthermore, $\mathbb{E}_{n-1} \tilde{\xi}_{n} \tilde{\xi}_{n}^{T}=S_{\nu}+\tilde{\Sigma}_{\nu}\left(\tilde{\Delta}_{n-1}\right)$. Then,
\begin{equation}\label{FOCMEequ70}
\begin{aligned}
\mathbb{E}\left\|A_{i} A_{i}^{T}\right\| &\leq \operatorname{tr}\left(A^{-1}\left(\sum_{p=a_{m}}^{i} S_{\nu}+\mathbb{E} \tilde{\Sigma}_{\nu}\left(\tilde{\Delta}_{n-1}\right)\right) A^{-T}\right)\\
&=\mathbb{E}\operatorname{tr}\left(A^{-1}\left(l_{i} S_{\nu}+\sum_{p=a_{m}}^{i} \tilde{\Sigma}_{\nu}\left(\tilde{\Delta}_{n-1}\right)\right) A^{-T}\right)\\
&\lesssim \mathbb{E}\left\|A^{-1}\left(l_{i} S_{\nu}+\sum_{p=a_{m}}^{i} \tilde{\Sigma}_{\nu}\left(\tilde{\Delta}_{n-1}\right)\right) A^{-T}\right\|\\
&\lesssim l_{i}\|S_{\nu}\|+\sum_{p=a_{m}}^{i} \mathbb{E}\left\|\tilde{\Sigma}_{\nu}\left(\tilde{\Delta}_{n-1}\right)\right\|.
\end{aligned}
\end{equation}
Note that $\left\|\tilde{\Sigma}_{\nu}\left(\tilde{\Delta}\right)\right\| \leq \tilde{\Sigma}_{1} \|\tilde{\Delta}\|+\tilde{\Sigma}_{2}\|\tilde{\Delta}\|^{2}$ for any $\tilde{\Delta}$ and
$$\mathbb{E}\left\|\tilde{\Delta}_{n}\right\| \leq n^{-\alpha / 2}\left(1+\left\|\tilde{\Delta}_{0}\right\|\right),\quad \mathbb{E}\left\|\tilde{\Delta}_{n}\right\|^{2} \leq n^{-\alpha}\left(1+\left\|\tilde{\Delta}_{0}\right\|^{2}\right).$$ 
Then we can further bound $\mathbb{E}\left\|A_{i} A_{i}^{T}\right\|$ as
\begin{equation}\label{FOCMEequ71}
\begin{aligned}
\mathbb{E}\left\|A_{i} A_{i}^{T}\right\| &\lesssim l_{i}\|S_{\nu}\|+\sum_{p=a_{m}}^{i}\left(\tilde{\Sigma}_{1}\mathbb{E}\left\|\tilde{\Delta}_{p-1}\right\|+\tilde{\Sigma}_{2}\mathbb{E}\left\|\tilde{\Delta}_{p-1}\right\|^{2}\right)\\
&\lesssim l_{i}(1+\nu+\nu^{2})+\sum_{p=a_{m}}^{i}(p-1)^{-\alpha / 2}\tilde{\Sigma}_{1}\left(1+\|\tilde{\Delta}_{0}\|\right)\\
&\lesssim l_{i}(1+\nu+\nu^{2})+l_{i}\left(a_{m}-1\right)^{-\alpha / 2}\tilde{\Sigma}_{1}\left(1+\nu\right)\\
&\lesssim (1+\nu+\nu^{2})l_{i}.
\end{aligned}
\end{equation}
Hence, we we can bound the remaining term as
\begin{equation}\label{FOCMEequ72}
    \left(\sum_{i=1}^{n} l_{i}\right)^{-1} \sum_{i=1}^{n} \mathbb{E}\left\|A_{i} A_{i}^{T}\right\| \lesssim O(1+\nu+\nu^{2})
\end{equation}
Combining (\ref{FOCMEequ72}) and bound of III, we have $$I I \lesssim \sqrt{1+\nu+\nu^{2}} M^{\frac{(\alpha-1) \beta+1}{2}}.$$
Now, all three terms I, II, III are bounded by 
$$\color{black}\sqrt{1+\tilde{\Sigma}_{1}+C_{d}^{3}+\nu+(\nu+\nu^{2}+\nu^{3})\tilde{\Sigma}_{1}} M^{-\frac{\alpha \beta}{4}}+\sqrt{(1+\nu^{2}+\nu^{4})}M^{-\frac{1}{2}}.+\sqrt{1+\nu+\nu^{2}}M^{\frac{(\alpha-1) \beta+1}{2}}.$$
\end{proof}
\begin{lemma}\label{FOCMElma8}
Under same conditions in Lemma \ref{FOCMElma5}, we have
\begin{equation}\label{FOCMEequ73}\color{black}
    \mathbb{E}\left\|\left(\sum_{i=1}^{n} l_{i}\right)^{-1} \sum_{i=1}^{n} l_{i}^{2} \bar{U}_{n} \bar{U}_{n}^{T}\right\| \lesssim (1+\tilde{\Sigma}_{1}+(1+\tilde{\Sigma}_{1})\nu+\nu^{2}) M^{-1}
\end{equation}
where $a_{M} \leq n<a_{M+1}$.
\end{lemma}
\begin{proof}
Since $\bar{U}_{n} \bar{U}_{n}^{T}$ is positive semi-definite, we have
\begin{equation}\label{FOCMEequ74}
\mathbb{E}\left\|\bar{U}_{n} \bar{U}_{n}^{T}\right\| \leq \mathbb{E} \operatorname{tr}\left(\bar{U}_{n} \bar{U}_{n}^{T}\right)=n^{-2} \operatorname{tr}\left(\mathbb{E}\left(\sum_{i=1}^{n} U_{i}\right)\left(\sum_{i=1}^{n} U_{i}\right)^{T}\right).
\end{equation}
Recall that $U_{i}=Y_{0}^{i} U_{0}+\sum_{p=1}^{i} Y_{p}^{i} \eta_{p} \tilde{\xi}_{p},$ then we have
$$\sum_{i=1}^{n} U_{i}=\sum_{i=1}^{n}\left(Y_{0}^{i} U_{0}+\sum_{p=1}^{i} Y_{p}^{i} \eta_{p} \tilde{\xi}_{p}\right)=S_{0}^{n} U_{0}+\sum_{p=1}^{n}\left(I+S_{p}^{n}\right) \eta_{p} \tilde{\xi}_{p}.$$
Note that $\tilde{\xi}_{p}$ is martingale difference. We have the following inequality after plugging in the expansion above,
\begin{equation}\label{FOCMEequ75}
\begin{aligned}
\mathbb{E}\left\|\bar{U}_{n} \bar{U}_{n}^{T}\right\| &\leq n^{-2} \operatorname{tr}\left(\mathbb{E}\left(S_{0}^{n} U_{0}+\sum_{p=1}^{n}\left(I+S_{p}^{n}\right) \eta_{p} \tilde{\xi}_{p}\right)\left(S_{0}^{n} U_{0}+\sum_{p=1}^{n}\left(I+S_{p}^{n}\right) \eta_{p} \tilde{\xi}_{p}\right)^{T}\right)\\
&=n^{-2} \operatorname{tr}\left(S_{0}^{n} \mathbb{E} U_{0} U_{0}^{T} S_{0}^{n T}+\sum_{p=1}^{n}\left(I+S_{p}^{n}\right) \eta_{p}^{2} \mathbb{E} \tilde{\xi}_{p} \tilde{\xi}_{p}^{T}\left(I+S_{p}^{n}\right)^{T}\right)\\
&=n^{-2}\left(\left\|S_{0}^{n}\right\|^{2} \mathbb{E}\left\|U_{0}\right\|^{2}+\sum_{p=1}^{n}\left\|\left(I+S_{p}^{n}\right)\right\|^{2} \eta_{p}^{2} \mathbb{E}\left\|\tilde{\xi}_{p}\right\|^{2}\right).
\end{aligned}
\end{equation}
In Lemma \ref{FOCMElma2} we show that $\left\|S_{i}^{j}\right\| \lesssim(i+1)^{\alpha} \asymp i^{\alpha} .$ So here we have $\left\|S_{0}^{n}\right\|^{2}=O(1)$, and
$$\sum_{p=1}^{n}\left\|\left(I+S_{p}^{n}\right)\right\|^{2} \eta_{p}^{2} \lesssim \sum_{j=1}^{n} p^{2 \alpha} p^{-2 \alpha}=O(n).$$
Since $\mathbb{E}\left\|U_{0}\right\|^{2}$ is bounded by a constant, we have 
\begin{align*}
    \mathbb{E}\left\|\tilde{\xi}_{p}\right\|^{2} &=\mathrm{tr}\mathbb{E}\tilde{\xi}_{p}\tilde{\xi}_{p}^{T}\leq \mathrm{tr}S_{\nu}+\left|\mathrm{tr}\left(\tilde{\Sigma}_{\nu}(\tilde{\Delta}_{p-1})\right)\right|\\
    &\lesssim \tilde{\Sigma}_{1}(1+\nu)+(1+\nu+\nu^{2}).
\end{align*}
Hence, we have
\begin{equation}\label{FOCMEequ76}\color{black}
\mathbb{E}\left\|\bar{U}_{n} \bar{U}_{n}^{T}\right\| \lesssim O\left(n^{-1}(1+\tilde{\Sigma}_{1}+(1+\tilde{\Sigma}_{1})\nu+\nu^{2})\right).
\end{equation}
Now, note that
\begin{equation}\label{FOCMEequ77}
\frac{\sum_{i=1}^{n} l_{i}^{2}}{\sum_{i=1}^{n} l_{i}} \leq \frac{\sum_{i=1}^{n} l_{i} \max _{k \leq M}\left(a_{k+1}-a_{k}\right)}{\sum_{i=1}^{n} l_{i}} \leq n_{M}.
\end{equation}
Due to the definition $n_{M}=M^{\beta-1}, n \asymp M^{1 / \beta}$. Hence, we obtain
{\color{black}
\begin{align*}
   \mathbb{E}\left\|\left(\sum_{i=1}^{n} l_{i}\right)^{-1} \sum_{i=1}^{n} l_{i}^{2} \bar{U}_{n} \bar{U}_{n}^{T}\right\| &\leq \frac{\sum_{i=1}^{n} l_{i}^{2}}{\sum_{i=1}^{n} l_{i}} \mathbb{E}\left\|\bar{U}_{n} \bar{U}_{n}^{T}\right\|\\ 
   &\lesssim n_{M} n^{-1}(1+\tilde{\Sigma}_{1}+(1+\tilde{\Sigma}_{1})\nu+\nu^{2})\\ 
   &\asymp (1+\tilde{\Sigma}_{1}+(1+\tilde{\Sigma}_{1})\nu+\nu^{2}) M^{-1}. 
\end{align*}
}
\end{proof}

\begin{lemma}\label{FOCMElma9}
Under same conditions in Lemma \ref{FOCMElma5}. When $a_{M} \leq n<a_{M+1},$ we have
\begin{equation}\label{FOCMEequ78}\color{black}
    \mathbb{E}\left\|\left(\sum_{i=1}^{n} l_{i}\right)^{-1} \sum_{i=1}^{n}\left(\sum_{k=t_{i}}^{i} U_{k}\right)\left(l_{i} \bar{U}_{n}\right)^{T}\right\| \lesssim (\tilde{\Sigma}_{1}(1+\nu)+(1+\nu+\nu^{2}))M^{-\frac{1}{2}}
\end{equation}
\end{lemma}
\begin{proof}
Applying Cauchy-Schwarz inequality twice we have
\begin{equation}\label{FOCMEequ79}
\mathbb{E}\left\|\left(\sum_{i=1}^{n} l_{i}\right)^{-1} \sum_{i=1}^{n}\left(\sum_{k=t_{i}}^{i} U_{k}\right)\left(l_{i} \bar{U}_{n}\right)^{T}\right\|\leq \sqrt{\frac{\mathbb{E}\left\|\sum_{i=1}^{n}\left(\sum_{k=t_{i}}^{i} U_{k}\right)\left(\sum_{k=t_{i}}^{i} U_{k}\right)^{T}\right\|}{\sum_{i=1}^{n} l_{i}} \frac{\mathbb{E}\left\|\sum_{i=1}^{n} l_{i}^{2} \bar{U}_{n} \bar{U}_{n}^{T}\right\|}{\sum_{i=1}^{n} l_{i}}}.
\end{equation}
In Lemma \ref{FOCMElma8}, we already showed that $$\mathbb{E}\left\|\left(\sum_{i=1}^{n} l_{i}\right)^{-1} \sum_{i=1}^{n} l_{i}^{2} \bar{U}_{n} \bar{U}_{n}^{T}\right\| \lesssim (1+\tilde{\Sigma}_{1}+(1+\tilde{\Sigma}_{1})\nu+\nu^{2})M^{-1}$$ 
Moreover, the $L_{2}$ norm of $\left(\sum_{k=a_{m}}^{i} U_{k}\right)\left(\sum_{k=a_{m}}^{i} U_{k}\right)^{T}$ is less than or equal to its trace since it's positive definite. Then we have the left side above bounded by
$$O\left(\sqrt{1+\tilde{\Sigma}_{1}+(1+\tilde{\Sigma}_{1})\nu+\nu^{2}} M^{-\frac{1}{2}}\right) \sqrt{\left(\sum_{i=1}^{n} l_{i}\right)^{-1} \sum_{i=1}^{n} \mathbb{E} \operatorname{tr}\left(\sum_{k=t_{i}}^{i} U_{k}\right)\left(\sum_{k=t_{i}}^{i} U_{k}\right)^{T}}.$$
Now, let $$I=\left(\sum_{i=1}^{n} l_{i}\right)^{-1} \sum_{i=1}^{n} \mathbb{E}\operatorname{tr}\left(\sum_{k=t_{i}}^{i} U_{k}\right)\left(\sum_{k=t_{i}}^{i} U_{k}\right)^{T}.$$ 
To show lemma \ref{FOCMElma9}, it is suffices
to show $I \lesssim O(1)$. Note that 
$$\lim _{M \rightarrow \infty} \frac{\sum_{i=1}^{n} l_{i}}{\sum_{i=1}^{a_{M}-1} l_{i}}=1,$$ 
and 
$$\operatorname{tr}\left(\sum_{k=t_{i}}^{i} U_{k}\right)\left(\sum_{k=t_{i}}^{i} U_{k}\right)^{T} \geq 0.$$
Plug $U_{k}=Y_{s}^{k} U_{s}+\sum_{p=s+1}^{k} Y_{p}^{k} \eta_{p} \tilde{\xi}_{p}$ into I, we have
\begin{equation}\label{FOCMEequ80}
\begin{aligned}
    I &\lesssim\left(\sum_{i=1}^{a_{M+1}-1} l_{i}\right)^{-1} \sum_{m=1}^{M} \sum_{i=a_{m}}^{a_{m+1}-1} \operatorname{Etr}\left(\sum_{k=a_{m}}^{i}\left(Y_{0}^{k} U_{0}+\sum_{p=1}^{k} Y_{p}^{k} \eta_{p} \tilde{\xi}_{p}\right)\right)\left(\sum_{k=a_{m}}^{i}\left(Y_{0}^{k} U_{0}+\sum_{p=1}^{k} Y_{p}^{k} \eta_{p} \tilde{\xi}_{p}\right)\right)^{T}\\
    &=\left(\sum_{i=1}^{a_{M+1}-1} l_{i}\right)^{-1} \sum_{m=1}^{M} \sum_{i=a_{m}}^{a_{m+1}-1} \operatorname{tr}\left(\sum_{k=a_{m}}^{i} Y_{0}^{k}\right) \mathbb{E} U_{0} U_{0}^{T}\left(\sum_{k=a_{m}}^{i} Y_{0}^{k}\right)^{T}\\
    &\quad +\left(\sum_{i=1}^{a_{M+1}-1} l_{i}\right)^{-1} \sum_{m=1}^{M} \sum_{i=a_{m}}^{a_{m+1}-1} \sum_{p=1}^{i} \operatorname{tr}\left(\sum_{k=\max \left(a_{m}, p\right)}^{i} Y_{p}^{k}\right) \mathbb{E} \tilde{\xi}_{p} \tilde{\xi}_{p}^{T}\left(\sum_{k=\max \left(a_{m}, p\right)}^{i} Y_{p}^{k}\right)^{T} \eta_{p}^{2}\\
    &=I I+I I I.
\end{aligned}
\end{equation}
The first term above 
$$II=\left(\sum_{i=1}^{a_{M+1}-1} l_{i}\right)^{-1} \sum_{m=1}^{M} \sum_{i=a_{m}}^{a_{m+1}-1} \operatorname{tr}\left(\sum_{k=a_{m}}^{i} Y_{0}^{k}\right) \mathbb{E} U_{0} U_{0}^{T}\left(\sum_{k=a_{m}}^{i} Y_{0}^{k}\right)^{T}$$ can be
bounded using $\operatorname{tr}(C) \leq d\|C\|$ as
\begin{equation}\label{FOCMEequ81}
I I \lesssim\left(\sum_{i=1}^{a_{M+1}-1} l_{i}\right)^{-1} \sum_{m=1}^{M} \sum_{i=a_{m}}^{a_{m+1}-1}\left\|\left(\sum_{k=a_{m}}^{i} Y_{0}^{k}\right)\right\|^{2}\left\|\mathbb{E} U_{0} U_{0}^{T}\right\|.
\end{equation}
From Lemma \ref{FOCMElma2}, we obtain
$$\left\|\left(\sum_{k=a_{m}}^{i} Y_{0}^{k}\right)\right\|^{2}=\left\|S_{0}^{i}-S_{0}^{a_{m}}\right\|^{2} \lesssim\left\|S_{0}^{i}\right\|^{2}+\left\|S_{0}^{a_{m}}\right\|^{2} \lesssim O(1).$$
Also note that $\left\|\mathbb{E} U_{0} U_{0}^{T}\right\| \lesssim O(1) .$ Then
\begin{equation}\label{FOCMEequ82}
I I \lesssim\left(\sum_{i=1}^{a_{M+1}-1} l_{i}\right)^{-1} \sum_{m=1}^{M} \sum_{i=a_{m}}^{a_{m+1}-1} O(1)=O(1).
\end{equation}
The second term can be bounded as:
\begin{equation}\label{FOCMEequ83}
\begin{aligned}
III&=\left(\sum_{i=1}^{a_{M+1}-1} l_{i}\right)^{-1} \sum_{m=1}^{M} \sum_{i=a_{m}}^{a_{m+1}-1} \sum_{p=1}^{i} \operatorname{tr}\left(\sum_{k=\max \left(a_{m}, p\right)}^{i} Y_{p}^{k}\right) \mathbb{E} \tilde{\xi}_{p} \tilde{\xi}_{p}^{T}\left(\sum_{k=\max \left(a_{m}, p\right)}^{i} Y_{p}^{k}\right)^{T} \eta_{p}^{2}\\
&\lesssim\left(\sum_{i=1}^{a_{M+1}-1} l_{i}\right)^{-1} \sum_{m=1}^{M} \sum_{i=a_{m}}^{a_{m+1}-1} \sum_{p=1}^{i}\left\|\sum_{k=\max \left(a_{m}, p\right)}^{i} Y_{p}^{k}\right\|^{2} \eta_{p}^{2}\left(\tilde{\Sigma}_{1}(1+\nu)+(1+\nu+\nu^{2})\right)\\
&\leq\left(\sum_{i=1}^{a_{M+1}-1} l_{i}\right)^{-1} \sum_{m=1}^{M} \sum_{i=a_{m}}^{a_{m+1}-1} \sum_{p=1}^{i}\left(\sum_{k=\max \left(a_{m}, p\right)}^{i}\left\|Y_{p}^{k}\right\|\right)^{2} \eta_{p}^{2}\left(\tilde{\Sigma}_{1}(1+\nu)+(1+\nu+\nu^{2})\right).
\end{aligned}
\end{equation}
Let $I V=\sum_{p=1}^{i}\left(\sum_{k=\max \left(a_{m}, p\right)}^{i}\left\|Y_{p}^{k}\right\|\right)^{2} \eta_{p}^{2}$. From Lemma \ref{FOCMElma1}, we have
$$\left\|Y_{i}^{j}\right\| \leq \exp \left[-\frac{\lambda_{A} \eta}{1-\alpha}\left(j^{1-\alpha}-(i+1)^{1-\alpha}\right)\right].$$
Then for $a_{m} \leq i<a_{m+1},$ we have
\begin{equation}\label{FOCMEequ84}
I V \leq \sum_{p=1}^{i}\left(\sum_{k=\max \left(a_{m}, p\right)}^{i} \exp \left(-\eta \lambda_{A} \frac{k^{1-\alpha}}{1-\alpha}\right)\right)^{2} \eta_{p}^{2} e^{\frac{2 \eta \lambda_{A}}{1-\alpha} p^{1-\alpha}}.
\end{equation}
Using the integration, we can further bound the above as
\begin{equation}\label{FOCMEequ85}
\begin{aligned}
I V &\lesssim \sum_{p=1}^{i}\left(\int_{\max \left(a_{m}, p\right)}^{i} \exp \left(-\eta \lambda_{A} \frac{k^{1-\alpha}}{1-\alpha}\right) d k\right)^{2} p^{-2 \alpha} e^{\frac{2 \eta \lambda}{1-\alpha} p^{1-\alpha}}\\
&\lesssim \sum_{p=1}^{i}\left(\int_{\max \left(a_{m}, p\right)^{1-\alpha}}^{i^{1-\alpha}} e^{-\frac{\eta \lambda_{A}}{1-\alpha} t} t^{\frac{\alpha}{1-\alpha}} d t\right)^{2} p^{-2 \alpha} e^{\frac{2 \eta \lambda}{1-\alpha} p^{1-\alpha}}\\
&\lesssim \sum_{p=1}^{i} e^{-\frac{2 \eta \lambda_{A}}{1-\alpha} \max \left(a_{m}, p\right)^{1-\alpha}} \max \left(a_{m}, p\right)^{2 \alpha} p^{-2 \alpha} e^{\frac{2 \eta \lambda}{1-\alpha} p^{1-\alpha}}\\
&\lesssim \sum_{p=1}^{a_{m}-1} e^{-\frac{2 \eta \lambda_{A}}{1-\alpha}\left(a_{m}^{1-\alpha}-p^{1-\alpha}\right)}\left(\frac{a_{m}}{p}\right)^{2 \alpha}+l_{i}.
\end{aligned}
\end{equation}
Then III is bounded by
\begin{equation}\label{FOCMEequ86}\color{black}
I I I \lesssim \left[\left(\sum_{i=1}^{a_{M+1}-1} l_{i}\right)^{-1} \sum_{m=1}^{M} \sum_{i=a_{m}}^{a_{m+1}-1} \sum_{p=1}^{a_{m}-1} e^{-\frac{2 \eta \lambda}{1-\alpha}\left(a_{m}^{1-\alpha}-p^{1-\alpha}\right)}\left(\frac{a_{m}}{p}\right)^{2 \alpha}+1\right]\left(\tilde{\Sigma}_{1}(1+\nu)+(1+\nu+\nu^{2})\right)
\end{equation}
Now it suffices to show that
\begin{equation}\label{FOCMEequ87}
\sum_{m=1}^{M} \sum_{i=a_{m}}^{a_{m+1}-1} \sum_{p=1}^{a_{m}-1} e^{-\frac{2 \eta \lambda_{A}}{1-\alpha}\left(a_{m}^{1-\alpha}-p^{1-\alpha}\right)}\left(\frac{a_{m}}{p}\right)^{2 \alpha} \lesssim \sum_{i=1}^{a_{M+1}-1} l_{i}.
\end{equation}
Note that using partial integration we have the following,
$$\int_{1}^{a_{m}-1} e^{\frac{2 \eta \lambda_{A}}{1-\alpha} p^{1-\alpha}} p^{-2 \alpha} d p=\int_{\frac{2 \eta \lambda_{A}}{1-\alpha}}^{\frac{2 \eta \lambda_{A}}{1-\alpha}\left(a_{m}-1\right)^{1-\alpha}} e^{u} u^{-\frac{\alpha}{1-\alpha}} d u \lesssim e^{\frac{2 \eta \lambda}{1-\alpha} a_{m-1}^{1-\alpha}}\left(a_{m}-1\right)^{-\alpha}.$$
Hence, we obtain
\begin{equation}\label{FOCMEequ88}
\begin{aligned}
&\sum_{m=1}^{M} \sum_{i=a_{m}}^{a_{m+1}-1} \sum_{p=1}^{a_{m}-1} e^{-\frac{2 \eta \lambda_{A}}{1-\alpha}\left(a_{m}^{1-\alpha}-p^{1-\alpha}\right)}\left(\frac{a_{m}}{p}\right)^{2 \alpha}\\
\lesssim &\sum_{m=1}^{M} \sum_{i=a_{m}}^{a_{m+1}-1} e^{-\frac{2 \eta \lambda_{A}}{1-\alpha} a_{m}^{1-\alpha}} a_{m}^{2 \alpha} \int_{1}^{a_{m}-1} e^{\frac{2 \eta \lambda}{1-\alpha} p^{1-\alpha}} p^{-2 \alpha} d p\\
\lesssim &\sum_{m=1}^{M} n_{m} a_{m}^{\alpha}.
\end{aligned}
\end{equation}
Note that $a_{m}^{\alpha} \lesssim n_{m},$ so we have the following
\begin{equation}\label{FOCMEequ89}
\sum_{m=1}^{M} \sum_{i=a_{m}}^{a_{m+1}-1} \sum_{p=1}^{a_{m}-1} e^{-\frac{2 \eta \lambda_{A}}{1-\alpha}\left(a_{m}^{1-\alpha}-p^{1-\alpha}\right)}\left(\frac{a_{m}}{p}\right)^{2 \alpha} \lesssim \sum_{m=1}^{M} n_{m} a_{m}^{\alpha} \lesssim \sum_{m=1}^{M} n_{m}^{2} \asymp \sum_{i=1}^{a_{M+1}-1} l_{i}.
\end{equation}
Therefore, we have 
$$III\lesssim O(\tilde{\Sigma}_{1}(1+\nu)+(1+\nu+\nu^{2})),$$
and 
$$\color{black}\mathbb{E}\left\|\left(\sum_{i=1}^{n} l_{i}\right)^{-1} \sum_{i=1}^{n}\left(\sum_{k=t_{i}}^{i} U_{k}\right)\left(l_{i} \bar{U}_{n}\right)^{T}\right\| \lesssim (\tilde{\Sigma}_{1}(1+\nu)+(1+\nu+\nu^{2}))M^{-\frac{1}{2}}.$$
\end{proof}
We are now ready to prove Theorem \ref{zerothsgdthm1}.

\subsection{Proof for Theorem \ref{zerothsgdthm1}}
\begin{proof}
To show Theorem \ref{zerothsgdthm1} it suffices to show that the order of $\mathbb{E}\|\tilde{\Sigma}_{n}-\widehat{\Sigma}_{n}\|$ can be bounded by the same order as $\mathbb{E}\|\tilde{\Sigma}_{n}-\tilde{V}\|$. Recall that $U_{n}=\left(I-\eta_{n} A\right) U_{n-1}+\eta_{n} \tilde{\xi}_{n}$. Note that $\tilde{\Delta}_{n}=x_{n}-x_{\nu}^{*}$ and $\bar{\Delta}_{n}=\frac{1}{n+1}\sum_{i=0}^{n}\tilde{\Delta}_{i}$, then $\widehat{\Sigma}_{n}$ can be rewritten as
$$\widehat{\Sigma}_{n}=\left(\sum_{i=1}^{n} l_{i}\right)^{-1}\left(\sum_{k=t_{i}}^{i} \tilde{\Delta}_{k}-l_{i} \bar{\Delta}_{n}\right)\left(\sum_{k=t_{i}}^{i} \tilde{\Delta}_{k}-l_{i} \bar{\Delta}_{n}\right)^{T}.$$
Let $s_{n}=\tilde{\Delta}_{n}-U_{n}$ be the difference between the error sequence $\tilde{\Delta}_{n}$ and the linear sequence $U_{n}$. Then, we have
\begin{align*}
s_{n}&=\tilde{\Delta}_{n-1}-\eta_{n} \nabla f_{\nu}\left(x_{n-1}\right)-\left(I-\eta_{n} A\right) U_{n-1}\\
&=\left(I-\eta_{n} A\right)\left(\tilde{\Delta}_{n-1}-U_{n-1}\right)-\eta_{n}\left(\nabla f_{\nu}\left(x_{n-1}\right)-A \tilde{\Delta}_{n-1}\right)\\
&=\left(I-\eta_{n} A\right) s_{n-1}-\eta_{n}\left(\nabla f_{\nu}\left(x_{n-1}\right)-A \tilde{\Delta}_{n-1}\right)
\end{align*}
Plugging in the difference $s_{n}=\tilde{\Delta}_{n}-U_{n},$ we can expand $\mathbb{E}\|\tilde{\Sigma}_{n}-\widehat{\Sigma}_{n}\|$ as
\begin{equation}\label{FOCMEequ92}
\begin{aligned}
\mathbb{E}\|\tilde{\Sigma}_{n}-\widehat{\Sigma}_{n}\| &\leq 2 \mathbb{E}\left\|\left(\sum_{i=1}^{n} l_{i}\right)^{-1} \sum_{i=1}^{n}\left(\sum_{k=t_{i}}^{i} U_{k}-l_{i} \bar{U}_{n}\right)\left(\sum_{k=t_{i}}^{i} s_{k}-l_{i} \bar{s}_{n}\right)^{T}\right\|\\
&+\mathbb{E}\left\|\left(\sum_{i=1}^{n} l_{i}\right)^{-1} \sum_{i=1}^{n}\left(\sum_{k=t_{i}}^{i} s_{k}-l_{i} \bar{s}_{n}\right)\left(\sum_{k=t_{i}}^{i} s_{k}-l_{i} \bar{s}_{n}\right)^{T}\right\|
\end{aligned}
\end{equation}
Now, we show that the following is true:
\begin{align}\label{FOCMEclm1}
\mathbb{E}\left\|\left(\sum_{i=1}^{n} l_{i}\right)^{-1} \sum_{i=1}^{n}\left(\sum_{k=t_{i}}^{i} s_{k}-l_{i} \bar{s}_{n}\right)\left(\sum_{k=t_{i}}^{i} s_{k}-l_{i} \bar{s}_{n}\right)^{T}\right\| \lesssim M^{-1}.
\end{align}

To see that, first note that by Young's inequality the first part in LHS of (\ref{FOCMEequ92}) can be bounded as following since $\sqrt{\mathbb{E}\|\tilde{\Sigma}_{n}\|}$ is bounded by some constant:
\begin{equation}\label{FOCMEequ93}
\begin{aligned}
&\mathbb{E}\left\|\left(\sum_{i=1}^{n} l_{i}\right)^{-1} \sum_{i=1}^{n}\left(\sum_{k=t_{i}}^{i} U_{k}-l_{i} \bar{U}_{n}\right)\left(\sum_{k=t_{i}}^{i} s_{k}-l_{i} \bar{s}_{n}\right)^{T}\right\|\\
\leq &\sqrt{\mathbb{E}\|\tilde{\Sigma}_{n}\|} \sqrt{\mathbb{E}\left\|\left(\sum_{i=1}^{n} l_{i}\right)^{-1} \sum_{i=1}^{n}\left(\sum_{k=t_{i}}^{i} s_{k}-l_{i} \bar{s}_{n}\right)\left(\sum_{k=t_{i}}^{i} s_{k}-l_{i} \bar{s}_{n}\right)^{T}\right\|}\lesssim M^{-1 / 2}.
\end{aligned}
\end{equation}
Hence, the difference $\mathbb{E}\|\tilde{\Sigma}_{n}-\widehat{\Sigma}_{n}\| \lesssim M^{-1 / 2}$ and we have Theorem \ref{zerothsgdthm1}. Now, all we need to prove~\eqref{FOCMEclm1} is the claim
$$\mathbb{E}\left\|\left(\sum_{i=1}^{n} l_{i}\right)^{-1} \sum_{i=1}^{n}\left(\sum_{k=t_{i}}^{i} s_{k}-l_{i} \bar{s}_{n}\right)\left(\sum_{k=t_{i}}^{i} s_{k}-l_{i} \bar{s}_{n}\right)^{T}\right\| \lesssim M^{-1}$$
By triangle inequality and the fact that$\|C\| \leq \operatorname{tr}(C)$ for any positive semidefinite matrix C, we have
\begin{equation}\label{FOCMEequ94}
\begin{aligned}
    &\mathbb{E}\left\|\left(\sum_{i=1}^{n} l_{i}\right)^{-1} \sum_{i=1}^{n}\left(\sum_{k=t_{i}}^{i} s_{k}-l_{i} \bar{s}_{n}\right)\left(\sum_{k=t_{i}}^{i} s_{k}-l_{i} \bar{s}_{n}\right)^{T}\right\|\\
    \leq&\left(\sum_{i=1}^{n} l_{i}\right)^{-1} \sum_{i=1}^{n} \mathbb{E} \operatorname{tr}\left(\left(\sum_{k=t_{i}}^{i} s_{k}-l_{i} \bar{s}_{n}\right)\left(\sum_{k=t_{i}}^{i} s_{k}-l_{i} \bar{s}_{n}\right)^{T}\right)\\
    =&\left(\sum_{i=1}^{n} l_{i}\right)^{-1} \sum_{i=1}^{n} \mathbb{E}\left\|\sum_{k=t_{i}}^{i} s_{k}-l_{i} \bar{s}_{n}\right\|^{2}\\
    \lesssim&\left(\sum_{i=1}^{n} l_{i}\right)^{-1} \sum_{i=1}^{n} \mathbb{E}\left\|\sum_{k=t_{i}}^{i} s_{k}\right\|^{2}+\left(\sum_{i=1}^{n} l_{i}\right)^{-1} \sum_{i=1}^{n} l_{i}^{2} \mathbb{E}\left\|\bar{s}_{n}\right\|^{2}.
\end{aligned}
\end{equation}
Note that $s_{n}$ takes the form
$$s_{n}=\left(I-\eta_{n} A\right) s_{n-1}-\eta_{n}\left(\nabla f_{\nu}\left(x_{n-1}\right)-A \tilde{\Delta}_{n-1}\right).$$
First, we will show that $\left(\sum_{i=1}^{n} l_{i}\right)^{-1} \sum_{i=1}^{n} l_{i}^{2} \mathbb{E}\left\|\bar{s}_{n}\right\|^{2}$ is bounded by $M^{-1}.$ To see that, based on the definition of $Y_{p}^{k}$, we first have
$$s_{k}=\sum_{p=1}^{k} Y_{p}^{k} \eta_{p}\left[A \tilde{\Delta}_{p-1}-\nabla f_{\nu}\left(x_{p-1}\right)\right],$$
and
\begin{align*}
    \bar{s}_{n}&=n^{-1} \sum_{k=1}^{n} \sum_{p=1}^{k} Y_{p}^{k} \eta_{p}\left[A \tilde{\Delta}_{p-1}-\nabla f_{\nu}\left(x_{p-1}\right)\right]\\
    &=n^{-1} \sum_{p=1}^{n}\left(I+S_{p}^{n}\right) \eta_{p}\left[A \tilde{\Delta}_{p-1}-\nabla f_{\nu}\left(x_{p-1}\right)\right].
\end{align*}
By Cauchy-Schwartz inequality, we then have
\begin{equation}\label{FOCMEequ95}
\begin{aligned}
\mathbb{E}\left\|\bar{s}_{n}\right\|^{2}&=n^{-2} \mathbb{E}\left\|\sum_{p=1}^{n}\left(I+S_{p}^{n}\right) \eta_{p}\left[A \tilde{\Delta}_{p-1}-\nabla f_{\nu}\left(x_{p-1}\right)\right]\right\|^{2}\\
&\leq n^{-2} \mathbb{E}\left(\sum_{p=1}^{n}\left\|I+S_{p}^{n}\right\| \eta_{p}\left\|A \tilde{\Delta}_{p-1}-\nabla f_{\nu}\left(x_{p-1}\right)\right\|\right)^{2}\\
&\leq n^{-2}\left(\sum_{p=1}^{n}\left\|I+S_{p}^{n}\right\|^{2} \eta_{p}^{2}\right)\left(\sum_{p=1}^{n} \mathbb{E}\left\|A \tilde{\Delta}_{p-1}-\nabla f_{\nu}\left(x_{p-1}\right)\right\|^{2}\right).
\end{aligned}
\end{equation}
From Lemma \ref{FOCMElma2} $\left\|S_{p}^{n}\right\| \lesssim p^{\alpha},$ and so 
$$\sum_{p=1}^{n}\left\|I+S_{p}^{n}\right\|^{2} \eta_{p}^{2} \lesssim \sum_{p=1}^{n} p^{2 \alpha} p^{-2 \alpha}=n.$$ 
According to Taylor expansion around $x_{\nu}^{*},\left\|A \tilde{\Delta}_{p}-\nabla f_{\nu}\left(x_{p}\right)\right\|=O\left(\left\|\tilde{\Delta}_{p}\right\|^{2}\right)$. Hence,
\begin{equation}\label{FOCMEequ96}
    \sum_{p=1}^{n} \mathbb{E}\left\|A \tilde{\Delta}_{p-1}-\nabla f_{\nu}\left(x_{p-1}\right)\right\|^{2} \asymp \sum_{p=1}^{n} \mathbb{E}\left\|\tilde{\Delta}_{p-1}\right\|^{4} \lesssim \sum_{p=1}^{n}(p-1)^{-2 \alpha}.
\end{equation}
Since $\alpha>\frac{1}{2}$, we have $\sum_{p=1}^{n}(p-1)^{-2 \alpha}=O(1)$. Then $\mathbb{E}\left\|\bar{s}_{n}\right\|^{2} \lesssim n^{-1}$. Recalling that $n_{k}=C k^{\beta-1}$ and $M \asymp n^{1 / \beta}$ we have,
\begin{equation}\label{FOCMEequ97}
\left(\sum_{i=1}^{n} l_{i}\right)^{-1} \sum_{i=1}^{n} l_{i}^{2} \mathbb{E}\left\|\bar{s}_{n}\right\|^{2} \lesssim n^{-1} n_{M} \asymp M^{-1}.
\end{equation}
Next we will prove that $\left(\sum_{i=1}^{n} l_{i}\right)^{-1} \sum_{i=1}^{n} \mathbb{E}\left\|\sum_{k=t_{i}}^{i} s_{k}\right\|^{2}$ is bounded by $M^{-1}$. Note that when 
$t_{i} \leq k \leq i$, we have
\begin{align*}
    s_{k}&=\prod_{p=t_{i}}^{k}\left(I-\eta_{p} A\right) s_{t_{i}-1}+\sum_{p=t_{i}}^{k} \prod_{i=p+1}^{k}\left(I-\eta_{i} A\right) \eta_{p}\left(A \tilde{\Delta}_{p-1}-\nabla f_{\nu}\left(x_{p-1}\right)\right)\\
    &=Y_{t_{i}-1}^{k} s_{t_{i}-1}+\sum_{p=t_{i}}^{k} Y_{p}^{k} \eta_{p}\left(A \tilde{\Delta}_{p-1}-\nabla f_{\nu}\left(x_{p-1}\right)\right),
\end{align*}
and
$$\sum_{k=t_{i}}^{i} s_{k}=S_{t_{i}-1}^{k} s_{t_{i}-1}+\sum_{p=t_{i}}^{i}\left(I+S_{p}^{i}\right) \eta_{p}\left(A \tilde{\Delta}_{p-1}-\nabla f_{\nu}\left(x_{p-1}\right)\right).$$
Using triangle inequality and Cauchy-Schwartz inequality as above, we obtain
\begin{equation}\label{FOCMEequ98}
\begin{aligned}
\mathbb{E}\left\|\sum_{k=t_{i}}^{i} s_{k}\right\|^{2} &\lesssim \mathbb{E}\left(\left\|S_{t_{i-1}}^{i} s_{t_{i}-1}\right\|^{2}+\left(\sum_{p=t_{i}}^{i}\left\|I+S_{p}^{i}\right\| \eta_{p}\left\|A \tilde{\Delta}_{p-1}-\nabla f_{\nu}\left(x_{p-1}\right)\right\|\right)^{2}\right)\\
&\lesssim\left\|S_{t_{i-1}}^{i}\right\|^{2} \mathbb{E}\left\|s_{t_{i}-1}\right\|^{2}+\left(\sum_{p=t_{i}}^{i}\left\|I+S_{p}^{i}\right\|^{2} \eta_{p}^{2}\right)\left(\sum_{p=t_{i}}^{i} \mathbb{E}\left\|A \tilde{\Delta}_{p-1}-\nabla f_{\nu}\left(x_{p-1}\right)\right\|^{2}\right).
\end{aligned}
\end{equation}
From Lemma \ref{FOCMElma2}, we see that $\left\|S_{p}^{i}\right\| \lesssim p^{\alpha}$. Hence, we have
$$\sum_{p=t_{i}}^{i}\left\|I+S_{p}^{n}\right\|^{2} \eta_{p}^{2} \lesssim \sum_{p=t_{i}}^{i} p^{2 \alpha} p^{-2 \alpha}=l_{i}.$$
According to Taylor expansion around $x_{\nu}^{*}$, we have
 $\left\|A \tilde{\Delta}_{p}-\nabla f_{\nu}\left(x_{p}\right)\right\|=O\left(\left\|\tilde{\Delta}_{p}\right\|^{2}\right)$. Hence we have,
\begin{equation}\label{FOCMEequ99}
    \sum_{p=t_{i}}^{i} \mathbb{E}\left\|A \tilde{\Delta}_{p-1}-\nabla f_{\nu}\left(x_{p-1}\right)\right\|^{2} \asymp \sum_{p=t_{i}}^{i} \mathbb{E}\left\|\tilde{\Delta}_{p-1}\right\|^{4} \lesssim l_{i} t_{i}^{-2 \alpha}.
\end{equation}
Now note that $s_{k}=\tilde{\Delta}_{k}-U_{k}$. From Lemma \ref{FOCMElma3}, we have $\mathbb{E}\left\|\tilde{\Delta}_{k}\right\| \asymp \mathbb{E}\left\|U_{k}\right\| \lesssim k^{-\alpha} .$ So,
$$\mathbb{E}\left\|s_{k}\right\|^{2} \leq 2 \mathbb{E}\left\|\tilde{\Delta}_{k}\right\|^{2}+2 \mathbb{E}\left\|U_{k}\right\|^{2} \lesssim k^{-2 \alpha}.$$
Thus,
$$\mathbb{E}\left\|\sum_{k=t_{i}}^{i} s_{k}\right\|^{2} \lesssim t_{i}^{2 \alpha} t_{i}^{-2 \alpha}+l_{i}^{2} t_{i}^{-2 \alpha}=1+l_{i}^{2} t_{i}^{-2 \alpha}.$$
Since $\left(\sum_{i=1}^{n} l_{i}\right)^{-1} \asymp\left(\sum_{m=1}^{M} n_{m}^{2}\right)^{-1}$, $n_{m} \asymp m^{(1+\alpha) /(1-\alpha)}$ and $a_{m} \asymp m^{2 /(1-\alpha)}$, we obtain
\begin{equation}\label{FOCMEequ100}
\begin{aligned}
\left(\sum_{i=1}^{n} l_{i}\right)^{-1} \sum_{i=1}^{n} \mathbb{E}\left\|\sum_{k=t_{i}}^{i} s_{k}\right\|^{2} &\lesssim\left(\sum_{m=1}^{M} n_{m}^{2}\right)^{-1}\left(\sum_{m=1}^{M} \sum_{i=a_{m}}^{a_{m+1}-1}\left(1+l_{i}^{2} a_{m}^{-2 \alpha}\right)\right)\\
&\lesssim n_{M}^{-1}+\left(\sum_{m=1}^{M} n_{m}^{2}\right)^{-1}\left(\sum_{m=1}^{M} n_{m}^{3} a_{m}^{-2 \alpha}\right)\\
&\lesssim M^{-1}.
\end{aligned}
\end{equation}
Hence,~\eqref{FOCMEclm1} is proved through (\ref{FOCMEequ94}), (\ref{FOCMEequ97}) and (\ref{FOCMEequ100}).
\end{proof}

\section{Technical Lemmas}
\noindent For the sake of completeness, we state some technical lemmas from~\cite{polyak1992acceleration} and~\cite{chen2016statistical} that were used in the proof before.
\begin{lemma}\label{FOCMElma1}
For any $i \in \mathbb{N},$ define a matrix sequence $\left\{Y_{i}^{j}\right\}$ with $Y_{i}^{i}=I$ and for any $j>i$ we have
$Y_{i}^{j}=\prod_{k=i+1}^{j}\left(I-\eta_{k} A\right),$
where $\eta_{k}$ is chosen to be $\eta k^{-\alpha}$ for $\alpha \in(0.5,1) .$ Then we have
$$\left\|Y_{i}^{j}\right\| \leq \exp \left(-\eta \lambda_{A} \sum_{k=i+1}^{j} k^{-\alpha}\right) \leq \exp \left[-\frac{\lambda_{A} \eta}{1-\alpha}\left(j^{1-\alpha}-(i+1)^{1-\alpha}\right)\right]$$
where $\lambda_{A}>0$ is the minimum eigenvalue of $A$.
\end{lemma}
\begin{lemma}\label{FOCMElma2}
With $Y_{i}^{j}$ defined in Lemma \ref{FOCMElma1}, let $S_{i}^{j}=\sum_{k=i+1}^{j} Y_{i}^{k}$ for any $j>i$ with $S_{i}^{i}=0 .$ Then we have $\left\|S_{i}^{j}\right\| \lesssim(i+1)^{\alpha} \asymp i^{\alpha}.$
\end{lemma}

\begin{lemma}\label{FOCMElma4}
Let $a_{k}=\left\lfloor C k^{\beta}\right\rfloor, k \geq 2,$ for some constant $C$ and $\beta>\frac{1}{1-\alpha} .$ If $a_{M} \leq n< a_{M+1},$ then we have
\begin{equation}\label{FOCMEequ30}
    \lim _{M \rightarrow \infty} \frac{\sum_{i=1}^{n} l_{i}}{\sum_{i=1}^{a_{M+1}-1} l_{i}}=1,
\end{equation}
and
\begin{equation}\label{FOCMEequ31}
    \frac{\left(a_{M+1}-a_{M}\right)^{2}}{\sum_{m=1}^{M}\left(a_{m+1}-a_{m}\right)^{2}} \lesssim M^{-1},~~\text{and}~~\frac{a_{M}^{\alpha}}{n_{M}} \rightarrow 0.
\end{equation}
\end{lemma}
\begin{lemma}\label{FOCMElma3}
With definition of $Y_{i}^{j}$ in Lemma \ref{FOCMElma1} sequence $U_{n}$ defined in Equation (\ref{linearseq}) can be rewritten as $U_{k}=\left(I-\eta_{k} A\right) U_{k-1}+\eta_{k} \tilde{\xi}_{k}=Y_{s}^{k} U_{s}+\sum_{p=s+1}^{k} Y_{p}^{k} \eta_{p} \tilde{\xi}_{p}$. Then, according to Lemma B.3 in \cite{chen2016statistical} we have $\mathbb{E}\left\|U_{k}\right\|^{2} \lesssim k^{-\alpha}.$
\end{lemma}

\bibliographystyle{amsalpha}
\bibliography{ms}

\newcommand{\etalchar}[1]{$^{#1}$}
\begin{thebibliography}{CPPH{\etalchar{+}}20}

\bibitem[ABE19]{anastasiou2019normal}
Andreas Anastasiou, Krishnakumar Balasubramanian, and Murat~A Erdogdu.
\newblock Normal approximation for stochastic gradient descent via
  non-asymptotic rates of martingale {CLT}.
\newblock In {\em Conference on Learning Theory}, pages 115--137. PMLR, 2019.

\bibitem[AD19]{asi2019stochastic}
Hilal Asi and John~C Duchi.
\newblock Stochastic (approximate) proximal point methods: Convergence,
  optimality, and adaptivity.
\newblock {\em SIAM Journal on Optimization}, 2019.

\bibitem[ADX10]{agarwal2010optimal}
Alekh Agarwal, Ofer Dekel, and Lin Xiao.
\newblock Optimal algorithms for online convex optimization with multi-point
  bandit feedback.
\newblock In {\em Conference on Learning Theory}, pages 28--40, 2010.

\bibitem[AH17]{audet2017derivative}
Charles Audet and Warren Hare.
\newblock {\em Derivative-free and blackbox optimization}.
\newblock Springer, 2017.

\bibitem[AS42]{aitken1942xv}
AC~Aitken and H~Silverstone.
\newblock On the estimation of statistical parameters.
\newblock {\em Proceedings of the Royal Society of Edinburgh Section A:
  Mathematics}, 61(2):186--194, 1942.

\bibitem[BG18]{balasubramanian2018zerothc}
Krishnakumar Balasubramanian and Saeed Ghadimi.
\newblock Zeroth-order (non)-convex stochastic optimization via conditional
  gradient and gradient updates.
\newblock In {\em Advances in Neural Information Processing Systems}, pages
  3459--3468, 2018.

\bibitem[BG21]{balasubramanian2018zeroth}
Krishnakumar Balasubramanian and Saeed Ghadimi.
\newblock Zeroth-order nonconvex stochastic optimization: Handling constraints,
  high-dimensionality and saddle-points.
\newblock {\em Foundations of Computational Mathematics (to appear)}, 2021.

\bibitem[Blu54]{blum1954multidimensional}
Julius~R Blum.
\newblock Multidimensional stochastic approximation methods.
\newblock {\em The Annals of Mathematical Statistics}, pages 737--744, 1954.

\bibitem[Bre13]{brent2013algorithms}
Richard~P Brent.
\newblock {\em Algorithms for minimization without derivatives}.
\newblock Courier Corporation, 2013.

\bibitem[Chu54]{chung1954stochastic}
Kai~Lai Chung.
\newblock On a stochastic approximation method.
\newblock {\em The Annals of Mathematical Statistics}, pages 463--483, 1954.

\bibitem[CLC{\etalchar{+}}18]{cheng2018query}
Minhao Cheng, Thong Le, Pin-Yu Chen, Huan Zhang, JinFeng Yi, and Cho-Jui Hsieh.
\newblock Query-efficient hard-label black-box attack: An optimization-based
  approach.
\newblock In {\em International Conference on Learning Representations}, 2018.

\bibitem[CLTZ20]{chen2016statistical}
Xi~Chen, Jason~D Lee, Xin~T Tong, and Yichen Zhang.
\newblock Statistical inference for model parameters in stochastic gradient
  descent.
\newblock {\em The Annals of Statistics}, 2020.

\bibitem[CLX{\etalchar{+}}19]{chen2019zo}
Xiangyi Chen, Sijia Liu, Kaidi Xu, Xingguo Li, Xue Lin, Mingyi Hong, and David
  Cox.
\newblock {ZO-AdaMM: Zeroth-order adaptive momentum method for black-box
  optimization}.
\newblock {\em Advances in Neural Information Processing Systems}, 32, 2019.

\bibitem[CMYZ20]{cai2020zeroth}
HanQin Cai, Daniel Mckenzie, Wotao Yin, and Zhenliang Zhang.
\newblock Zeroth-order regularized optimization ({ZORO}): Approximately sparse
  gradients and adaptive sampling.
\newblock {\em arXiv preprint arXiv:2003.13001}, 2020.

\bibitem[CPPH{\etalchar{+}}20]{choromanski2020provably}
Krzysztof Choromanski, Aldo Pacchiano, Jack Parker-Holder, Yunhao Tang, Deepali
  Jain, Yuxiang Yang, Atil Iscen, Jasmine Hsu, and Vikas Sindhwani.
\newblock Provably robust blackbox optimization for reinforcement learning.
\newblock In {\em Conference on Robot Learning}, pages 683--696. PMLR, 2020.

\bibitem[Cra46]{cramir1946mathematical}
Harald Cramer.
\newblock Mathematical methods of statistics.
\newblock {\em Princeton U. Press, Princeton}, 500, 1946.

\bibitem[CSC{\etalchar{+}}19]{cheng2019sign}
Minhao Cheng, Simranjit Singh, Patrick~H Chen, Pin-Yu Chen, Sijia Liu, and
  Cho-Jui Hsieh.
\newblock {Sign-OPT: A Query-Efficient Hard-label Adversarial Attack}.
\newblock In {\em International Conference on Learning Representations}, 2019.

\bibitem[CSV09]{conn2009introduction}
Andrew~R Conn, Katya Scheinberg, and Luis~N Vicente.
\newblock {\em Introduction to derivative-free optimization}.
\newblock SIAM, 2009.

\bibitem[CZS{\etalchar{+}}17]{chen2017zoo}
Pin-Yu Chen, Huan Zhang, Yash Sharma, Jinfeng Yi, and Cho-Jui Hsieh.
\newblock {ZOO}: Zeroth order optimization based black-box attacks to deep
  neural networks without training substitute models.
\newblock In {\em Proceedings of the 10th ACM workshop on artificial
  intelligence and security}, pages 15--26, 2017.

\bibitem[CZYG20]{chen2020frank}
Jinghui Chen, Dongruo Zhou, Jinfeng Yi, and Quanquan Gu.
\newblock A {Frank-Wolfe} framework for efficient and effective adversarial
  attacks.
\newblock In {\em Proceedings of the AAAI Conference on Artificial
  Intelligence}, volume~34, pages 3486--3494, 2020.

\bibitem[Dar45]{darmois1945limites}
Georges Darmois.
\newblock Sur les limites de la dispersion de certaines estimations.
\newblock {\em Revue de l'Institut International de Statistique}, pages 9--15,
  1945.

\bibitem[DDB20]{dieuleveut2020bridging}
Aymeric Dieuleveut, Alain Durmus, and Francis Bach.
\newblock Bridging the gap between constant step size stochastic gradient
  descent and {M}arkov chains.
\newblock {\em Annals of Statistics}, 48(3):1348--1382, 2020.

\bibitem[Dip03]{dippon2003accelerated}
J{\"u}rgen Dippon.
\newblock Accelerated randomized stochastic optimization.
\newblock {\em The Annals of Statistics}, 31(4):1260--1281, 2003.

\bibitem[DJWW15]{duchi2015optimal}
John~C Duchi, Michael~I Jordan, Martin~J Wainwright, and Andre Wibisono.
\newblock Optimal rates for zero-order convex optimization: The power of two
  function evaluations.
\newblock {\em IEEE Transactions on Information Theory}, 61(5):2788--2806,
  2015.

\bibitem[DKKW12]{dror2012yahoo}
Gideon Dror, Noam Koenigstein, Yehuda Koren, and Markus Weimer.
\newblock {The Yahoo! music dataset and KDD-cup'11}.
\newblock In {\em Proceedings of KDD Cup}, 2012.

\bibitem[DR97]{dippon1997weighted}
J{\"u}rgen Dippon and Joachim Renz.
\newblock Weighted means in stochastic approximation of minima.
\newblock {\em SIAM Journal on Control and Optimization}, 35(5):1811--1827,
  1997.

\bibitem[DR20]{duchi2016local}
John Duchi and Feng Ruan.
\newblock Asymptotic optimality in stochastic optimization.
\newblock {\em The Annals of Statistics}, 2020.

\bibitem[DSW{\etalchar{+}}19]{dong2019efficient}
Yinpeng Dong, Hang Su, Baoyuan Wu, Zhifeng Li, Wei Liu, Tong Zhang, and Jun
  Zhu.
\newblock Efficient decision-based black-box adversarial attacks on face
  recognition.
\newblock In {\em IEEE Conference on Computer Vision and Pattern Recognition},
  2019.

\bibitem[Fab68]{fabian1968asymptotic}
Vaclav Fabian.
\newblock On asymptotic normality in stochastic approximation.
\newblock {\em The Annals of Mathematical Statistics}, 39(4):1327--1332, 1968.

\bibitem[Fer52]{fermi1952numerical}
Enrico Fermi.
\newblock Numerical solution of a minimum problem.
\newblock Technical report, Los Alamos Scientific Lab., Los Alamos, NM, 1952.

\bibitem[FLLZ18]{fang2018spider}
Cong Fang, Chris~Junchi Li, Zhouchen Lin, and Tong Zhang.
\newblock {SPIDER}: Near-optimal non-convex optimization via stochastic path
  integrated differential estimator.
\newblock {\em arXiv preprint arXiv:1807.01695}, 2018.

\bibitem[Fr{\'e}43]{frechet1943extension}
Maurice Fr{\'e}chet.
\newblock Sur l'extension de certaines {\'e}valuations statistiques au cas de
  petits {\'e}chantillons.
\newblock {\em Revue de l'Institut International de Statistique}, pages
  182--205, 1943.

\bibitem[FXY18]{fang2018online}
Yixin Fang, Jinfeng Xu, and Lei Yang.
\newblock Online bootstrap confidence intervals for the stochastic gradient
  descent estimator.
\newblock {\em The Journal of Machine Learning Research}, 19(1):3053--3073,
  2018.

\bibitem[GDG18]{gorbunov2018accelerated}
Eduard Gorbunov, Pavel Dvurechensky, and Alexander Gasnikov.
\newblock An accelerated method for derivative-free smooth stochastic convex
  optimization.
\newblock {\em arXiv preprint arXiv:1802.09022}, 2018.

\bibitem[Ger97]{gerencser1997rate}
L{\'a}szl{\'o} Gerencs{\'e}r.
\newblock {Rate of convergence of moments of Spall's SPSA method}.
\newblock In {\em IEEE European Control Conference (ECC)}, 1997.

\bibitem[GH20]{gao2020can}
Hongchang Gao and Heng Huang.
\newblock Can stochastic zeroth-order {Frank-Wolfe} method converge faster for
  non-convex problems?
\newblock In {\em International Conference on Machine Learning}, pages
  3377--3386. PMLR, 2020.

\bibitem[GKK{\etalchar{+}}19]{golovin2019gradientless}
Daniel Golovin, John Karro, Greg Kochanski, Chansoo Lee, Xingyou Song, and
  Qiuyi Zhang.
\newblock Gradientless descent: High-dimensional zeroth-order optimization.
\newblock In {\em International Conference on Learning Representations}, 2019.

\bibitem[GL13]{ghadimi2013stochastic}
Saeed Ghadimi and Guanghui Lan.
\newblock Stochastic first-and zeroth-order methods for nonconvex stochastic
  programming.
\newblock {\em SIAM Journal on Optimization}, 23(4):2341--2368, 2013.

\bibitem[HGH{\etalchar{+}}19]{huang2019faster}
Feihu Huang, Bin Gu, Zhouyuan Huo, Songcan Chen, and Heng Huang.
\newblock Faster gradient-free proximal stochastic methods for nonconvex
  nonsmooth optimization.
\newblock In {\em Proceedings of the AAAI Conference on Artificial
  Intelligence}, volume~33, pages 1503--1510, 2019.

\bibitem[HJ61]{hooke1961direct}
Robert Hooke and Terry~A Jeeves.
\newblock Direct search solution of numerical and statistical problems.
\newblock {\em Journal of the ACM (JACM)}, 8(2):212--229, 1961.

\bibitem[HTC20]{huang2020accelerated}
Feihu Huang, Lue Tao, and Songcan Chen.
\newblock Accelerated stochastic gradient-free and projection-free methods.
\newblock In {\em International Conference on Machine Learning}, pages
  4519--4530. PMLR, 2020.

\bibitem[IEAL18]{ilyas2018black}
Andrew Ilyas, Logan Engstrom, Anish Athalye, and Jessy Lin.
\newblock Black-box adversarial attacks with limited queries and information.
\newblock In {\em International Conference on Machine Learning}, pages
  2137--2146. PMLR, 2018.

\bibitem[JNR12]{jamieson2012query}
Kevin~G Jamieson, Robert Nowak, and Ben Recht.
\newblock Query complexity of derivative-free optimization.
\newblock {\em Advances in Neural Information Processing Systems}, 25, 2012.

\bibitem[JT18]{ji2018risk}
Ziwei Ji and Matus Telgarsky.
\newblock Risk and parameter convergence of logistic regression.
\newblock {\em arXiv preprint arXiv:1803.07300}, 2018.

\bibitem[JT19]{ji2019implicit}
Ziwei Ji and Matus Telgarsky.
\newblock The implicit bias of gradient descent on nonseparable data.
\newblock In {\em Conference on Learning Theory}, pages 1772--1798. PMLR, 2019.

\bibitem[JWZL19]{ji2019improved}
Kaiyi Ji, Zhe Wang, Yi~Zhou, and Yingbin Liang.
\newblock Improved zeroth-order variance reduced algorithms and analysis for
  nonconvex optimization.
\newblock In {\em International Conference on Machine Learning}, pages
  3100--3109. PMLR, 2019.

\bibitem[KLT03]{kolda2003optimization}
Tamara~G Kolda, Robert~Michael Lewis, and Virginia Torczon.
\newblock Optimization by direct search: New perspectives on some classical and
  modern methods.
\newblock {\em SIAM review}, 45(3):385--482, 2003.

\bibitem[KSN99]{kleinman1999simulation}
Nathan~L Kleinman, James~C Spall, and Daniel~Q Naiman.
\newblock Simulation-based optimization with stochastic approximation using
  common random numbers.
\newblock {\em Management Science}, 45(11):1570--1578, 1999.

\bibitem[KW52]{kiefer1952stochastic}
Jack Kiefer and Jacob Wolfowitz.
\newblock Stochastic estimation of the maximum of a regression function.
\newblock {\em The Annals of Mathematical Statistics}, 23(3):462--466, 1952.

\bibitem[LBM20]{li2020zeroth}
Jiaxiang Li, Krishnakumar Balasubramanian, and Shiqian Ma.
\newblock Stochastic zeroth-order riemannian derivative estimation and
  optimization.
\newblock {\em arXiv preprint arXiv:2003.11238}, 2020.

\bibitem[LCCH18]{liu2018zeroth}
Sijia Liu, Jie Chen, Pin-Yu Chen, and Alfred Hero.
\newblock Zeroth-order online alternating direction method of multipliers:
  Convergence analysis and applications.
\newblock In {\em International Conference on Artificial Intelligence and
  Statistics}, 2018.

\bibitem[LCK{\etalchar{+}}20]{liu2020primer}
Sijia Liu, Pin-Yu Chen, Bhavya Kailkhura, Gaoyuan Zhang, Alfred~O Hero~III, and
  Pramod~K Varshney.
\newblock A primer on zeroth-order optimization in signal processing and
  machine learning.
\newblock {\em IEEE Signal Processing Magazine}, 37(5):43--54, 2020.

\bibitem[Lia10]{liang2010trajectory}
Faming Liang.
\newblock Trajectory averaging for stochastic approximation {MCMC} algorithms.
\newblock {\em The Annals of Statistics}, 38(5):2823--2856, 2010.

\bibitem[LKC{\etalchar{+}}18]{liu2018zerothv}
Sijia Liu, Bhavya Kailkhura, Pin-Yu Chen, Paishun Ting, Shiyu Chang, and Lisa
  Amini.
\newblock Zeroth-order stochastic variance reduction for nonconvex
  optimization.
\newblock {\em Advances in Neural Information Processing Systems},
  31:3727--3737, 2018.

\bibitem[LKLC18]{li2018approximate}
Tianyang Li, Anastasios Kyrillidis, Liu Liu, and Constantine Caramanis.
\newblock Approximate newton-based statistical inference using only stochastic
  gradients.
\newblock {\em arXiv preprint arXiv:1805.08920}, 2018.

\bibitem[LLW{\etalchar{+}}19]{li2019nattack}
Yandong Li, Lijun Li, Liqiang Wang, Tong Zhang, and Boqing Gong.
\newblock {NATTACK}: {L}earning the distributions of adversarial examples for
  an improved black-box attack on deep neural networks.
\newblock In {\em International Conference on Machine Learning}, pages
  3866--3876. PMLR, 2019.

\bibitem[LMW19]{larson2019derivative}
Jeffrey Larson, Matt Menickelly, and Stefan~M Wild.
\newblock Derivative-free optimization methods.
\newblock {\em Acta Numerica}, 28:287--404, 2019.

\bibitem[LRV{\etalchar{+}}20]{liu2020admm}
Sijia Liu, Parikshit Ram, Deepak Vijaykeerthy, Djallel Bouneffouf, Gregory
  Bramble, Horst Samulowitz, Dakuo Wang, Andrew Conn, and Alexander Gray.
\newblock An {ADMM} based framework for {AutoML} pipeline configuration.
\newblock In {\em Proceedings of the AAAI Conference on Artificial
  Intelligence}, volume~34, pages 4892--4899, 2020.

\bibitem[LY98]{l1998budget}
Pierre L'Ecuyer and George Yin.
\newblock Budget-dependent convergence rate of stochastic approximation.
\newblock {\em SIAM Journal on Optimization}, 8(1):217--247, 1998.

\bibitem[LZH{\etalchar{+}}16]{lian2016comprehensive}
Xiangru Lian, Huan Zhang, Cho-Jui Hsieh, Yijun Huang, and Ji~Liu.
\newblock A comprehensive linear speedup analysis for asynchronous stochastic
  parallel optimization from zeroth-order to first-order.
\newblock In {\em Proceedings of the 30th International Conference on Neural
  Information Processing Systems}, pages 3062--3070, 2016.

\bibitem[MGR18]{mania2018simple}
Horia Mania, Aurelia Guy, and Benjamin Recht.
\newblock Simple random search of static linear policies is competitive for
  reinforcement learning.
\newblock In {\em Advances in Neural Information Processing Systems}, pages
  1800--1809, 2018.

\bibitem[MP07]{mokkadem2007companion}
Abdelkader Mokkadem and Mariane Pelletier.
\newblock {A companion for the Kiefer--Wolfowitz--Blum stochastic approximation
  algorithm}.
\newblock {\em The Annals of Statistics}, 35(4):1749--1772, 2007.

\bibitem[NM65]{nelder1965simplex}
John~A Nelder and Roger Mead.
\newblock A simplex method for function minimization.
\newblock {\em The computer journal}, 7(4):308--313, 1965.

\bibitem[NS17]{nesterov2017random}
Yurii Nesterov and Vladimir Spokoiny.
\newblock Random gradient-free minimization of convex functions.
\newblock {\em Foundations of Computational Mathematics}, 17(2):527--566, 2017.

\bibitem[NY83]{nemirovskij1983problem}
Arkadi Nemirovski and David~Borisovich Yudin.
\newblock Problem complexity and method efficiency in optimization.
\newblock {\em Wiley-Interscience}, 1983.

\bibitem[PJ92]{polyak1992acceleration}
Boris~T Polyak and Anatoli~B Juditsky.
\newblock Acceleration of stochastic approximation by averaging.
\newblock {\em SIAM journal on control and optimization}, 30(4):838--855, 1992.

\bibitem[Pow64]{powell1964efficient}
Michael~JD Powell.
\newblock An efficient method for finding the minimum of a function of several
  variables without calculating derivatives.
\newblock {\em The computer journal}, 7(2):155--162, 1964.

\bibitem[Rao45]{rao1945information}
C~Radhakrishna Rao.
\newblock Information and the accuracy attainable in the estimation of
  statistical parameters.
\newblock {\em Reson. J. Sci. Educ}, 20:78--90, 1945.

\bibitem[Rup88]{ruppert1988efficient}
David Ruppert.
\newblock {Efficient estimations from a slowly convergent Robbins-Monro
  process}.
\newblock Technical report, Cornell University Operations Research and
  Industrial Engineering, 1988.

\bibitem[Sac58]{sacks1958asymptotic}
Jerome Sacks.
\newblock Asymptotic distribution of stochastic approximation procedures.
\newblock {\em The Annals of Mathematical Statistics}, 29(2):373--405, 1958.

\bibitem[Sh86]{sh1986liptzer}
R~Sh.
\newblock Liptzer and an shiryaev, martingale theory.
\newblock {\em Nauka, Moscow}, 1986.

\bibitem[Sha89]{shapiro1989asymptotic}
Alexander Shapiro.
\newblock Asymptotic properties of statistical estimators in stochastic
  programming.
\newblock {\em The Annals of Statistics}, 17(2):841--858, 1989.

\bibitem[Sha96]{shapiro1996simulation}
Alexander Shapiro.
\newblock Simulation-based optimization: Convergence analysis and statistical
  inference.
\newblock {\em Stochastic Models}, 12(3):425--454, 1996.

\bibitem[Sha13]{shamir2013complexity}
Ohad Shamir.
\newblock On the complexity of bandit and derivative-free stochastic convex
  optimization.
\newblock In {\em Conference on Learning Theory}, pages 3--24. PMLR, 2013.

\bibitem[Sha17]{shamir2017optimal}
Ohad Shamir.
\newblock An optimal algorithm for bandit and zero-order convex optimization
  with two-point feedback.
\newblock {\em The Journal of Machine Learning Research}, 18(1):1703--1713,
  2017.

\bibitem[SHH62]{spendley1962sequential}
W.~Spendley, G.~R. Hext, and F.~R. Himsworth.
\newblock Sequential application of simplex designs in optimisation and
  evolutionary operation.
\newblock {\em Technometrics}, 4(4):441--461, 1962.

\bibitem[SHN{\etalchar{+}}18]{soudry2018implicit}
Daniel Soudry, Elad Hoffer, Mor~Shpigel Nacson, Suriya Gunasekar, and Nathan
  Srebro.
\newblock The implicit bias of gradient descent on separable data.
\newblock {\em The Journal of Machine Learning Research}, 19(1):2822--2878,
  2018.

\bibitem[SLA12]{snoek2012practical}
Jasper Snoek, Hugo Larochelle, and Ryan~P Adams.
\newblock Practical bayesian optimization of machine learning algorithms.
\newblock In {\em Proceedings of the 25th International Conference on Neural
  Information Processing Systems-Volume 2}, pages 2951--2959, 2012.

\bibitem[Spa87]{spall1987stochastic}
James~C Spall.
\newblock A stochastic approximation technique for generating maximum
  likelihood parameter estimates.
\newblock In {\em 1987 American control conference}, pages 1161--1167. IEEE,
  1987.

\bibitem[Spa92]{spall1992multivariate}
James~C Spall.
\newblock Multivariate stochastic approximation using a simultaneous
  perturbation gradient approximation.
\newblock {\em IEEE transactions on automatic control}, 37(3):332--341, 1992.

\bibitem[Spa05]{spall2005introduction}
James~C Spall.
\newblock {\em Introduction to stochastic search and optimization:
  {E}stimation, simulation, and control}, volume~65.
\newblock John Wiley \& Sons, 2005.

\bibitem[Ste72]{stein1972bound}
Charles Stein.
\newblock A bound for the error in the normal approximation to the distribution
  of a sum of dependent random variables.
\newblock In {\em Proceedings of the Sixth Berkeley Symposium on Mathematical
  Statistics and Probability, Volume 2: Probability Theory}. The Regents of the
  University of California, 1972.

\bibitem[SZ18]{su2018statistical}
Weijie Su and Yuancheng Zhu.
\newblock Statistical inference for online learning and stochastic
  approximation via hierarchical incremental gradient descent.
\newblock {\em arXiv preprint arXiv:1802.04876}, 2018.

\bibitem[SZK19]{sahu2019towards}
Anit~Kumar Sahu, Manzil Zaheer, and Soummya Kar.
\newblock Towards gradient free and projection free stochastic optimization.
\newblock In {\em The 22nd International Conference on Artificial Intelligence
  and Statistics}, pages 3468--3477. PMLR, 2019.

\bibitem[TA17]{toulis2017asymptotic}
Panos Toulis and Edoardo~M Airoldi.
\newblock Asymptotic and finite-sample properties of estimators based on
  stochastic gradients.
\newblock {\em The Annals of Statistics}, 45(4):1694--1727, 2017.

\bibitem[TLC99]{tang1999asymptotic}
Qian-Yu Tang, Pierre L'Ecuyer, and Han-Fu Chen.
\newblock Asymptotic efficiency of perturbation-analysis-based stochastic
  approximation with averaging.
\newblock {\em SIAM Journal on Control and Optimization}, 37(6):1822--1847,
  1999.

\bibitem[VdV00]{van2000asymptotic}
Aad~W Van~der Vaart.
\newblock {\em Asymptotic statistics}.
\newblock Cambridge university press, 2000.

\bibitem[WDBS18]{wang2018stochastic}
Yining Wang, Simon Du, Sivaraman Balakrishnan, and Aarti Singh.
\newblock Stochastic zeroth-order optimization in high dimensions.
\newblock In {\em International Conference on Artificial Intelligence and
  Statistics}, pages 1356--1365. PMLR, 2018.

\bibitem[YBVE20]{yu2020analysis}
Lu~Yu, Krishnakumar Balasubramanian, Stanislav Volgushev, and Murat~A Erdogdu.
\newblock An analysis of constant step size sgd in the non-convex regime:
  {A}symptotic normality and bias.
\newblock {\em arXiv preprint arXiv:2006.07904}, 2020.

\bibitem[Yin99]{yin1999rates}
George Yin.
\newblock Rates of convergence for a class of global stochastic optimization
  algorithms.
\newblock {\em SIAM Journal on Optimization}, 10(1):99--120, 1999.

\bibitem[YKLY18]{yu2018generic}
Xiaotian Yu, Irwin King, Michael~R Lyu, and Tianbao Yang.
\newblock A generic approach for accelerating stochastic zeroth-order convex
  optimization.
\newblock In {\em Proceedings of the 27th International Joint Conference on
  Artificial Intelligence}, pages 3040--3046, 2018.

\bibitem[Zab13]{zabinsky2013stochastic}
Zelda~B Zabinsky.
\newblock {\em Stochastic adaptive search for global optimization}, volume~72.
\newblock Springer Science \& Business Media, 2013.

\bibitem[ZCW20]{zhu2020fully}
Wanrong Zhu, Xi~Chen, and Wei~Biao Wu.
\newblock A fully online approach for covariance matrices estimation of
  stochastic gradient descent solutions.
\newblock {\em arXiv preprint arXiv:2002.03979}, 2020.

\bibitem[ZLC{\etalchar{+}}19]{zhao2019design}
P.~Zhao, S.~Liu, P-Y. Chen, N.~Hoang, K.~Xu, B.~Kailkhura, and X.~Lin.
\newblock On the design of black-box adversarial examples by leveraging
  gradient-free optimization and operator splitting method.
\newblock In {\em IEEE CVPR}, 2019.

\end{thebibliography}
\end{document}